\documentclass{article}

\PassOptionsToPackage{numbers, compress}{natbib}

\usepackage[preprint]{neurips_2026}

\usepackage[utf8]{inputenc} 
\usepackage[T1]{fontenc}    
\usepackage{hyperref}       
\usepackage{url}            
\usepackage{booktabs}       
\usepackage{amsfonts}       
\usepackage{nicefrac}       
\usepackage{microtype}      
\usepackage{xcolor}         

\usepackage{amsmath,amssymb,amsthm}
\usepackage{mathtools}
\usepackage{algorithm}
\usepackage{algorithmic}
\usepackage{cleveref}
\usepackage{booktabs}
\usepackage{graphicx}
\usepackage{enumitem}
\usepackage{multirow}
\usepackage{subcaption}
\usepackage{placeins}
\newtheorem{theorem}{Theorem}[section]
\newtheorem{proposition}[theorem]{Proposition}
\newtheorem{lemma}[theorem]{Lemma}

\newtheorem{definition}[theorem]{Definition}
\newtheorem{remark}[theorem]{Remark}

\newcommand{\E}{\mathbb{E}}
\newcommand{\R}{\mathbb{R}}
\newcommand{\KL}{\mathrm{KL}}
\newcommand{\MBON}{\mathrm{MBON}}
\newcommand{\BON}{\mathrm{BoN}}

\newcommand{\Var}{\mathrm{Var}}

\newcommand{\ind}{\mathbf{1}}
\DeclareMathOperator*{\argmax}{arg\,max}

\title{Privacy Without Regret: Differentially Private Inference-Time Alignment}

\author{%
  Ishi Jain \\
  Indian Institute of Technology Kanpur \\
  \texttt{ijain43000@gmail.com} \\
  \And
  Nandini Bhattad \\
  Indian Institute of Technology Kanpur \\
\texttt{bhattadnandini13@gmail.com} \\
\And
 Sayak Ray Chowdhury 
    \\
  Indian Institute of Technology Kanpur\\
  \texttt{sayakrc@iitk.ac.in} \\
}

\begin{document}

\maketitle

\begin{abstract}
Best-of-N (BoN) sampling is the simplest and most widely deployed inference-time alignment strategy, but it suffers from two distinct problems: reward hacking, in which the selected response exploits errors in the proxy reward model, and the absence of any privacy protection for the sensitive human preference data used to train that reward model. We show that a single intervention—adding calibrated noise to reward scores before selection—resolves both.
Our first result, Private Best-of-N (PrivBoN), establishes that Gumbel noise at an appropriate scale simultaneously provides $\epsilon$-differential privacy and implements KL-regularized alignment. Whenever the privacy budget exceeds a critical threshold $\epsilon^*$, the privacy-mandated noise is the regret-optimal regularization, and privacy imposes zero additional alignment cost—matching the information-theoretic skyline of \citet{huang2025bon}.
Because $\epsilon^*$ depends on an unknown coverage coefficient, we introduce Private Inference-Time Pessimism (PrivITP), which combines $\chi^2$-regularized rejection sampling with a two-phase Gaussian mechanism. PrivITP achieves ex-post $(\epsilon,\delta)$-DP with a privacy cost independent of the number of responses $n$, cleanly decouples the regularization parameter from the privacy parameter, and attains the skyline up to a noise-inflation term. 
Experiments across several language models, datasets, and reward models confirm our results: PrivBoN and PrivITP are scaling-monotonic (unlike BoN, which degrades past a critical $n$), and PrivITP matches or outperforms PrivBoN at equivalent privacy levels, with the largest gains in the strong-privacy regime.
\end{abstract}

\section{Introduction}
\label{sec:intro}

Aligning large language models (LLMs) with human preferences is a central challenge in modern AI~\citep{stiennon2020learning}. While training-time methods such as RLHF~\citep{christiano2017deep} and DPO~\citep{rafailov2023direct} have received extensive attention, a parallel line of work focuses on \emph{inference-time alignment}: improving the quality of a frozen model's outputs at serving time, without updating its parameters. Inference-time alignment has become increasingly important as frontier models are deployed via APIs to millions of users, making it impractical to fine-tune per-user or per-application preferences into the model weights. The simplest and most widely deployed inference-time strategy is Best-of-$N$ (BoN) sampling~\citep{nakano2021webgpt,gui2024bonbon}, which generates $n$ candidate responses from a base policy $\pi_0$, scores each with a learned reward model $\hat r$, and returns the top-scoring candidate. BoN is attractive for its simplicity -- it requires no gradient computation, works with any frozen model, and scales naturally with test-time compute~\citep{snell2024scaling}.

A growing body of recent work has established tight theoretical characterizations of BoN's KL vs. win-rate tradeoff~\citep{beirami2024theoretical,yang2024asymptotics} and proposed smoothed variants that interpolate between hard maximization and random selection~\citep{verdun2025soft,aminian2025smoothing,khalaf2025inference}. However, BoN is known to be vulnerable to \emph{reward hacking}~\citep{gao2023scaling,khalaf2025inference}: as the number of responses $n$ grows, BoN increasingly selects responses in which the learned reward $\hat r$ overestimates the true reward $r^*$, thereby degrading alignment. \citet{huang2025bon} formalizes this by showing that BoN is not scaling-monotonic -- more candidates can yield \emph{worse} alignment and that BoN's regret crucially depends on the quality of the learned reward model $\hat r$. 

Reward models are typically trained on sensitive human preference data -- pairwise
comparisons from crowd workers or domain experts whose judgments encode personal
values and cultural context — and are increasingly shared across applications.
Each reward-model evaluation leaks information about this training data, and over a deployment session of $T$ queries, the cumulative leakage through selected outputs
is uncontrolled. Concretely, consider a reward model $\hat{r}_\mathcal{D}$ deployed
as a BoN selection service. An adversary wishing to determine whether a specific
preference example $(x_0, y_0^+, y_0^-)$ was in $\mathcal{D}$ can craft prompts
whose plausible completions are structured around $y_0^+$ and $y_0^-$, query the
service repeatedly, and observe shifts in the empirical distribution of returned
responses. Since BoN's argmax is deterministic given reward scores, even small
training-attributable shifts in $\hat{r}_\mathcal{D}$ manifest as measurable shifts
in selection frequencies---the exact signal exploited by membership-inference
attacks on LLMs~\citep{carlini2021extracting} and downstream-output leakage
attacks~\citep{tramerposition}.


While differential privacy has been widely applied to training-time alignment (SFT,
reward-model training, and RLHF)~\citep{yu2025limits,wang2025private}, inference-time
privacy remains largely unstudied. We fill this gap by introducing \emph{differentially
private inference-time alignment}: mechanisms ensuring that the selected response
satisfies $\varepsilon$-DP with respect to $D$, bounding the information any adversary
can extract about any individual preference---regardless of query budget or side
information.

We make the following \textit{contributions}.

\textbf{(1) Private Best-of-N} (PrivBoN, \Cref{sec:mbon}).
We show that adding i.i.d.\ Gumbel noise with scale $\sigma = 2\Delta_r/\varepsilon$ to
reward scores before BoN selection yields a mechanism that simultaneously satisfies
$\varepsilon$-DP and implements the KL-regularized alignment optimum in the large-$n$
limit, where $\Delta_r$ is the reward-model sensitivity on $D$. When the privacy budget
exceeds the critical threshold $\varepsilon^*(x) = 2\Delta_r\sqrt{C^{\pi^*}(x)}/\varepsilon_{\mathrm{RM}}(x)$,
the privacy-mandated noise coincides with the regret-optimal regularization and privacy
becomes ``free''---PrivBoN's regret $O(\sqrt{C^{\pi^*}\varepsilon_{\mathrm{RM}}^2})$ matches
the information-theoretic skyline of~\citet{huang2025bon} at zero additional alignment cost.
Here $\varepsilon_{\mathrm{RM}}$ denotes reward-model error and $C^{\pi^*}$ the coverage
coefficient of a comparator policy~$\pi^*$.

\textbf{(2) Private Inference-Time Pessimism} (PrivITP, \Cref{sec:privitp}).
PrivBoN's ``privacy is free'' threshold depends on the unknown coverage coefficient
$C^{\pi^*}$, making it unverifiable in practice. Our main contribution, PrivITP, resolves
this by combining $\chi^2$-regularized rejection sampling with a two-phase Gaussian
mechanism. We prove an ex-post approximate-DP guarantee with an
$n$-independent bound that depends only on the realized halting time, via a \emph{reformulation lemma}
that maps ReLU-based rejection sampling onto a randomized-threshold mechanism. The
regret bound $O(\sqrt{C^{\pi^*}(\varepsilon_{\mathrm{RM}}^2 + R_{\max}\sigma)})$ cleanly
separates alignment error from privacy cost and matches the skyline up to noise inflation,
where $R_{\max}$ is the maximum reward; the analysis relies on a
\emph{smoothed-weight policy comparison lemma} that bounds value gaps via the
$\ell_\infty$-distance of policy weight functions. The ex-post structure further enables
adaptive composition via the FSRC framework~\citep{lebensold2024privacy}, answering
substantially more queries under a fixed privacy budget than standard composition. 

\textbf{(3) Empirical validation} (\Cref{sec:experiments}).
Across four reward models and two
language models of varying sizes on GSM8K, MMLU, and MATH datasets, we demonstrate that the theoretical results hold: Private algorithms eliminate BoN's reward-hacking pathology while preserving its scaling benefits; PrivITP consistently matches or outperforms PrivBoN at matched privacy levels; and the ex-post composition structure translates into substantial multi-query deployment gains. 

PrivITP and PrivBoN are, to our knowledge, the first inference-time alignment algorithms whose regret matches the information-theoretic skyline under formal differential privacy guarantees.

\section{Preliminaries}
\label{sec:prelim}

\textbf{Inference-time alignment.} We consider a setting in which a user submits a prompt $x \in \mathcal{X}$ and the system must return a single response $y \in \mathcal{Y}$. The system has access to a base language model (the \emph{reference policy}) $\pi_0: \mathcal{X} \to \Delta(\mathcal{Y})$, assumed to have full support on $\mathcal{Y}$, and a reward model $\hat{r}_D:\mathcal{X} \times \mathcal{Y} \to \mathbb{R}$ trained on a dataset $D$. The true reward function $r^* : \mathcal{X} \times \mathcal{Y} \to [0, R_{\max}]$ is unknown, and the trained reward model $\hat{r}_D$ serves as its proxy.
The base policy $\pi_0$ is a trained language model (e.g., an instruction-tuned model using SFT and/or RLHF), the true reward $r^*$ captures the desired objective (e.g., alignment with human preferences or correctness under a proof checker), and proxy reward $\hat r_D$ is a custom-trained model on a (possibly sensitive) dataset $D$.

Given a base policy $\pi_0$, reward model $\hat r_D$, and a prompt $x \in \mathcal{X}$, the goal of the algorithm designer is to construct a policy $\hat\pi$, which generates high-quality responses as measured by the true reward $r^*(x,y)$. This is formalized using the notion of \emph{inference-time regret} \citep{huang2025bon}. The \emph{regret} of a policy $\hat \pi$ relative to a comparator policy $\pi^*$ is defined as
\begin{equation}
\label{eq:regret_def}
\mathrm{Reg}(\hat \pi; x) = J(\pi^*;x) - J(\hat \pi;x),
\end{equation}
where $J(\pi;x) = \E_{y \sim \pi(\cdot|x)}[r^*(x,y)]$ denotes the value of a policy $\pi$ for a given prompt $x$.

Without assumptions on the accuracy of the reward model $\hat r_D$ and on the coverage of the base policy $\pi_0$, achieving low regret is not possible~\citep{huang2025bon}.
To this end, define the reward model error $\varepsilon(x,y) = r^*(x,y)-\hat{r}_D(x,y)$. The 
expected squared error under the base policy,
\begin{align*}
    \varepsilon_{\mathrm{RM}}^2(x) = \E_{y\sim\pi_0(\cdot |x)}[\varepsilon(x,y)^2]~,
\end{align*}
measures the mean-squared discrepancy between proxy and rewards. We abstract away reward model training by assuming access to a model $\hat r_D$ with squared-error $\varepsilon_{\mathrm{RM}}^2(x)$ and study how regret depends on this error and what algorithmic interventions can make it small. 
The \emph{$\chi^2$-coverage coefficient} $$C^{\pi}(x) = \E_{y\sim\pi(\cdot | x)}\left[\frac{\pi(y|x)}{\pi_0(y|x)}\right] = 1 + 2\,\chi^2(\pi(\cdot|x)\|\pi_0(\cdot|x))$$ measures how well the base policy $\pi_0$ covers a given policy $\pi$, where $\chi^2(p\|q)$ denotes the $\chi^2$-divergence between probability distributions $p,q$. A comparator policy $\pi^*$ with a larger $C^{\pi^*}$ indicates a harder alignment target. Empirical studies have shown that standard language models provide sufficient coverage of high-quality responses, enabling performance gains at inference time~\citep{huang2025bon}. \citet{huang2025bon} lays down the importance of coverage by establishing a regret skyline in the \emph{sample and evaluate} framework, where, for a given prompt $x$, the learner can sample $n$ responses $\lbrace y_j \rbrace_j$ from the base policy $\pi_0$, and observe their likelihood $\pi_0(y_j|x)$ and reward scores $\hat r_D(x,y_j)$.

\begin{proposition}[Regret skyline {\citep{huang2025bon}}]
\label{prop:skyline}
Fix a prompt $x$ and base policy $\pi_0$. For any inference-time alignment algorithm $\hat\pi$, there exist a reward function $r^*$ and reward model $\hat r_D$ with error $\varepsilon_{\mathrm{RM}}(x)$ such that $\mathrm{Reg}(\pi;x) \geq \tfrac{1}{4}\sqrt{C^{\pi^*}\varepsilon_{\mathrm{RM}}^2(x)}$.
\end{proposition}
This lower bound is tight: \citet{huang2025bon} achieves it with InferenceTimePessimism, a $\chi^2$-regularized rejection sampling algorithm. A central question of this work is whether differential privacy -- which requires adding noise that ostensibly degrades performance -- can be achieved without exceeding this skyline. To this end, we first formalize the privacy definitions.

\textbf{Differential privacy.} We protect the training data $D$ of the reward model. Two datasets $D$ and $D'$ from a database $\mathcal{D}$ are \emph{adjacent}, written $D \sim D'$, if they differ in a single individual's data. The \emph{sensitivity} of the reward model is $\Delta_r = \sup_{x,y} \sup_{D \sim D'} |\hat{r}_D(x,y) - \hat{r}_{D'}(x,y)|$, which quantifies the maximum influence of any one training point on any reward score.

\begin{definition}
\label{def:privacy}
A randomized algorithm $\mathcal{M}:\mathcal{D} \to \mathcal{S}$ satisfies:
\begin{itemize}[leftmargin=*, itemsep=2pt]
\item $\varepsilon$-DP if $\,\Pr[\mathcal{M}(D) = s] \leq e^\varepsilon \Pr[\mathcal{M}(D') = s]\,$ for all $s \in \mathcal{S}$ and all $D \sim D'$.
\item $(\varepsilon,\delta)$-DP if $\,\Pr[\mathcal{M}(D) = s] \leq e^\varepsilon \Pr[\mathcal{M}(D') = s] +\delta\,$ for all $s \in \mathcal{S}$ and all $D \sim D'$.
\item $\varepsilon_p$-ex-post DP if $\,\exists\,$ a function $\varepsilon_p:\mathcal{D} \to \mathcal{S}$ such that $\,\Pr[\mathcal{M}(D) = s] \leq e^{\varepsilon_p(s)} \Pr[\mathcal{M}(D') = s]\,$ for all $s \in \mathcal{S}$ and all $D \sim D'$.
\end{itemize}
\end{definition}
Traditional ex-ante privacy~\citep{dwork2014algorithmic} bounds the guarantee
worst-case over outcomes, before the mechanism's randomness is realized; ex-post
privacy~\citep{lebensold2024privacy} bounds it conditional on the realized output,
allowing queries with favorable outcomes to ``save'' budget for later queries under
a fixed total.





\section{PrivBoN: Private Best-of-N}
\label{sec:mbon}

We introduce the PrivBoN (Private Best-of-N) framework, which adds calibrated noise to reward scores before selection. The mechanism is simple: sample $n$ candidates $y_1,\ldots,y_n$ independently from $\pi_0(\cdot|x)$, perturb each reward score $\hat r_D(x,y_i)$ with independent noise $g_i$, and return the candidate with the highest noisy reward. The noise distribution determines the privacy guarantee and the induced policy. We use the
$\mathrm{Gumbel}(0,\sigma)$ distribution, which has the PDF $f(x)=\frac{1}{\sigma}\exp\!\left(-x/\sigma-\exp(-x/\sigma)\right)$ and the CDF $F(x)=\exp\!\left(-\exp(-x/\sigma)\right)$.

\begin{algorithm}[t]
\caption{PrivBoN: Private Best-of-$N$}
\label{alg:mbon}
\begin{algorithmic}[1]
\REQUIRE Prompt $x$, base policy $\pi_0$, reward model $\hat{r}_D$, number of candidates $n$, noise scale $\sigma > 0$
\FOR{$i=1$ to $n$}
    \STATE Sample candidate $y_i \sim \pi_0(\cdot|x)$ and noise $g_i \sim \mathrm{Gumbel}(0,\sigma)$ 
    \STATE Compute noisy reward $\tilde r_i \leftarrow \hat{r}_D(x,y_i) + g_i$
\ENDFOR
\RETURN $y_{i^*}$, where $i^*={\argmax_i \tilde r_i}$
\end{algorithmic}
\end{algorithm} 


\begin{theorem}[PrivBoN privacy]
\label{thm:gumbel_priv}
Let $\hat{r}_D(x,y)$ be a reward model trained on a dataset $D$ with sensitivity $\Delta_r$. Then,
\Cref{alg:mbon} with $\sigma = 2\Delta_r/\varepsilon$ satisfies $\varepsilon$-DP.
\end{theorem}
The result follows from a classical connection: adding Gumbel noise and taking the argmax is equivalent to sampling from the softmax distribution (thanks to the Gumbel-max trick), which is precisely the exponential mechanism of McSherry and Talwar~\citep{mcsherry2007mechanism}. It also yields a clean, closed-form expression for the induced policy, given below.

\begin{proposition}[PrivBoN policy]
\label{prop:gumbel_policy}
PrivBoN with Gumbel noise induces the policy
\begin{equation*}
\pi_\mathrm{PrivBoN}(y|x) = n\,\pi_0(y|x)\,\E_{y_{2:n}\sim\pi_0(\cdot |x)}\!\left[\frac{e^{\hat{r}_D(x,y)/\sigma}}{\sum_{j=1}^n e^{\hat{r}_D(x,y_j)/\sigma}}\right],
\end{equation*}
where $y_1 = y$ is fixed and $y_{2:n}$ are drawn independently from $\pi_0$. Assuming the reward distribution under $\pi_0$ is continuous, (a) it recovers the standard BoN policy as $\sigma \to 0$ and (b) it converges pointwise to the KL-regularized optimum $\pi_\infty(y|x) \propto \pi_0(y|x)\exp(\hat{r}_D(x,y)/\sigma)$ as $n \to \infty$.
\end{proposition}
Part~(b) is the crucial structural property: PrivBoN implements, in the large-$n$
limit, the KL-regularized optimum $\max_\pi \E_{y \sim \pi}[\hat{r}_D(x,y)] -
\sigma\KL(\pi \| \pi_0)$. The noise scale $\sigma$ simultaneously controls the
privacy budget ($\varepsilon = 2\Delta_r/\sigma$) and the KL-regularization
strength---raising the question of whether the two impose compatible demands on $\sigma$.

\subsection{Regret guarantee}
\label{sec:regret}


We now analyze how the noise required for privacy affects alignment quality within
the regret framework of~\citet{huang2025bon}. Our main result is that
privacy-mandated noise implements exactly the regularization needed for optimal
alignment whenever the privacy budget exceeds a threshold $\varepsilon^*$.


\begin{theorem}[PrivBoN regret]
\label{thm:gumbel_regret}

Let $r^*(x,y) \in [0, R_{\max}]$ be the true reward and $\hat{r}_D(x,y)$ be a reward model trained on a dataset $D$. Let the reward model error $\epsilon(x,y) \sim \mathcal{N}(0,\sigma^2_r(x))$ under the base policy $\pi_0$. Let $\pi^*$ be any comparator policy with finite coverage $C^{\pi^*}(x) < \infty$. Then, for $n = \widetilde \Omega \left( e^{(R_{\max} + \varepsilon_{\mathrm{RM}})/\sigma}\log(R_{\max}/\varepsilon_{\mathrm{RM}}\right)$, the regret of PrivBoN satisfies
\begin{equation}
\label{eq:gumbel_reg}
\mathrm{Reg}(\pi_\mathrm{PrivBoN};x) \lesssim \underbrace{\sigma C^{\pi^*}(x)}_{\textup{KL-bias}} + \underbrace{\varepsilon_{\mathrm{RM}}^2(x)/\sigma}_{\textup{overoptimization}} + \underbrace{\sqrt{C^{\pi^*}(x)}\,\varepsilon_{\mathrm{RM}}(x)}_{\textup{irreducible}} + \underbrace{\frac{R_{\max}e^{\frac{R_{\max} + \varepsilon_{\mathrm{RM}}\sqrt{\log n}}{\sigma}}}{\sqrt{n}}}_{\textup{finite sample}}.
\end{equation}
where $\varepsilon_{\mathrm{RM}}^2(x) = \E_{y \sim \pi_0(\cdot|x)}[\epsilon(x,y)^2]=\sigma_r^2(x)$ is the reward model variance.
\end{theorem}

\textbf{Optimal noise scale.} The third term is the skyline itself. The first two terms exhibit a bias-variance tradeoff in $\sigma$: a large $\sigma$ (strong regularization, strong privacy) incurs bias from under-exploiting the reward signal, whereas a small $\sigma$ (weak regularization, weak privacy) incurs overoptimization from exploiting reward model errors. Balancing the first two terms yields $\sigma^*(x) = \varepsilon_{\mathrm{RM}}(x)/\sqrt{C^{\pi^*}(x)}$, at which the regret,
\begin{equation*}
\mathrm{Reg}(\pi_\mathrm{PrivBoN};x) \leq 3\sqrt{C^{\pi^*}(x)\varepsilon_{\mathrm{RM}}^2(x)}~,
\end{equation*}
matches the information-theoretic skyline (Proposition~\ref{prop:skyline}) in the large-$n$ limit.

\textbf{``Privacy is free'' threshold.} When $\sigma$ is set by the privacy budget ($\sigma = 2\Delta_r/\varepsilon$), the regret becomes
\begin{equation}\label{eq:privacy-free-reg}
\mathrm{Reg}(\pi_\mathrm{PrivBoN};x) \leq \frac{2\Delta_r C^{\pi^*}(x)}{\varepsilon} + \frac{\varepsilon \cdot \varepsilon_{\mathrm{RM}}^2(x)}{2\Delta_r} + \sqrt{C^{\pi^*}(x)}\varepsilon_{\mathrm{RM}}(x).
\end{equation}
The privacy-mandated noise $\sigma$ is no larger than the regret-optimal noise $\sigma^*(x)$ precisely when
\begin{equation}
\label{eq:eps_star}
\varepsilon \geq \varepsilon^*(x) = \frac{2\Delta_r\sqrt{C^{\pi^*}(x)}}{\varepsilon_{\mathrm{RM}}(x)}.
\end{equation}
For all $\varepsilon \geq \varepsilon^*(x)$, the regret matches the skyline: the noise required for privacy is no more than what alignment already needs.
E.g., for a reward model with $\varepsilon_{\mathrm{RM}} = 0.35, \Delta_r = 0.1, C^{\pi^*} = 2$, this gives $\varepsilon^* \approx 0.8$, i.e., any privacy budget above
0.8 is free.
Below $\epsilon < \varepsilon^*(x)$, the first term in~\eqref{eq:privacy-free-reg} dominates and the regret degrades gracefully as $O\left(\frac{\Delta_rC^{\pi^*}(x)}{\varepsilon}+\sqrt{C^{\pi^*}(x)}\varepsilon_{\mathrm{RM}}(x)\right)$.

The regret bound~\eqref{eq:gumbel_reg} goes down with $n$ for $n \geq n_0(\sigma, R_{\max}, \epsilon_{\mathrm{RM}})$, confirming that \emph{more candidates never worsen the upper bound}, which rules out the catastrophic reward-hacking regime of BoN where the bound degrades with an increase in $n$. The finite-sample term, however, carries an $R_{\max} e^{(R_{\max}+\epsilon_{\mathrm{RM}}\sqrt{\log n})/\sigma}/\sqrt{n}$ prefactor: for typical values ($R_{\max}=5$, $\sigma=1$) the asymptotic rate kicks in only for $n \gtrsim 10^4$, and for moderate $n$ the regret behaves as $O(R_{\max}e^{R_{\max}/\sigma}/\sqrt{n})$. PrivITP (Section~\ref{sec:privitp}) avoids this exponential dependence via rejection sampling.

\begin{remark}[Privacy-mandated noise as optimal regularization]
\label{rem:privacy_free}
The coincidence $\sigma = \sigma^*(x)$ at $\varepsilon = \varepsilon^*(x)$ is
structural: both privacy and alignment require softening the $\argmax$ to prevent
overfitting to the proxy reward, and both implement this softening through the
same mechanism---Gumbel noise of scale $\sigma$, equivalent via the Gumbel-max
trick~\citep{mcsherry2007mechanism} to softmax with temperature $\sigma$. Privacy
imposes $\sigma \geq 2\Delta_r/\varepsilon$; alignment optimum requires
$\sigma^*(x) \leq \varepsilon_{\mathrm{RM}}(x)/\sqrt{C^{\pi^*}(x)}$. The two are compatible
whenever $\varepsilon \geq \varepsilon^*(x)$.
\end{remark}


\begin{remark}[Comparison with prior analyses]
\label{rem:softbon-comparison}
SoftBoN~\citep{aminian2025smoothing, khalaf2025inference} is algorithmically
equivalent to PrivBoN but analyzed differently: \citet{aminian2025smoothing} bound the gap
within the SoftBoN policy class at a fixed temperature, whereas we bound regret against an
arbitrary comparator $\pi^*$ and match the skyline of~\citet{huang2025bon}; \citet{khalaf2025inference}
treat temperature as a free tuning parameter, whereas we derive it from the privacy budget,
revealing the threshold $\epsilon^*(x)$ at which privacy becomes free.
\end{remark}
The Gaussian-error assumption yields the closed-form threshold $\epsilon^*(x)$; under a general MSE assumption, the qualitative conclusions persist, but the threshold loses its closed form. See Appendix~\ref{app:gaussian-assumption}.

\subsection{Reward hacking robustness}
\label{sec:hacking}

\Cref{thm:gumbel_regret} establishes that PrivBoN achieves low regret. A separate
and complementary question is whether PrivBoN prevents \emph{reward hacking}---the
pathology where the selected response exploits errors in the proxy reward, yielding
high $\hat r_D$ but low $r^*$~\citep{gao2023scaling,khalaf2025inference}. We quantify
this through two measures: the win-rate against the base policy and the hacking gap
$\mathbb{E}_{\hat\pi}[\epsilon(x,y)]$, which directly measures the extent to which
the selected response overfits to reward-model errors. For a given prompt $x$, the win-rate~\citep{gui2024bonbon} of a policy $\hat\pi$
against $\pi_0$ is
\[
    p_{\hat\pi \succ \pi_0 \mid x} = \Pr_{Y \sim \hat\pi,\, Y_0 \sim \pi_0}
    \bigl(\hat r_D(x,Y) \geq \hat r_D(x,Y_0)\bigr).
\]
High win-rate does not imply low regret when $\hat r_D$ and $r^*$ are mismatched:
BoN achieves win-rate $n/(n+1)$~\citep{beirami2024theoretical} while
simultaneously incurring regret $\widetilde\Omega(\sqrt{\varepsilon_{\mathrm{RM}}^2 n})$~\citep{huang2025bon}.
The two quantities diverge precisely because BoN's argmax overfits to $\hat r_D$'s
errors. The following proposition shows that PrivBoN's softmax selection caps both
quantities at finite limits, eliminating this pathology.

\begin{proposition}[Win-rate saturation and bounded hacking]
\label{prop:hacking}
Let $\epsilon(x,y) \sim \mathcal{N}(0, \sigma_r^2(x))$ under $\pi_0$. As $n \to \infty$, we have (a) BoN's win-rate converges to $1$, whereas PrivBoN's saturates at
$\Phi(\sigma_r(x)/(\sigma\sqrt{2})) < 1$ and (b) BoN's hacking gap grows without bound as
    $O(\sigma_r(x)\sqrt{\log n})$, while PrivBoN's converges to the finite limit
    $\sigma_r^2(x)/\sigma$.
\end{proposition}
The finite saturation limits reflect the same mechanism: the privacy/regularization
noise $\sigma$ prevents the selection from fully exploiting the reward signal, even
with unlimited candidates. Consequently, there exists a crossover $n^*(\sigma, \sigma_r)$
beyond which PrivBoN achieves strictly higher \emph{true} reward than BoN -- BoN's
expected true reward eventually decreases with $n$ (Goodhart's law), while PrivBoN's
converges monotonically to $\mathbb{E}_{\pi_\infty}[r^*]$. For typical reward models (e.g., Oasst on GSM8K dataset),
$n^*$ falls well below $100$ (see \Cref{fig:main_all}(top row)). Proofs are deferred to Appendix~\ref{app:proofs_bon}.

\section{Private Inference Time Pessimism}
\label{sec:privitp}

PrivBoN's threshold $\varepsilon^*(x)$ depends on the unknown
coverage coefficient $C^{\pi^*}(x)$ \eqref{eq:eps_star}, so a practitioner cannot verify whether their
privacy budget falls in the free regime. Private InferenceTimePessimism (PrivITP)
resolves this by combining the $\chi^2$-regularized rejection sampling
of~\citet{huang2025bon} with calibrated noise injection, decoupling reward-hacking
mitigation (controlled by the regularization) from privacy (controlled by the noise). 


\textbf{Algorithm.}
PrivITP operates in two phases, each consuming independent noise. Phase~1 draws
$n$ candidates from $\pi_0(\cdot|x)$ and privately releases the normalization
constant $\tilde \lambda$ of the $\chi^2$-regularized policy by adding Gaussian
noise to the solution $\lambda$ of the empirical normalization equation (we omit dependence on $x$ for brevity).
Phase~2 draws a fresh batch of $n$ candidates, perturbs their rewards with
Gaussian noise, and performs rejection sampling with an acceptance threshold
$\tilde\lambda$ from Phase~1. Independence of the two batches is essential
for both the privacy and regret analyses; pseudocode appears
in~\Cref{alg:privitp}. The ReLU in the acceptance weights implements pessimism: candidates whose noisy
reward falls below $\tilde\lambda$ receive zero weight, discarding responses that
the reward model does not confidently rank above the population average. This
mechanism prevents reward hacking independently of the injected noise, so privacy
and alignment can be tuned separately (\Cref{thm:privitp_regret}).

\begin{algorithm}[t]
\caption{Private InferenceTimePessimism (PrivITP)}
\label{alg:privitp}
\begin{algorithmic}[1]
\REQUIRE Prompt $x$, base policy $\pi_0$, reward model $\hat{r}_D$, candidates $n$, regularization $\beta > 0$, noise scales $\sigma_X,\sigma_Z > 0$, truncation parameter $L$
\STATE \textbf{Phase 1 (private normalization):}
\STATE \quad Draw $y_1,\ldots,y_n \sim \pi_0(\cdot|x)$; compute $ r_i = \hat{r}_D(x,y_i)$
\STATE \quad Find $\lambda$ such that $\frac{1}{n}\sum_{i=1}^n \mathrm{relu}\big(\beta^{-1}( r_i - \lambda)\big) = 1$
\STATE \quad Set $\tilde \lambda = \lambda + \mathcal{N}(0,\sigma^2_X)$ and $M = (R_{\max} + \sigma_Z L - \tilde\lambda)/\beta$
\STATE \textbf{Phase 2 (private rejection sampling):}
\STATE \quad Draw fresh $y_1',\ldots,y_n' \sim \pi_0(\cdot|x)$; compute $\tilde r_i = \hat{r}_D(x,y_i') + \mathcal{N}(0,\sigma_Z^2)$
\FOR{$i=1$ to $n$}
    \STATE $w_i \leftarrow \mathrm{relu}\big(\beta^{-1}(\tilde r_i - \tilde\lambda\big)$; \quad sample $\xi_i \sim \mathrm{Ber}(\min(w_i/M, 1))$
    \IF{$\xi_i = 1$} \RETURN $y_i'$ 
    \ENDIF
\ENDFOR
\RETURN $y_{n+1}' \sim \pi_0(\cdot|x)$
\end{algorithmic}
\end{algorithm}

\textbf{Privacy guarantee.} Each phase accesses the dataset $D$ through reward evaluations. The injected Gaussian noise in the normalizer in Phase~1 and in the rewards in Phase~2 ensures that the released quantities ($\tilde{\lambda}$ from Phase~1 and the selected response from Phase~2) are differentially private.

\begin{theorem}[PrivITP privacy]
\label{thm:privitp_priv}
Let $\hat{r}_D(x, \cdot) \in [0, R_{\max}]$ with sensitivity $\Delta_r$. \Cref{alg:privitp} with Phase~1 scalar Gaussian noise $\sigma_X$ and Phase~2 query noise $\sigma_Z$ satisfies $(\varepsilon_1 + \varepsilon_2^{\mathrm{post}}(t),\, \delta)$-DP \emph{ex-post}, where \textbf{Phase~1} contributes $(\varepsilon_1, \delta)$-DP with $\varepsilon_1 = \Delta_r\sqrt{2\ln(1.25/\delta)}/\sigma_X$ and 
\textbf{Phase~2} contributes \emph{pure} ex-post $\varepsilon_2^{\mathrm{post}}(t)$-DP, where $t$ is the halting time (index of the first accepted candidate):
\begin{equation*}
\label{eq:privitp_expost}
\varepsilon_2^{\mathrm{post}}(t) = (t-1)\,\log\frac{\displaystyle\E_{u \sim \mathrm{Unif}[0,1]}\!\!\left[\Phi\!\left(\tfrac{\tilde{\lambda} + \beta M u - R_{\max} + \Delta_r}{\sigma_Z}\right)\right]}{\displaystyle\E_{u \sim \mathrm{Unif}[0,1]}\!\!\left[\Phi\!\left(\tfrac{\tilde{\lambda} + \beta M u - R_{\max}}{\sigma_Z}\right)\right]} + \log\frac{\displaystyle\E_{u \sim \mathrm{Unif}[0,1]}\!\!\left[\Phi\!\left(\tfrac{\Delta_r - \tilde{\lambda} - \beta M u}{\sigma_Z}\right)\right]}{\displaystyle\E_{u \sim \mathrm{Unif}[0,1]}\!\!\left[\Phi\!\left(\tfrac{-\tilde{\lambda} - \beta M u}{\sigma_Z}\right)\right]}.
\end{equation*}
\end{theorem}
Unlike standard privacy composition over $n$ candidates (which gives a worst-case bound proportional to $n$), \Cref{thm:privitp_priv} gives a bound, which is independent of $n$, computable via one-dimensional numerical integration and proportional to the \emph{actual halting time} $t$. The coefficient of the linear term,
\[
\kappa := \log\frac{\E_u[\Phi((\tilde{\lambda} + \beta M u - R_{\max} + \Delta_r)/\sigma_Z)]}{\E_u[\Phi((\tilde{\lambda} + \beta M u - R_{\max})/\sigma_Z)]},
\]
is the per-round privacy cost of rejection, which depends on the relative positions of $\tilde{\lambda}$, $R_{\max}$, and $\sigma_Z$. When $\sigma_Z$ is large relative to $\Delta_r$ (strong privacy), $\kappa$ is small; when $\sigma_Z$ is comparable to $\Delta_r$, $\kappa$ dominates the privacy budget. The bound requires the reward model's outputs to lie in a bounded range $[0, R_{\max}]$, which is a mild assumption in practice as reward models for language alignment are typically trained with bounded output heads (e.g., sigmoid or clipped outputs)~\cite{huang2025bon}.

\begin{remark}[Ex-post vs.~ex-ante accounting]
\label{rem:ex-post-vs-ex-ante}
Theorem~\ref{thm:privitp_priv} gives an ex-post guarantee: the privacy cost depends on
the realized halting time $t$. For worst-case accounting, $\epsilon_2^{\mathrm{post}}(n)$
upper-bounds the per-query cost. In practice, rejection happens quickly when the reward
model is informative ($\mathbb{E}[t] = O(1)$), so the ex-post bound is substantially tighter
and enables the composition of Section~\ref{sec:composition}.
\end{remark} 

\textbf{Regret guarantee.} We now analyze the regret of PrivITP against an arbitrary comparator policy $\pi^*$. The central result is that PrivITP achieves the information-theoretic skyline (Proposition~\ref{prop:skyline})with a noise inflation and it does so while satisfying pure ex-post $\varepsilon$-DP (\Cref{thm:privitp_priv}).

\begin{theorem}[PrivITP regret]
\label{thm:privitp_regret}
Let $r^*(x,y) \in [0, R_{\max}]$ be the true (unknown) reward and $\hat{r}_D \colon \mathcal{X} \times \mathcal{Y} \to [0,R_{\max}]$ be a reward model trained on a dataset $D$ with mean-squared error $\varepsilon_{\mathrm{RM}}^2(x) = \E_{y \sim \pi_0(x)}[(\hat{r}_D(x,y)-r^*(x,y))^2]$.
Let $\pi^*$ be any comparator policy with finite coverage $C^{\pi^*}(x) < \infty$. For regularization $\beta \ge 2\sigma_X$, truncation parameter $L=O(\log n)$ and $n = \max \left\lbrace\Omega\!\left(\frac{R_{\max}}{\beta}\log\!\left(\frac{R_{\max}}{\beta \,\epsilon_{\mathrm{RM}}}\right)\right), \widetilde \Omega\left(\frac{R_{\max}+\sigma_Z}{\beta}\right)\right\rbrace$, PrivITP (\Cref{alg:privitp}) satisfies
\begin{equation}
\label{eq:privitp_regret_general}
\mathrm{Reg}(\pi_{\mathrm{PrivITP}};x) \lesssim \underbrace{\beta \cdot C^{\pi^*}(x)}_{\chi^2 \text{ bias}} + \underbrace{\beta^{-1}(\varepsilon_{\mathrm{RM}}^2(x) + R_{\max}\sigma)}_{\text{overoptimization+noise}}+ \underbrace{\sqrt{C^{\pi^*}(x)}\varepsilon_{\mathrm{RM}}(x)}_{\text{irreducible}} +\underbrace{R_{\max}/\sqrt{n}}_{\text{finite sample}},
\end{equation}
where $\sigma = \sigma_X+\sigma_Z$ denotes the total noise level.
\end{theorem}
As with PrivBoN, the bound in~\eqref{eq:privitp_regret_general} is non-increasing in $n$ for $n \geq n_0(\beta, \sigma_Z, R_{\max}, \epsilon_{\mathrm{RM}})$, ruling out the reward-hacking regime where adding candidates provably degrades performance. 

\textbf{Optimal regularization.} The first two terms exhibit a bias-variance tradeoff in $\beta$: a large $\sigma$ (strong regularization) incurs bias from under-exploiting the reward signal, whereas a small $\beta$ (weak regularization) incurs overoptimization from exploiting reward model errors.
Balancing the first two terms yields $\beta^*(x) = \sqrt{(\varepsilon_{\mathrm{RM}}^2(x)+R_{\max}\sigma)/C^{\pi^*}(x)}$, at which the regret,
\begin{equation*}
\label{eq:privitp_regret_opt}
\mathrm{Reg}(\pi_{\mathrm{PrivITP}};x) \leq 3\sqrt{C^{\pi^*}(x)\big(\varepsilon_{\mathrm{RM}}^2(x) + R_{\max}\sigma\big)}~,
\end{equation*}
matches the information-theoretic skyline (proposition~\ref{prop:skyline}) in the large-$n$ limit (up to the noise inflation due to privacy). Proofs are deferred to Appendix~\ref{app:proofs_itp}.
\begin{remark}[Compatibility of $\beta^* \geq 2\sigma_X$]
\label{rem:beta-sigma-compat}
The condition $\beta \geq 2\sigma_X$ in Theorem~\ref{thm:privitp_regret} is satisfied by $\beta^*(x)$ whenever $\sigma_X \leq \tfrac{1}{2}\sqrt{(\epsilon_{\mathrm{RM}}^2 + R_{\max}\sigma)/C^{\pi^*}(x)}$, a mild bound translating to a lower bound on $\epsilon_1$ in the privacy-calibrated regime. When violated, setting $\beta = 2\sigma_X \geq \beta^*(x)$ recovers the bound with an additive $2\sigma_X C^{\pi^*}(x)$ penalty in the $\chi^2$-bias term.
\end{remark}
\textbf{Privacy-mandated noise.} The Phase~2 query noise $\sigma_Z$ governs the pure ex-post pure-DP guarantee of \Cref{thm:privitp_priv}: smaller $\sigma_Z$ yields larger per-query privacy cost $\varepsilon_2^{\mathrm{post}}(t)$. To quantify the trade-off, we use the order-of-magnitude relation obtained by linearizing~\eqref{eq:privitp_expost} in the small-$\Delta_r/\sigma_Z$ regime: the per-round log-ratio is of order $\Delta_r/\sigma_Z$, giving $\varepsilon_2^{\mathrm{post}}(t) \asymp t\cdot\Delta_r/\sigma_Z$, or equivalently
$\sigma_Z \asymp t\cdot\Delta_r/\varepsilon_2^{\mathrm{post}}(t)$.
For Phase 1 per-query privacy cost $(\varepsilon_1,\delta)$, we set noise $\sigma_X \asymp \Delta_r\sqrt{\ln(1/\delta)}/\varepsilon_1$.
Substituting into~\eqref{eq:privitp_regret_opt} yields the privacy-mandated regret bound
\begin{equation}
\label{eq:privitp_regret_priv}
\mathrm{Reg}(\pi_{\mathrm{PrivITP}};x) \leq 3\sqrt{C^{\pi^*}(x)\!\left(\varepsilon_{\mathrm{RM}}^2(x) + \frac{R_{\max}\cdot t\Delta_r}{\varepsilon_2^{\mathrm{post}}(t)}+ \frac{R_{\max}\Delta_r\sqrt{\ln(1/\delta)}}{\epsilon_1}\right)}.
\end{equation}
\textbf{On the relationship to prior work.}
PrivITP combines the $\chi^2$-regularized rejection sampling of \citet{huang2025bon} with the Gaussian above-threshold mechanism of \citet{lebensold2024privacy}. The technical novelty lies in the analysis of privacy and regret: (i) a reformulation lemma mapping PrivITP's ReLU acceptance weights onto a randomized-threshold mechanism, enabling an $n$-independent ex-post DP bound (Lemma~\ref{lem:reformulation}); (ii) a smoothed-weight policy comparison lemma unifying regret analysis across threshold noise, query noise, and finite-sample effects (Lemma~\ref{lem:smooth_weight}); and (iii) a regret-optimality proof (Theorem~\ref{thm:privitp_regret}) showing PrivITP matches \citet{huang2025bon} skyline up to an $R_{\max}\sigma$ noise-inflation term that cleanly separates alignment error from privacy cost. 

\subsection{Multi-Query Deployment}
\label{sec:composition}

In deployment, an aligned language model answers a stream of $T$ prompts and produces a transcript of responses $(y^{*(1)}, \ldots, y^{*(T)})$. Each query consumes some privacy budget; the question is how to allocate this budget across queries to maximize alignment quality over the lifetime of the deployment. PrivITP's ex-post privacy guarantee (\Cref{thm:privitp_priv}) provides a structural advantage: the privacy cost of each query depends on the actual halting time $t$, not on the worst-case budget $n$. When rejection sampling accepts early -- the typical case for informative reward models, where $\E[t] = O(1)$ --  the per-query privacy cost is much smaller than the worst-case bound. The FSRC framework of~\citet[Alg.~2]{lebensold2024privacy} is designed to exploit this structure: it composes mechanisms with ex-post privacy guarantees and halts when the cumulative privacy cost exceeds the total budget.

\begin{proposition}[PrivITP composition via FSRC]
\label{thm:privitp_fsrc_comp}
\Cref{alg:privitp_fsrc} satisfies $(\varepsilon_{\mathrm{total}}, \delta)$-DP under fully-adaptive query selection. The expected number of queries answered before halting is $\E[T^*] \geq \frac{\varepsilon_{\mathrm{total}}}{\E[\varepsilon_\tau]}$,
where $\E[\varepsilon_\tau]$ is the expected per-query privacy cost under the distribution of halting times $t_\tau$ induced by the noisy reward distribution. 
\end{proposition}
In the favorable regime where $\E[\varepsilon_\tau] \ll \varepsilon_{\max}$, where
$\epsilon_{\max} := \epsilon_1 + \epsilon_2^{\mathrm{post}}(n)$ is the worst-case
per-query cost, FSRC-composed PrivITP answers $\Theta(\varepsilon_{\max}/\E[\varepsilon_\tau])$ times more queries than standard composition for the same total budget. FSRC strictly dominates privacy filters~\citep{rogers2023adaptive}, which must charge the worst-case per-query cost rather than the realized ex-post cost (see Remark~\ref{rem:fsrc_vs_filters})

\section{Experiments}
\label{sec:experiments}

\begin{figure}
    \centering
    \begin{subfigure}[t]{\textwidth}
        \centering
        \includegraphics[width=\textwidth]{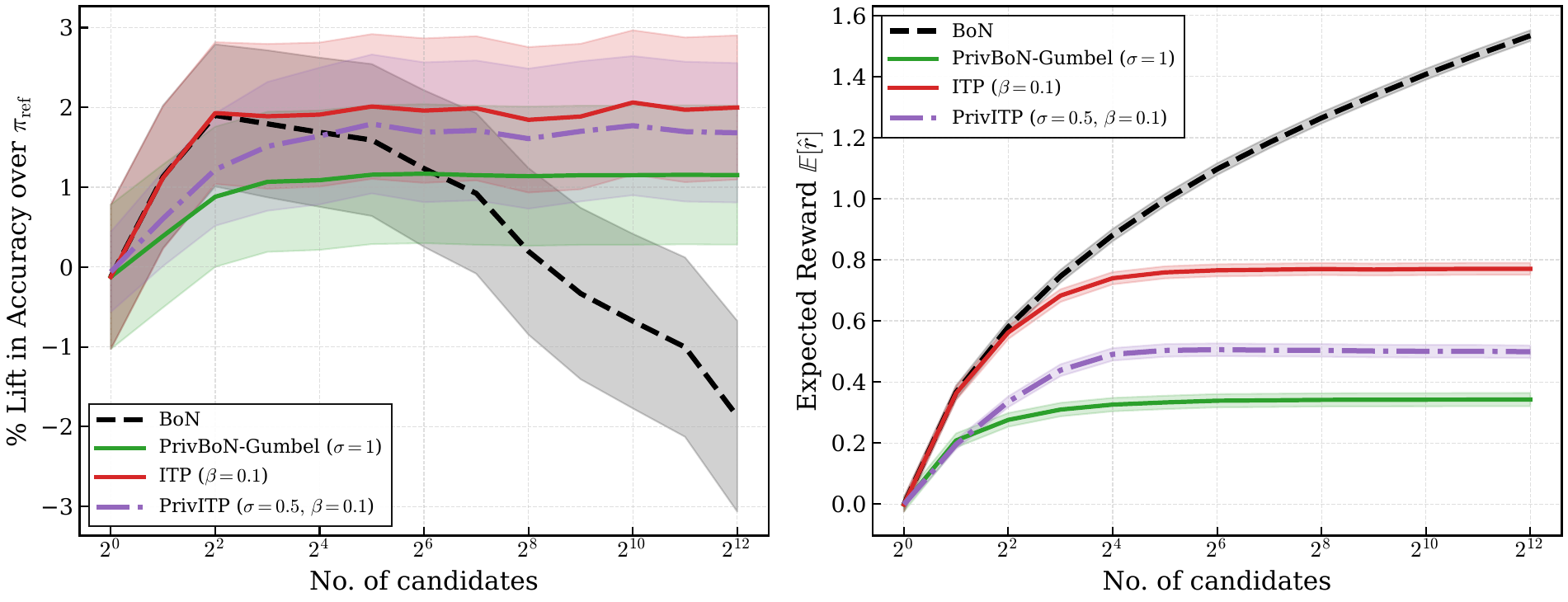}
    \end{subfigure}
    \vskip -0.5mm
    \begin{subfigure}[t]{0.49\textwidth}
        \centering
        \includegraphics[width=\textwidth]{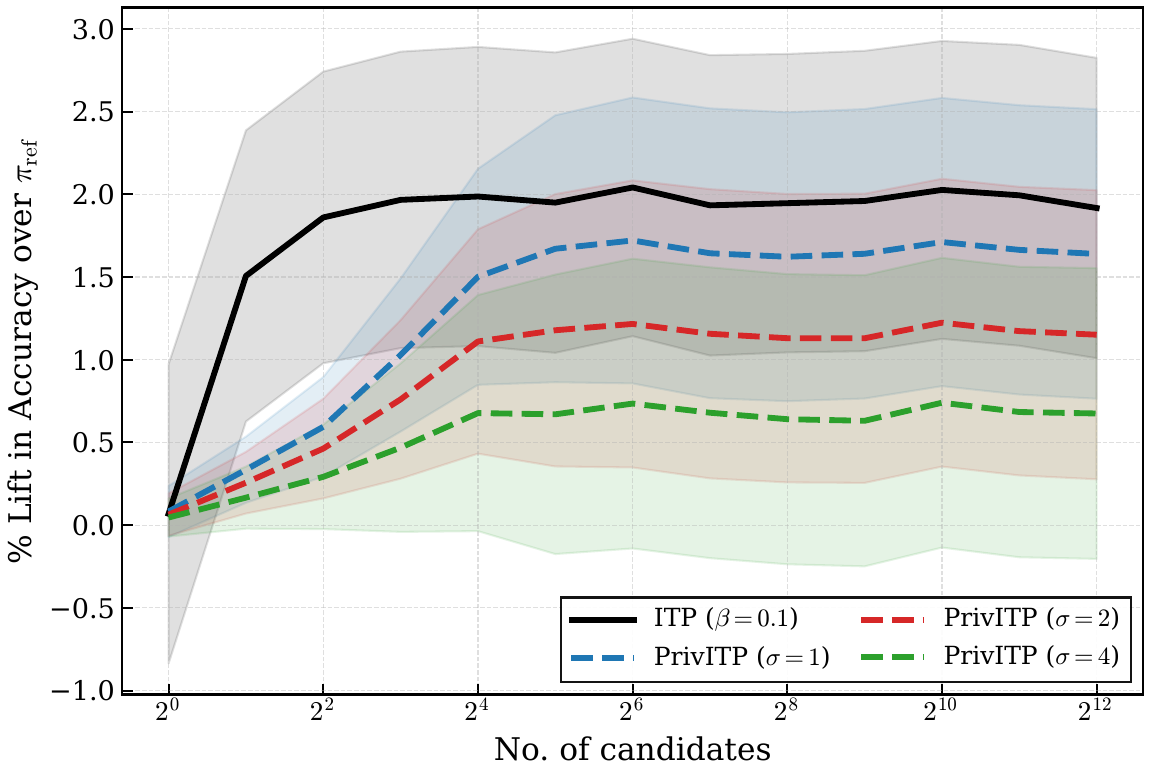}
    \end{subfigure}
     \hfill
    \begin{subfigure}[t]{0.49\textwidth}
        \centering
\includegraphics[width=\textwidth,height=4.4cm]{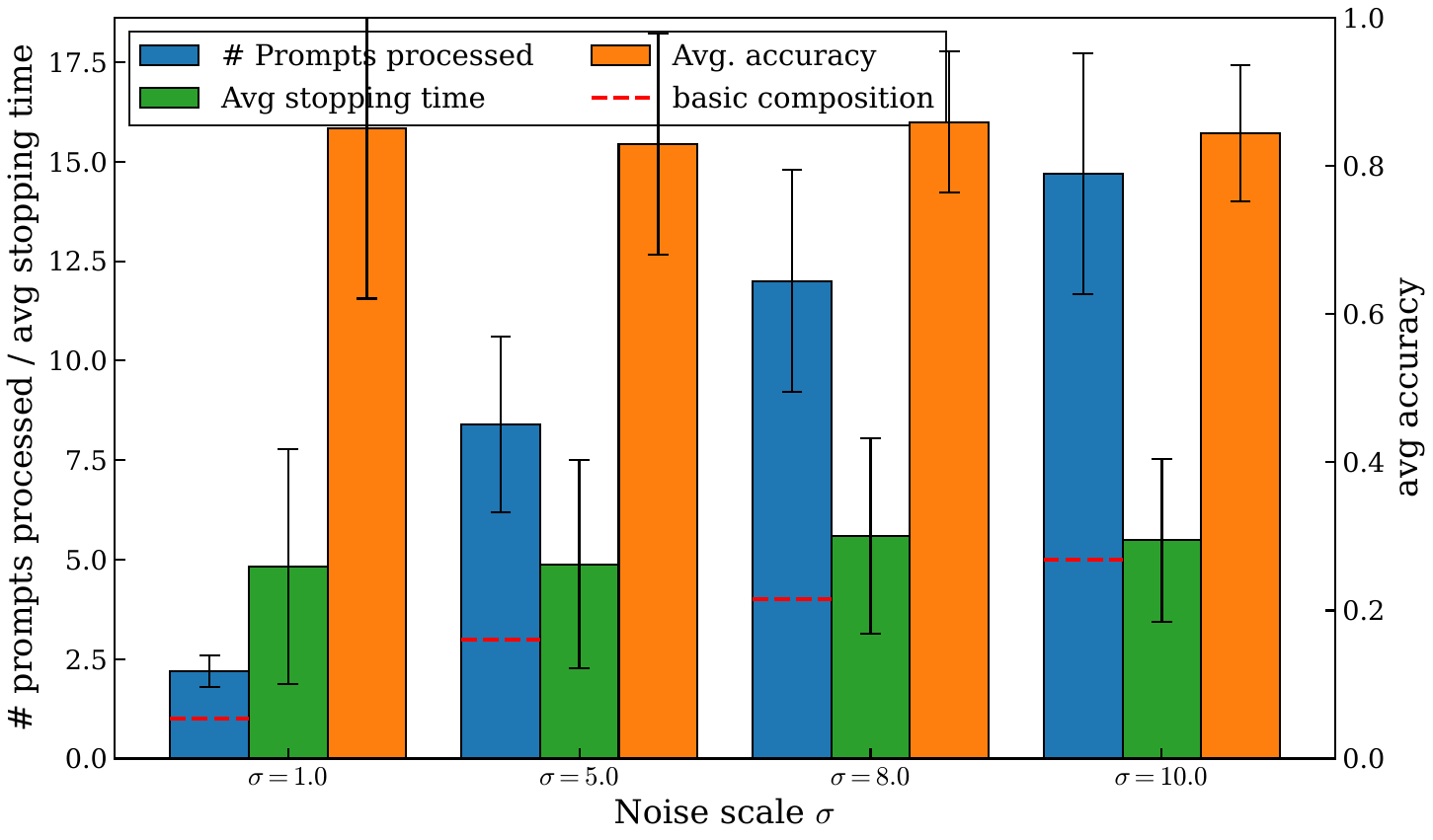}
    \end{subfigure}
    \vskip -1mm
    \caption{\footnotesize{Experiments for Phi-3-Mini-Instruct policy and Gemma-RM on the GSM8K test split. \emph{Top row} demonstrates reward hacking in BoN and its mitigation by PrivBoN, ITP, and PrivITP. \emph{Bottom left} shows privacy vs. utility trade-off in ITP: low noise (low privacy) leads to better accuracy, and vice versa. \emph{Bottom right} demonstrates the advantage of FSRC composition in PrivITP's ex-post privacy bound compared to the standard composition in ex-ante bound.}}
    \label{fig:main_all}
\end{figure}

\begin{table}[htbp]
\centering
\caption{Percentage lift in accuracy over base policy $\text{Phi-3-Mini-Instruct}$ at $n = 2^{12}$
candidates. Values are mean $\pm$ standard
error over all prompts across the GSM8K, MMLU, and MATH test splits. Same (RM, dataset) configuration uses the same noise $\sigma$ for private variants and the same regularization $\beta$ for ITP variants. 
}
\label{tab:rm_dataset_sweep_new}
\begin{tabular}{@{}llcccc@{}}
\toprule
\textbf{Dataset} & \textbf{Algorithm} & \textbf{Oasst-RM} & \textbf{Gemma-RM} & \textbf{Llama-RM} & \textbf{Armo-RM} \\
\midrule
\multirow{4}{*}{GSM8K}
 & BoN     & $-5.18 \pm 1.23$ & $-1.87 \pm 1.20$  & $7.64 \pm 0.82$ & $6.47 \pm 0.86 $ \\
 & PrivBoN & $0.32 \pm 0.91$ & $1.15 \pm 0.87$ & $6.90 \pm 0.76$ & $5.78 \pm 0.78 $ \\
 & PrivITP & $0.81 \pm 0.89$ & $1.68 \pm 0.87$ & $7.46 \pm 0.78 $ & $6.63 \pm 0.78$ \\
 \cmidrule(l){2-6}
 & ITP     & $0.85 \pm 0.90$ & $ 2.00 \pm 0.90$ & $7.61 \pm 0.79  $ & $6.81 \pm 0.79 $ \\
\midrule
\multirow{4}{*}{MMLU}
 & BoN     & $2.54 \pm 2.74$& $5.45 \pm 2.94 $& $-12.74 \pm 2.92$& $8.40 \pm 4.37$\\
 & PrivBoN & $2.32 \pm 2.22$& $4.22 \pm 2.25 $& $1.88 \pm 2.24$& $3.37 \pm 3.22$\\
 & PrivITP & $3.80 \pm 2.28 $& $6.09 \pm 2.33 $& $2.76 \pm 2.29$ & $5.59 \pm 3.40$\\
 \cmidrule(l){2-6}
 & ITP     & $5.09 \pm 2.35$&$7.35 \pm 2.46 $ & $4.69 \pm 2.40$& $8.24 \pm 3.74$\\
\midrule
\multirow{4}{*}{MATH}
 & BoN     & $-5.21 \pm 1.60$ & $2.10 \pm 1.79 $ & $10.80 \pm 1.98$ & $6.10 \pm 1.91$ \\
 & PrivBoN & $1.33 \pm 1.55$ & $2.94 \pm 1.61$ & $9.14 \pm 1.73$ & $4.51 \pm 1.63$ \\
 & PrivITP & $1.72 \pm 1.58$ & $3.38 \pm 1.62$ & $10.28 \pm 1.79$ & $7.04 \pm 1.70$ \\
 \cmidrule(l){2-6}
 & ITP     & $1.64 \pm 1.59 $ & $3.45 \pm 1.63$ & $10.68 \pm 1.81$ & $7.80 \pm 1.77$ \\
\bottomrule
\end{tabular}
\end{table}

We validate our theoretical results following the experimental setup
of~\citet{huang2025bon}: four reward models (Oasst (1.4B), Gemma (2B), Llama (3B), and Armo (7B)) and two base LMs (Phi3-Mini-Instruct and Gemma-2-2B-Instruct) evaluated on test splits of GSM8K~\citep{cobbe2021gsm8k}, MATH-500~\citep{hendrycks2021math}, and math, chemistry splits of MMLU~\citep{hendrycks2020mmlu}. For each prompt, we sample $10000$
responses from base policy at temperature~$1$ with zero-shot CoT
prompting~\citep{wei2022cot}, then bootstrap $M = 50$ replicates of each sample to run each of the four algorithms: BoN, PrivBoN, ITP, and PrivITP.
Unless otherwise mentioned, private algorithms use the same noise $\sigma$ (for PrivITP, we set $\sigma_X=\sigma_Z=\sigma/2$), and ITP variants use the same regularization $\beta$ on the same (RM, dataset) configuration. Ablation studies on the noise level $\sigma$ (equivalently, on privacy budget $\epsilon$) and regularization $\beta$ are presented in Appendix~\ref{app:experiments}.

\textbf{Scaling behavior across algorithms
(Theorems~\ref{thm:gumbel_regret},~\ref{thm:privitp_regret}).}
\Cref{fig:main_all}(top row) shows \% percentage
lift in accuracy over base policy Phi3-Mini and expected proxy reward under Gemma-RM for GSM8K as we vary $n$. We observe (i)~BoN's accuracy rises then falls with $n$, exhibiting reward
hacking, while its proxy reward rises monotonically---the signature of overfitting to
$\hat r_\mathcal{D}$; (ii)~PrivBoN and PrivITP accuracies are monotone in $n$, confirming hacking mitigation; (iii)~PrivITP matches or exceeds PrivBoN at the same noise level. \Cref{fig:main_all}(bottom left) shows that PrivITP
approaches the non-private ITP skyline as noise (and hence, privacy) decreases. Figures for other RMs are in Appendix~\ref{app:experiments}.

\textbf{Generalization across RMs and datasets.}
At $n = 2^{12}$, we report the percentage lift in accuracy over Phi3-Mini-Instruct base policy for each (algorithm, RM, dataset) configuration in
\Cref{tab:rm_dataset_sweep}. PrivITP dominates PrivBoN on every
(RM, dataset) configuration, and on GSM8K it recovers between 84\% and 98\%
of the non-private ITP skyline across all four reward models. On MMLU and
MATH, PrivITP again closely tracks ITP within standard error. The reward-hacking
signature in BoN is most pronounced for the weaker RMs (Oasst, Gemma), while PrivITP recovers
near-skyline performance in each case. For stronger RMs
(Llama, Armo) on GSM8K and MATH, BoN remains competitive at $n = 2^{12}$, which indicates that reward-hacking pressure is
delayed when the base policy has lower accuracy. Results for Gemma2-2B-Instruct base model are in Appendix~\ref{app:experiments}.



\textbf{Multi-query deployment via FSRC
(Proposition~\ref{thm:privitp_fsrc_comp}).}
To validate the composition advantage of PrivITP's ex-post structure, we compare the number of GSM8K queries answered before budget exhaustion under
(a)~FSRC-composed PrivITP, (b)~PrivITP basic composition (worst-case per-query
cost $\varepsilon_{\max} = \varepsilon_1 + \varepsilon_2^{\mathrm{post}}(n)$) sweeping over 
noise $\sigma$. We fix $n=16$, $\beta=0.05$ and total privacy budget $\varepsilon_{\text{total}} = 50$, 
Under Gemma-RM and across 10 random seeds, \Cref{fig:main_all}(bottom right) shows that FSRC answers $\sim\!3\times$ more
queries, with empirical halting time $t_\tau \ll n$,
confirming Remark~\ref{rem:ex-post-vs-ex-ante}. Similar behavior is observed for other RMs, and plots are in Appendix~\ref{app:experiments}.




\textbf{Conclusion.} We have shown that differential privacy and reward-hacking mitigation are two instances of
the same intervention: softening the hard argmax in Best-of-N selection. For PrivBoN, the privacy-mandated scale coincides with the regret-optimal KL temperature above a threshold
$\varepsilon^*$, making privacy ``free''; PrivITP extends this to $\chi^2$-regularized rejection
sampling, decoupling the privacy and regularization parameters, and yielding ex-post DP with
$n$-independent composition. 
Open directions include extending to sequential test-time alignment~\citep{yu2025limits} and
connecting inference-time privacy to training-time private alignment~\citep{yu2021differentially}
for end-to-end guarantees.

\bibliographystyle{plainnat}
\bibliography{references}

\appendix
\section{Related work}

\emph{Best-of-$N$ theory.} Beirami et al.~\citep{beirami2024theoretical} and Yang et al.~\citep{yang2024asymptotics} analyze BoN's KL-reward tradeoff. Verdun et al.~\citep{verdun2025soft} introduce Soft BoN (SBoN), showing $O(1/n)$ convergence to the KL-regularized optimum. Aminian et al.~\citep{aminian2025smoothing} derive KL and regret bounds for SBoN under proxy reward error. Our MBON mechanisms are algorithmically equivalent to SBoN; the novelty lies in the privacy derivation of $\sigma$, the regret optimality proof, and the PrivITP extension. Khalaf et al.~\citep{khalaf2025inference} characterize inference-time reward hacking; our reward hacking bounds (\Cref{sec:hacking}) complement their results by quantifying how noise injection caps the hacking gap. Yu et al.~\citep{yu2025limits} show that sequential test-time compute dominates parallel methods such as BoN; extending PrivITP to the sequential setting is an open direction. Huang et al.~\citep{huang2025bon} establish the regret skyline $\Omega(\sqrt{C^{\pi^*}\varepsilon_{\mathrm{RM}}^2})$ and introduce InferenceTimePessimism, a $\chi^2$-regularized rejection sampling algorithm that achieves this bound.

\emph{Differential privacy in LLM alignment.} All prior work on private alignment operates at \emph{training time}. Zhang et al.~\citep{zhang2025kl} prove that KL-regularization itself provides DP for the trained policy. Zhou et al.~\citep{zhou2025private} study private and robust offline RLHF under label DP. Wang et al.~\citep{wang2025private} propose DP-AdamW for private DPO fine-tuning. \citet{zhou2025square} gives private and robust algorithms for training-based alignment. Prior work on private inference has studied differentially private prediction~\citep{dwork2018privacy}, private top-$k$ selection~\citep{durfee2019practical}, but the specific interaction between privacy noise and alignment quality in inference-time methods has not been analyzed. Our work is the first to provide formal privacy guarantees at \emph{inference time}, protecting the reward model's training data during deployment rather than during training.

\section{Proofs for Private Best-of-N}
\label{app:proofs_bon}

We provide detailed proofs for every result stated in the main text.

\subsection{Proof of \Cref{thm:gumbel_priv}: PrivBoN Privacy}

\begin{proof}
Condition on the candidate set $y_{1:n}$, drawn i.i.d.\ from $\pi_0(\cdot|x)$ independently of $D$. The Gumbel-max trick states that if $g_1, \ldots, g_n \overset{\mathrm{iid}}{\sim} \mathrm{Gumbel}(0, \sigma)$, then:
\[
\Pr\!\left(\argmax_i (\hat r_D(x,y_i) + g_i) = k\right) = \frac{e^{\hat r_D(x,y_k)/\sigma}}{\sum_{j=1}^n e^{\hat r_D(x,y_j)/\sigma}}.
\]
This is precisely the exponential mechanism with score function $q(k; D) = \hat r_D(x, y_k)$ and privacy parameter $\varepsilon/(2\Delta_r) = 1/\sigma$.

For adjacent $D \sim D'$, any candidate $y_k$, and the fixed set $y_{1:n}$:
\begin{align}
\frac{\Pr(i^* = k \mid D)}{\Pr(i^* = k \mid D')} &= \frac{e^{\hat r_D(x,y_k)/\sigma}}{e^{\hat r_{D'}(x,y_k)/\sigma}} \cdot \frac{\sum_j e^{\hat r_{D'}(x,y_j)/\sigma}}{\sum_j e^{\hat r_D(x,y_j)/\sigma}}.
\end{align}
The first factor is at most $e^{\Delta_r/\sigma} = e^{\varepsilon/2}$ (since $|\hat r_D - \hat r_{D'}| \leq \Delta_r$). For the second factor, each term in the numerator sum is at most $e^{\Delta_r/\sigma}$ times the corresponding denominator term, so the ratio is at most $e^{\Delta_r/\sigma} = e^{\varepsilon/2}$. Combined: the likelihood ratio is at most $e^\varepsilon$.

Since the output $y_{i^*}$ is a deterministic function of $i^*$ and $y_{1:n}$ (post-processing), and $y_{1:n}$ is independent of $D$, the overall mechanism is $\varepsilon$-DP.
\end{proof}

\subsection{Proof of Proposition \ref{prop:gumbel_policy}: PrivBoN Closed-Form Policy}

\begin{proof}
For brevity, we denote $r(x,y):=\hat r_D(x,y)$.
Note that the event $\{y^* = y\}$ decomposes as $\bigcup_{i=1}^n \{i^* = i, y_i = y\}$. These are disjoint (only one index wins), so:
\[
\pi_\mathrm{PrivBoN}(y|x) = \E_{y_{1:n}}\left[\sum_{i=1}^n \Pr(i^* = i \mid y_{1:n})\,\ind\{y_i = y\}\right].
\]
By the i.i.d.\ structure of $y_1, \ldots, y_n$ and symmetry, all $n$ terms are identically distributed. Hence:
\[
= n\,\E_{y_{1:n}}\left[\Pr(i^* = 1 \mid y_{1:n})\,\ind\{y_1 = y\}\right].
\]
Apply the tower property, conditioning on $y_1$:
\[
= n\,\E_{y_1}\left[\ind\{y_1 = y\}\,\E_{y_{2:n}}\!\left[\frac{e^{r(x,y_1)/\sigma}}{\sum_{j=1}^{n} e^{r(x,y_j)/\sigma}}\right]\right]~,
\]
where inside the inner expectation, we have used the probability of selecting candidate 1 using the Gumbel-max trick. 

Now, since $y_1 \sim \pi_0(\cdot|x)$, we have $\E_{y_1}[\ind\{y_1=y\}f(y_1)] = \pi_0(y|x)\,f(y)$. Applying this identity yields the stated formula.

\paragraph{BoN Recovery ($\sigma \to 0$)}
Note that $\pi_\mathrm{PrivBoN}(y|x) = n\,\pi_0(y|x)\, \E_{y_{2:n}}\left[\frac{e^{r_1/\sigma}}{e^{r_1/\sigma} + \sum_{j=2}^n e^{r_j/\sigma}}\right]$, where $r_1=r(x,y)$ and $r_j = r(x,y_j)$ for $j \ge 2$. Inside the expectation, consider the softmax share for candidate 1:
\[
\psi(r_1)=\frac{e^{r_1/\sigma}}{e^{r_1/\sigma} + \sum_{j=2}^n e^{r_j/\sigma}} = \frac{1}{1 + \sum_{j=2}^n e^{(r_j - r_1)/\sigma}}.
\]
Let $M = \max_{j \geq 2} r_j$. Factor out $e^{(M-r_1)/\sigma}$ to get
\[
\psi(r_1)= \frac{1}{1 + e^{(M-r_1)/\sigma}\sum_{j=2}^n e^{(r_j - M)/\sigma}}.
\]
If $r_1 > M$ (i.e., $r_1$ is the unique maximum), then $(M-r_1)/\sigma \to -\infty$ when $\sigma \to 0$, so the softmax share $\psi(r_1)=\to 1$.
If $r_1 < M$, then $(M-r_1)/\sigma \to +\infty$, so the softmax share $\psi(r_1)=\to 0$.
If $r_1 = M$ (tie), then the limit depends on the number of ties, but ties have probability zero under continuous reward distributions.

Therefore, the softmax share satisfies
\[
\lim_{\sigma \to 0} \psi(r_1)=\lim_{\sigma \to 0} \frac{e^{r_1/\sigma}}{\sum_{j=1}^n e^{r_j/\sigma}} = \ind\{r_1 \geq \max_{j \geq 2} r_j\} 
\]
Taking expectation, we get
\[
\E_{y_{2:n}}[\ind\{r(x,y) \geq \max_{j\geq 2} r(x,y_j)\}] = \prod_{j=2}^n \Pr(r(x,y) \geq r(x,y_j)) = F_{\pi_0}(y|x)^{n-1},
\]
where $F_{\pi_0}(\cdot|x)$ is the cumulative distribution function of $r(x,Y)$ where $Y \sim \pi_0(\cdot|x)$.
Hence $\lim_{\sigma\to 0}\pi_\mathrm{PrivBoN}(y|x) = n\,\pi_0(y|x)\,F_{\pi_0}(y|x)^{n-1} = \pi_\BON(y|x)$.

\paragraph{PrivBoN Policy Limit ($n \to \infty$)}
Again, consider the softmax share for candidate 1 inside the expectation:
\[
\psi(r_1)=\frac{e^{r_1/\sigma}}{e^{r_1/\sigma} + \sum_{j=2}^n e^{r_j/\sigma}} = \frac{1}{1 + \sum_{j=2}^n e^{(r_j - r_1)/\sigma}}.
\]
Define $S_n = \frac{1}{n-1}\sum_{j=2}^n e^{r_j/\sigma}$. By the strong law of large numbers, $S_n \xrightarrow{\mathrm{a.s.}} M(\sigma) = \E_{Y\sim\pi_0(\cdot|x)}[e^{r(x,Y)/\sigma}]$. Hence
\[
\psi(r_1) \to \frac{1}{1 + (n-1)M(\sigma) e^{-r_1/\sigma}}  \implies n\,\psi(r_1) \to n \cdot e^{r_1/\sigma}/((n-1)M(\sigma) + e^{r_1/\sigma}).
\]
Hence $n\,\psi(r_1) \to e^{r_1/\sigma}/M(\sigma)$ as $n \to \infty$. Therefore $\pi_\mathrm{PrivBoN}(y|x) \to \pi_0(y|x) \cdot e^{r(x,y)/\sigma}/M(\sigma)$. 
\end{proof}


\subsection{Proof of \Cref{thm:gumbel_regret}: PrivBoN Regret}

\begin{proof}
Define $r(x,y):=\hat r_D(x,y)$.
We decompose the regret via the limiting tilted policy $\pi_\infty(y|x) = \pi_0(y|x)e^{r(x,y)/\sigma}/M(\sigma)$, where $M(\sigma) = \E_{\pi_0}[e^{r(x,y)/\sigma}]$:
\begin{equation}
\label{eq:regret_decomp}
\mathrm{Reg}(\pi_{\mathrm{PrivBoN}};x) = \underbrace{J(\pi^*;x) - J(\pi_\infty;x)}_{\mathrm{(A)}} + \underbrace{J(\pi_\infty;x) - J(\pi_{\mathrm{PrivBoN}};x)}_{\mathrm{(B)}}.
\end{equation}

\medskip
\noindent\textbf{Step 1: Bounding (A). Regret of PrivBoN under infinite samples.}

Define the shorthand $\E_\pi[f]:= \E_{y \sim \pi(\cdot|x)}[f(x,y)]$ for any function $f:\mathcal{X} \times \mathcal{Y} \to \mathbb{R}$ and $\KL(\pi\|\pi'):=\KL(\pi(\cdot|x)\|\pi'(\cdot|x))$.
Now, write the true reward in terms of the proxy: $r^*(x,y) = r(x,y) - \epsilon(x,y)$. Then
\begin{align}
\mathrm{(A)} &= \E_{\pi^*}[r^*] - \E_{\pi_\infty}[r^*] \nonumber\\
&= \big(\E_{\pi^*}[r] - \E_{\pi_\infty}[r]\big) - \big(\E_{\pi^*}[\epsilon] - \E_{\pi_\infty}[\epsilon]\big). \label{eq:A_expand}
\end{align}
\textit{Bounding the reward gap.} The tilted policy $\pi_\infty$ solves $\max_\pi\{\E_\pi[r] - \sigma\KL(\pi\|\pi_{\mathrm{ref}})\}$. By the optimality condition, for any policy $\pi$,
\begin{equation*}
\label{eq:kl_reg_gap}
\E_\pi[r] - \E_{\pi_\infty}[r] \leq \sigma\big(\KL(\pi\|\pi_0) - \KL(\pi_\infty\|\pi_0)\big) \leq \sigma\,\KL(\pi\|\pi_0).
\end{equation*}
Applying this with $\pi = \pi^*$, we get
\begin{equation*}
\label{eq:proxy_gap}
\E_{\pi^*}[r] - \E_{\pi_\infty}[r] \leq \sigma\,\KL(\pi^*\|\pi_{\mathrm{ref}}).
\end{equation*}
Since $\KL(\pi^*\|\pi_{\mathrm{ref}}) \leq \chi^2(\pi^*\|\pi_{\mathrm{ref}}) = \frac{1}{2}(C^{\pi^*}(x) - 1)$ (by the standard inequality $\KL \leq \chi^2$), we obtain
\begin{equation}
\label{eq:proxy_gap_final}
\E_{\pi^*}[r] - \E_{\pi_\infty}[r] \leq \frac{\sigma}{2}(C^{\pi^*}(x) - 1).
\end{equation}

\textit{Bounding the error terms.} For any policy $\pi$ with coverage $C^\pi(x)$, the Cauchy--Schwarz inequality gives:
\begin{align*}
|\E_\pi[\epsilon]| = \left|\E_{\pi_0}\!\left[\frac{\pi(y|x)}{\pi_0(y|x)}\epsilon(x,y)\right]\right| &\leq \sqrt{\E_{\pi_0}\!\left[\!\left(\frac{\pi(x,y)}{\pi_0(x,y)}\right)^{\!2}\right]\E_{\pi_0}[\epsilon^2(x,y)]}
\\ &= \sqrt{\E_{\pi}\!\left[\!\frac{\pi(x,y)}{\pi_0(x,y)}\right]\E_{\pi_0}[\epsilon^2(x,y)]}
= \sqrt{C^\pi(x)}\,\varepsilon_{\mathrm{RM}}(x).
\end{align*}
Applying this to $\pi^*$ directly, we get
\begin{equation}
\label{eq:pi_star_error}
|\E_{\pi^*}[\epsilon]| \leq \sqrt{C^{\pi^*}(x)}\,\varepsilon_{\mathrm{RM}}(x).
\end{equation}

Now $\epsilon(x,y) \sim \mathcal{N}(0,\sigma_r^2(x))$ under $\pi_0$. So, the tilted distribution of $\epsilon(x,y)$ under $\pi_\infty$ has density $\mathcal{N}(\sigma_r^2(x)/\sigma, \sigma_r^2(x))$. Therefore we get
\begin{equation}
\label{eq:tilted_error_final}
\E_{y \sim \pi_\infty(\cdot|x)}[\epsilon(x,y)] = \frac{\sigma_r^2(x)}{\sigma}.
\end{equation}

Combining~\eqref{eq:A_expand},~\eqref{eq:proxy_gap_final},~\eqref{eq:pi_star_error}, and~\eqref{eq:tilted_error_final}:
\begin{equation}
\label{eq:A_final}
\mathrm{(A)} \leq \frac{\sigma}{2}(C^{\pi^*}-1) + \sqrt{C^{\pi^*}}\,\varepsilon_{\mathrm{RM}} + \frac{\varepsilon_{\mathrm{RM}}^2}{\sigma}~.
\end{equation}

\paragraph{Bounding (B): Finite-sample gap.}

Let $w_k = e^{r(x, y_k)/\sigma}$ and $f_k = r^*(x, y_k)$ for candidates $y_1, \ldots, y_n \sim \pi_0$. Assume that $r(x,y) \in [a, b]$ with range $c = b - a$, we have $w_k \in [e^{a/\sigma}, e^{b/\sigma}]$. Since $r^* \in [0, R_{\max}]$, we have $f_k \in [0, R_{\max}]$. By the Gumbel-max trick, the selected candidate $y_{i^*}$ is drawn from the softmax distribution $\Pr[i^* = k \mid y_{1:n}] = w_k/\sum_j w_j$, so
\[
J(\hat{\pi}_\mathrm{PrivBoN};x) = \E_{\pi_\mathrm{PrivBoN}}[r^*]=\E_{y_{1:n}}\!\left[\frac{\sum_k w_k f_k}{\sum_k w_k}\right] = \E\!\left[\frac{\hat{Z}_n^{(r)}}{\hat{Z}_n}\right],
\]
where $\hat{Z}_n = \frac{1}{n}\sum_k w_k$ and $\hat{Z}_n^{(r)} = \frac{1}{n}\sum_k w_k f_k$ are empirical means, with population counterparts $Z = \E_{\pi_0}[w] \geq e^{a/\sigma}$ and $Z^{(r)} = \E_{\pi_0}[wf]$, where $w = e^{r(x, y)/\sigma}$ and $f= r^*(x, y)$, so that $J(\pi_\infty;x) =Z^{(r)}/Z$.

Let $A = \hat{Z}_n^{(r)} - Z^{(r)}$ and $B = \hat{Z}_n - Z$. Then:
\[
\frac{\hat{Z}_n^{(r)}}{\hat{Z}_n} - \frac{Z^{(r)}}{Z} = \frac{\hat{Z}_n^{(r)} Z - Z^{(r)} \hat{Z}_n}{Z \hat{Z}_n} = \frac{AZ - Z^{(r)}B}{Z\hat{Z}_n}.
\]
By the triangle inequality and $Z^{(r)} = \E_{\pi_0}[wr^*] \leq R_{\max}\E_{\pi_0}[w] = R_{\max}Z$ (since $r^* \in [0, R_{\max}]$ and $w > 0$), we have
\[
|AZ - Z^{(r)}B| \leq Z|A| + Z^{(r)}|B| \leq Z(|A| + R_{\max}|B|),
\]
giving $|\hat{Z}_n^{(r)}/\hat{Z}_n - Z^{(r)}/Z| \leq (|A| + R_{\max}|B|)/\hat{Z}_n$.

Define $\mathcal{E} = \{\hat{Z}_n \geq Z/2\}$. Since $w_k \in [e^{a/\sigma}, e^{b/\sigma}]$ are i.i.d.\ with mean $Z$, Hoeffding's inequality gives
\[
\Pr[\mathcal{E}^c] = \Pr[\hat{Z}_n - Z < -Z/2] \leq \exp\!\left(-\frac{2n(Z/2)^2}{(e^{b/\sigma} - e^{a/\sigma})^2}\right) \leq \exp\!\left(-\frac{nZ^2}{2e^{2b/\sigma}}\right),
\]
where the last inequality uses $e^{b/\sigma} - e^{a/\sigma} \leq e^{b/\sigma}$.

\emph{Bound on the good event.} On $\mathcal{E}$, $\hat{Z}_n \geq Z/2$, so $(|A| + R_{\max}|B|)/\hat{Z}_n \leq 2(|A| + R_{\max}|B|)/Z$. Now, via Jensen inequality ($\E|X| \leq \sqrt{\E[X^2]}$), we get
\[
\E[|B|] \leq \sqrt{\E[(\hat Z_n -Z)^2]}= \sqrt{\Var(\hat{Z}_n)} = \frac{\sqrt{\Var_{\pi_0}(w)}}{\sqrt{n}} \leq \frac{\sqrt{\E_{\pi_0}[w^2]}}{\sqrt{n}} \leq \frac{e^{b/\sigma}}{\sqrt{n}},
\]
using $\Var(w) \leq \E[w^2]$ and $w \leq e^{b/\sigma}$. Similarly
\[
\E[|A|] \leq \sqrt{\Var(\hat{Z}_n^{(r)})} = \frac{\sqrt{\Var_{\pi_0}(wr^*)}}{\sqrt{n}} \leq \frac{\sqrt{\E_{\pi_0}[w^2(r^*)^2]}}{\sqrt{n}} \leq \frac{R_{\max}e^{b/\sigma}}{\sqrt{n}},
\]
using $(r^*)^2 \leq R_{\max}^2$. Therefore, we get
\[
\E\!\left[\left|\frac{\hat{Z}_n^{(r)}}{\hat{Z}_n} - \frac{Z^{(r)}}{Z}\right| \cdot \mathbf{1}_{\mathcal{E}}\right] \leq \frac{2}{Z}\!\left(\frac{R_{\max}e^{b/\sigma}}{\sqrt{n}} + \frac{R_{\max}e^{b/\sigma}}{\sqrt{n}}\right) = \frac{4R_{\max}e^{b/\sigma}}{Z\sqrt{n}}.
\]

\emph{Bound on the bad event.} Both $\hat{Z}_n^{(r)}/\hat{Z}_n$ and $Z^{(r)}/Z$ are weighted averages of values in $[0, R_{\max}]$ (the former is a convex combination of $f_k \in [0, R_{\max}]$; the latter is $\E[wf]/\E[w]$ with $f \in [0, R_{\max}]$). Hence, their difference is bounded by $R_{\max}$. Thus
\[
\E\!\left[\left|\frac{\hat{Z}_n^{(r)}}{\hat{Z}_n} - \frac{Z^{(r)}}{Z}\right| \cdot \mathbf{1}_{\mathcal{E}^c}\right] \leq R_{\max} \cdot \Pr[\mathcal{E}^c] \leq R_{\max}\exp\!\left(-\frac{nZ^2}{2e^{2b/\sigma}}\right).
\]

By the triangle inequality for expectations, we obtain
\[
|J(\pi_\infty;x) - J(\pi_\mathrm{PrivBoN};x)| \leq \E\!\left[\left|\frac{\hat{Z}_n^{(r)}}{\hat{Z}_n} - \frac{Z^{(r)}}{Z}\right|\right] \leq \frac{4R_{\max}e^{b/\sigma}}{Z\sqrt{n}} + R_{\max}\exp\!\left(-\frac{nZ^2}{2e^{2b/\sigma}}\right),
\]
where $Z = \E_{\pi_0}[e^{r/\sigma}] \geq e^{a/\sigma}$. Using $Z \geq e^{a/\sigma}$, this simplifies to:
\begin{align}\label{eq:B_final}
  |J(\pi_\infty) - J(\hat{\pi}_{\MBON})| \leq \frac{4R_{\max}e^{c/\sigma}}{\sqrt{n}} + R_{\max}\exp\!\left(-\frac{n}{2e^{2c/\sigma}}\right) \leq \frac{4R_{\max}e^{c/\sigma}}{\sqrt{n}} + \varepsilon_{\mathrm{RM}} 
\end{align}
for $n \geq  2e^{2c/\sigma}\log(R_{\max}/\varepsilon_{\mathrm{RM}})$, where $c = b - a$ is the reward range.


Adding~\eqref{eq:A_final} and~\eqref{eq:B_final}, we get
\begin{equation*}
\mathrm{Reg}^{\pi^*}(x) \leq \sigma(C^{\pi^*}-1) + \frac{\varepsilon_{\mathrm{RM}}^2}{\sigma} + \sqrt{C^{\pi^*}}\,\varepsilon_{\mathrm{RM}} + O\!\left(\frac{R_{\max}e^{c/\sigma}}{\sqrt{n}}\right).
\end{equation*}

The regret bound is modulo the assumption that $r(x,y) \in [a, b]$ with range $c = b - a$. Under the Gaussian error model $r(x,y) = r^*(x,y) + \epsilon(x,y)$ with $\epsilon(x,y) \sim \mathcal{N}(0, \sigma_r^2(x))$, the reward is unbounded. We handle this by conditioning on the high-probability event $\mathcal{F} := \{\max_k |\epsilon(x,y_k)| \leq \sigma_r\sqrt{2\log(n/\Delta_r)}\}$, which has $\Pr[\mathcal{F}] \geq 1 - \Delta_r$ by Gaussian tail and union bound. On $\mathcal{F}$, the rewards lie in $[-\sigma_r\sqrt{2\log(n/\Delta_r)},\, R_{\max} + \sigma_r\sqrt{2\log(n/\Delta_r)}]$, so the bounded analysis applies with effective range $c = R_{\max} + 2\sigma_r\sqrt{2\log(n/\Delta_r)}$. On $\mathcal{F}^c$, the regret is bounded trivially by $R_{\max}$, contributing $R_{\max}\Delta_r$. Setting $\Delta_r = 1/\sqrt{n}$, the effective range becomes $c = O(R_{\max} + \sigma_r\sqrt{\log n})$. Substituting this, we get
\begin{equation*}
\mathrm{Reg}^{\pi^*}(x) \leq \sigma(C^{\pi^*}-1) + \frac{\varepsilon_{\mathrm{RM}}^2}{\sigma} + \sqrt{C^{\pi^*}}\,\varepsilon_{\mathrm{RM}} + O\!\left(\frac{R_{\max}e^{(R_{\max} + \varepsilon_{\mathrm{RM}}\sqrt{\log n})/\sigma}}{\sqrt{n}}\right)
\end{equation*}
for $n = \Omega \left( e^{(R_{\max} + \varepsilon_{\mathrm{RM}}\sqrt{\log n})/\sigma}\log(R_{\max}/\varepsilon_{\mathrm{RM}}\right)$.
\end{proof}

\subsection{On the Gaussian-error assumption in Theorem~\ref{thm:gumbel_regret}}
\label{app:gaussian-assumption}

The Gaussian-error assumption in Theorem~\ref{thm:gumbel_regret} gives the closed form $\E_{\pi_\infty}[\epsilon(x,y)] = \sigma_r^2(x)/\sigma$, which we use to derive the exact $\epsilon_{\mathrm{RM}}^2/\sigma$ overoptimization term. Under a general mean-squared-error assumption $\E_{\pi_0}[\epsilon(x,y)^2] = \epsilon_{\mathrm{RM}}^2(x)$, an application Cauchy--Schwarz inequality yields only $|\E_{\pi_\infty}[\epsilon]| \leq \sqrt{C^{\pi_\infty}(x)}\, \epsilon_{\mathrm{RM}}(x)$, where $C^{\pi_\infty}$ is the coverage of the tilted policy. Since $C^{\pi_\infty}$ scales as $\exp(R_{\max}/\sigma)$ in the worst case, the overoptimization term inflates, weakening the ``privacy is free'' threshold.

The finite-sample analysis conditions on a high-probability boundedness event via a Gaussian tail bound (the union bound $\Pr[\max_k |\epsilon(x, y_k)| \leq \sigma_r\sqrt{2\log(n/\delta)}] \geq 1 - \delta$). For sub-Gaussian errors with parameter $\sigma_r$ the analysis goes through unchanged up to constants. For bounded rewards $|\hat r_D| \leq B$, the tail step becomes trivial, but $B$ replaces $\sigma_r\sqrt{\log n}$ in the effective reward range, tightening the finite-sample term when $B \ll \sigma_r\sqrt{\log n}$ and loosening it otherwise.

The qualitative conclusions---bias-variance trade-off in $\sigma$, existence of an $\epsilon^*$ threshold above which privacy is free, and regret going down with $n$---hold under the general mean-squared-error assumption. The clean form of the threshold $\epsilon^*(x) = 2\Delta_r\sqrt{C^{\pi^*}(x)}/\epsilon_{\mathrm{RM}}(x)$, however, is specific to the Gaussian case.

\subsection{Proof of Proposition \ref{prop:hacking} (Win-Rate Bound and Reward hacking)}

\begin{proof}

\textbf{(a) Win-rate.} For brevity, we denote $r(x,y):=\hat r_D(x,y)$. Define $r_i = r(x,y_i)$ for $i \in [n]$ and $r_0 = r(x,y_0)$. Let $p_i = \Pr(r_i \geq r_0 \mid y_i)$ where $y_0,y_i \sim \pi_0(\cdot|x)$ independently.

We know that as $n \to \infty$, $\pi_\mathrm{PrivBoN} \to \pi_\infty$ where $\pi_\infty(y|x)\propto \pi_0(y|x)e^{r(x,y)/\sigma}$.

For Gaussian rewards $r(x,y) \sim \mathcal{N}(r^*(x,y),\sigma_r^2(x))$ under $\pi_0$. Now, the tilted distribution of $r(x,y)$ under $\pi_\infty$ has density $\mathcal{N}(r^*(x,y)+\sigma_r^2(x)/\sigma, \sigma_r^2(x))$. Hence $Y \sim \pi_\infty(\cdot|x)$ has reward $r_Y \sim \mathcal{N}(r^*(x,y)+\sigma_r^2(x)/\sigma, \sigma_r^2(x))$ and $Z \sim \pi_0(\cdot|x)$ has reward $r_Z \sim \mathcal{N}(r^*(x,y),\sigma_r^2(x))$. Then:
\[
r_Y - r_Z \sim \mathcal{N}(\sigma_r^2(x)/\sigma, 2\sigma_r^2(x)).
\]
So $\mathrm{WinRate}^{\mathrm{PrivBoN}}_\infty = \Pr(r_Y \geq r_Z) = \Phi(\sigma_r^2(x)/(\sigma\sqrt{2}\sigma_r(x))) = \Phi(\sigma_r(x)/(\sigma\sqrt{2})) < 1$.


\textbf{(b) Reward hacking.} Denote $r(x,y) :=\hat r_D(x,y)$,  $r_i:=r(x,y_i)$ and $\epsilon_i := \epsilon(x,y_i)$. For any fixed candidate set $y_{1:n}$, BoN selects $i^*=\argmax_i r_i$. This maximizes $\E[r_{i^*}]$ over all (possibly randomized) selection rules, since the deterministic $\argmax$ is optimal. PrivBoN's softmax selection is a specific randomized rule, hence suboptimal: $\E_{y \sim \pi_\mathrm{PrivBoN}(\cdot|x)}[r(x,y)] \leq \E_{\pi_\BON(\cdot|x)}[r(x,y)]$.

For BoN, the hacking gap is $\E_{y\sim\pi_{\mathrm{BoN}}(\cdot|x)}[\epsilon(x,y)]=\E_{y_{1:n}}[\sum_{i=1}^n \mathbf{1}\{i^* =i\} \epsilon_i] \le \E_{y_{1:n}}[\max_{i=1}^n  \epsilon_i] $. Since each $\epsilon_i \sim \mathcal{N}(0,\sigma_r^2(x))$ under $\pi_0$ and independent, we have
$\E[\max_i \epsilon_i] = O(\sigma_r(x)\sqrt{\log n})$ by standard results on Gaussian maximum.



In the limit $n\to \infty$, we have $\pi_{\mathrm{PrivBoN}} \to \pi_\infty$, where $\pi_\infty(y|x)\propto \pi_0(y|x)e^{r(x,y)/\sigma}$. Now $\epsilon(x,y) \sim \mathcal{N}(0,\sigma_r^2(x))$ under $\pi_0$. So, the tilted distribution of $\epsilon(x,y)$ under $\pi_\infty$ has density $\mathcal{N}(\sigma_r^2(x)/\sigma, \sigma_r^2(x))$. Therefore, the hacking gap of PrivBoN becomes $\E_{y \sim \pi_\infty(\cdot|x)}[\epsilon(x,y)] = \frac{\sigma_r^2(x)}{\sigma}$ in the limit $n \to \infty$.
\end{proof}

\section{Proofs for Private Inference-Time Pessimism}
\label{app:proofs_itp}


\subsection{Proof of \Cref{thm:privitp_priv} (PrivITP Privacy)}

\begin{proof}
We analyze the two phases separately and combine them via basic composition.

\textbf{Phase~1.} Define $\lambda^*(D)$ as the unique solution to
\[
F_D(\lambda) := \frac{1}{n}\sum_{i=1}^n \mathrm{relu}\!\big(\beta^{-1}(\hat{r}_D(x, y_i) - \lambda)\big) = 1,
\]
where the candidates $y_1, \ldots, y_n \sim \pi_0(\cdot|x)$ are drawn independently of $D$. The function $F_D$ is continuous and strictly decreasing in $\lambda$, with $F_D(\lambda) \to \infty$ as $\lambda \to -\infty$ and $F_D(\lambda) = 0$ for $\lambda \geq \max_i \hat{r}_D(x, y_i)$, so the solution exists and is unique.

\emph{Sensitivity of $\lambda^*$.} For adjacent datasets $D \sim D'$ and each candidate $y_i$, we have $|\hat{r}_D(x, y_i) - \hat{r}_{D'}(x, y_i)| \leq \Delta_r$. Since $\mathrm{relu}(\beta^{-1}(z - \lambda))$ is monotone increasing in $z$, this gives pointwise
\[
\mathrm{relu}\!\big(\beta^{-1}(\hat{r}_D(x, y_i) - (\lambda + \Delta_r))\big) \leq \mathrm{relu}\!\big(\beta^{-1}(\hat{r}_{D'}(x, y_i) - \lambda)\big) \leq \mathrm{relu}\!\big(\beta^{-1}(\hat{r}_D(x, y_i) - (\lambda - \Delta_r))\big).
\]
Averaging over $i$, we get
\[
F_D(\lambda + \Delta_r) \leq F_{D'}(\lambda) \leq F_D(\lambda - \Delta_r).
\]

Setting $\lambda = \lambda^*(D')$, so $F_{D'}(\lambda^*(D')) = 1 = F_D(\lambda^*(D))$, we obtain
\begin{itemize}[leftmargin=*,itemsep=0pt]
\item Upper bound: $1 \leq F_D(\lambda^*(D') - \Delta_r)$, so by monotonicity of $F_D$, $\lambda^*(D') - \Delta_r \leq \lambda^*(D)$, i.e., $\lambda^*(D') \leq \lambda^*(D) + \Delta_r$.
\item Lower bound: $1 \geq F_D(\lambda^*(D') + \Delta_r)$, so $\lambda^*(D') + \Delta_r \geq \lambda^*(D)$, i.e., $\lambda^*(D') \geq \lambda^*(D) - \Delta_r$.
\end{itemize}
Combining both, we get $|\lambda^*(D) - \lambda^*(D')| \leq \Delta_r$.

\emph{Gaussian mechanism.} Since $\lambda^* : \mathcal{D} \to \mathbb{R}$ is a scalar deterministic function with global sensitivity at most $\Delta_r$, releasing $\tilde{\lambda} = \lambda^*(D) + \zeta$ with $\zeta \sim \mathcal{N}(0, \sigma_X^2)$ is the scalar Gaussian mechanism~\citep{dwork2014algorithmic}, yielding $(\varepsilon_1, \delta)$-DP with $\varepsilon_1 = \Delta_r\sqrt{2\ln(1.25/\delta)}/\sigma_X$. 

\textbf{Phase~2.} Fix $\lambda$ (released privately from Phase~1, treated as public for the ex-post analysis) and fresh candidates $y'_1, \ldots, y'_n \sim \pi_0(\cdot|x)$ (independent of $D$). Let $q_i := \hat{r}_D(x, y'_i) \in [0, R_{\max}]$. At each round, Phase~2 computes $\tilde r_i = q_i + g'_i$ with $g_i \sim \mathcal{N}(0, \sigma_Z^2)$, weight $w_i = \mathrm{relu}(\beta^{-1}(\tilde{r}_i - \tilde{\lambda}))$, and acceptance $\xi_i \sim \mathrm{Bernoulli}(\min(w_i/M, 1))$. The halting time is $t^* = \min\{i : \xi_i = 1\}$.

\emph{Reformulation as randomized threshold.} We show rejection sampling is equivalent in distribution to a threshold check with a randomized effective threshold. Let $u_i \sim \mathrm{Unif}[0,1]$ be fresh independent uniform noise.

\begin{lemma}[Reformulation]\label{lem:reformulation}
$\Pr[\xi_i = 1 \mid \tilde{r}_i] = \Pr\!\big[\tilde{r}_i \geq \tilde \lambda + \beta M u_i \mid \tilde{r}_i\big]$.
\end{lemma}

\noindent Indeed: (a) if $\tilde{r}_i < \tilde{\lambda}$, both sides are zero; (b) if $\tilde{\lambda} \leq \tilde{r}_i \leq \tilde{\lambda} + \beta M$, the LHS equals $w_i/M = (\tilde{r}_i - \lambda)/(\beta M)$, and the RHS equals $\Pr[u_i \leq (\tilde{r}_i - \tilde{\lambda})/(\beta M)] = (\tilde{r}_i - \lambda)/(\beta M)$; (c) if $\tilde{r}_i > \tilde{\lambda} + \beta M$, both sides are one. 

With this reformulation, Phase~2 accepts candidate $i$ iff $g_i \geq \tilde{\lambda} + \beta M u_i - q_i$, where $(g_i, u_i)$ are independent fresh noise, independent across rounds.

\emph{Halting probability.} Since the noise is independent across rounds, the probability of halting at time $t$ factorizes:
\[
\Pr[t^* = t \mid D] = \prod_{i=1}^{t-1}(1 - p_i(D)) \cdot p_t(D),
\]
where the per-round acceptance probability (marginalizing over both $g_i$ and $u_i$) is
\[
p_i(D) = \Pr[g_i \geq \tilde \lambda + \beta M u_i - q_i] = \E_{u_i}\!\left[\Phi\!\left(\frac{q_i - \tilde{\lambda} - \beta M u_i}{\sigma_Z}\right)\right],
\]
using $\Pr[g_i \geq z] = \Phi(-z/\sigma_Z)$. Correspondingly,
\[
1 - p_i(D) = \E_{u_i}\!\left[\Phi\!\left(\frac{\tilde{\lambda} + \beta M u_i - q_i}{\sigma_Z}\right)\right].
\]

\emph{Ratio under coupling.} Under the coupling where all noise variables $(g_i, u_i)$ are shared between $D$ and $D'$, let $q'_i := \hat{r}_{D'}(x, y'_i)$ and $\delta_i := q_i - q'_i \in [-\Delta_r, \Delta_r]$. Define
\[
A_i(q, \delta) := \E_{u}\!\left[\Phi\!\left(\frac{\tilde{\lambda} + \beta M u - q - \delta}{\sigma_Z}\right)\right], \qquad B_i(q, \delta) := \E_{u}\!\left[\Phi\!\left(\frac{q + \delta - \tilde{\lambda} - \beta M u}{\sigma_Z}\right)\right].
\]
Then $1 - p_i(D) = A_i(q'_i, \delta_i)$ and $p_i(D) = B_i(q'_i, \delta_i)$, and similarly $1 - p_i(D') = A_i(q'_i, 0)$, $p_i(D') = B_i(q'_i, 0)$. The ratio of halting probabilities is
\[
\frac{\Pr[t^* = t \mid D]}{\Pr[t^* = t \mid D']} = \prod_{i=1}^{t-1}\frac{A_i(q'_i, \delta_i)}{A_i(q'_i, 0)} \cdot \frac{B_t(q'_t, \delta_t)}{B_t(q'_t, 0)}.
\]

\emph{Worst-case shifts.} The function $A_i(q'_i, \delta)$ is decreasing in $\delta$ (larger $\delta$ shrinks $\Phi$'s argument), so $A_i(q'_i, \delta_i)/A_i(q'_i, 0)$ is maximized at $\delta_i = -\Delta_r$ for $i < t$. The function $B_t(q'_t, \delta)$ is increasing in $\delta$, so $B_t(q'_t, \delta_t)/B_t(q'_t, 0)$ is maximized at $\delta_t = +\Delta_r$.

\emph{Worst-case query values.} After fixing the shifts, each ratio factor depends only on $q'_i$ through the arguments of $\Phi$. It holds that (see Lemma~\ref{lem:monotone})
\begin{enumerate}[label=(\roman*),leftmargin=*]
\item For $i < t$: $A_i(q'_i, -\Delta_r)/A_i(q'_i, 0)$ is non-decreasing in $q'_i \in [0, R_{\max}]$, maximized at $q'_i = R_{\max}$.
\item For $i = t$: $B_t(q'_t, \Delta_r)/B_t(q'_t, 0)$ is non-increasing in $q'_t$, maximized at $q'_t = 0$.
\end{enumerate}

\emph{Final bound.} Substituting $q'_i = R_{\max}$ for $i < t$ (with shift $\delta_i = -\Delta_r$) and $q'_t = 0$ (with shift $\delta_t = +\Delta_r$):
\begin{align*}
\frac{A_i(R_{\max}, -\Delta_r)}{A_i(R_{\max}, 0)} &= \frac{\E_u[\Phi((\tilde{\lambda} + \beta M u - R_{\max} + \Delta_r)/\sigma_Z)]}{\E_u[\Phi((\tilde{\lambda} + \beta M u - R_{\max})/\sigma_Z)]}, \\
\frac{B_t(0, \Delta_r)}{B_t(0, 0)} &= \frac{\E_u[\Phi((\Delta_r - \tilde{\lambda} - \beta M u)/\sigma_Z)]}{\E_u[\Phi((-\tilde{\lambda} - \beta M u)/\sigma_Z)]}.
\end{align*}
The factors for $i = 1, \ldots, t-1$ are identical (since the expression depends on $i$ only through $u_i$, and the $u_i$'s are i.i.d.), so their product is the first factor raised to the power $t-1$. Taking logarithms yields~\eqref{eq:privitp_expost}.

\textbf{Composition.} By basic composition of DP mechanisms, PrivITP satisfies $(\varepsilon_1 + \varepsilon_2^{\mathrm{post}}(t), \delta)$-DP ex-post.
\end{proof}

\begin{lemma}[Monotonicity of Gaussian CDF ratios under uniform averaging]\label{lem:monotone}
Let $\tau \sim \mathrm{Unif}[Q,R]$ with $Q<R$, and let $\Delta>0$. Define, for $x\in\mathbb{R}$,
\[
R_1(x) := \frac{\mathbb{E}\big[\Phi(\tau - x + \Delta)\big]}{\mathbb{E}\big[\Phi(\tau - x)\big]},
\qquad
R_2(x) := \frac{\mathbb{E}\big[\Phi(x - \tau + \Delta)\big]}{\mathbb{E}\big[\Phi(x - \tau)\big]}.
\]
Then $R_1(x)$ is non-decreasing in $x$ and $R_2(x)$ is non-increasing in $x$.  
\end{lemma}

\begin{proof}

The sign condition for both ratios reduces to log-concavity of a single
auxiliary function, which is then established via the
Pr\'{e}kopa--Leindler inequality.

Write $N(x) = \int_Q^R \Phi(\tau - x + \Delta)\,d\tau$ and
$D(x) = \int_Q^R \Phi(\tau - x)\,d\tau$, so that $R_1(x) = N(x)/D(x)$.
Since $\frac{d}{dx}\Phi(\tau - x + \Delta) = -\phi(\tau - x + \Delta)$,
\[
  R_1'(x)
  = \frac{N'(x)\,D(x) - N(x)\,D'(x)}{D(x)^2}
  = \frac{-\!\int_Q^R \phi(\tau\!-\!x\!+\!\Delta)\,d\tau \cdot D(x)
         \;+\; N(x)\cdot\!\int_Q^R \phi(\tau\!-\!x)\,d\tau}
        {D(x)^2}.
\]
Hence $R_1'(x) \ge 0$ if and only if
\begin{equation}\label{eq:sign}
  N(x)\int_Q^R \phi(\tau - x)\,d\tau
  \;\ge\;
  D(x)\int_Q^R \phi(\tau - x + \Delta)\,d\tau.
\end{equation}
Substitute $u = \tau - x$ and set $a = Q - x$, $b = R - x$, $L = R - Q$.
Define
\[
  F(t) \;:=\; \int_a^{b} \Phi(u + t)\,du,
  \qquad t \in \mathbb{R}.
\]
Then
\[
  N(x) = F(\Delta), \qquad
  D(x) = F(0), \qquad
  F'(t) = \int_a^{b} \phi(u+t)\,du = \Phi(b+t) - \Phi(a+t),
\]
so $F'(0) = \int_Q^R \phi(\tau - x)\,d\tau$ and
$F'(\Delta) = \int_Q^R \phi(\tau - x + \Delta)\,d\tau$.
Inequality~\eqref{eq:sign} therefore reads
\[
  F(\Delta)\,F'(0) \;\ge\; F(0)\,F'(\Delta),
  \qquad\text{i.e.,}\qquad
  \frac{F'(\Delta)}{F(\Delta)} \;\le\; \frac{F'(0)}{F(0)},
\]
which is equivalent to $(\log F)'$ being non-increasing, i.e.\
\emph{$F$ is log-concave in~$t$}.

Write $F$ as a convolution:
\[
  F(t)
  = \int_a^{b} \Phi(u + t)\,du
  = \int_{-\infty}^{\infty} \Phi(t - v)\,\mathbf{1}_{[-b,-a]}(v)\,dv
  = \bigl(\Phi \ast \mathbf{1}_{[-b,-a]}\bigr)(t).
\]
Both factors are log-concave: $\Phi$ is log-concave and $\mathbf{1}_{[-b,-a]}$ is the
indicator of a convex set, hence log-concave. By the
Pr\'{e}kopa--Leindler inequality, the convolution of log-concave functions
is log-concave. Therefore $F$ is log-concave in~$t$, which establishes
$R_1'(x) \ge 0$.

Similarly, set $\widetilde{N}(x) = \int_Q^R \Phi(x - \tau + \Delta)\,d\tau$ and
$\widetilde{D}(x) = \int_Q^R \Phi(x - \tau)\,d\tau$, so that
$R_2(x) = \widetilde{N}(x)/\widetilde{D}(x)$.
Substituting $v = x - \tau$ gives
\[
  \widetilde{N}(x) = \widetilde{F}(\Delta),
  \qquad
  \widetilde{D}(x) = \widetilde{F}(0),
  \qquad\text{where}\quad
  \widetilde{F}(t) := \int_{x-R}^{x-Q} \Phi(v + t)\,dv.
\]
Since $\frac{d}{dx}\Phi(x - \tau + \Delta) = \phi(x - \tau + \Delta)$,
the same computation yields $R_2'(x) \le 0$ if and only if
\[
  \widetilde{F}'(\Delta)\,\widetilde{F}(0)
  \;\ge\;
  \widetilde{F}(\Delta)\,\widetilde{F}'(0),
  \qquad\text{i.e.,}\qquad
  \frac{\widetilde{F}'(\Delta)}{\widetilde{F}(\Delta)}
  \;\ge\;
  \frac{\widetilde{F}'(0)}{\widetilde{F}(0)}.
\]
But $\widetilde{F}$ is again a convolution of $\Phi$ with an indicator, hence log-concave
in~$t$ by the same argument. Log-concavity gives
$(\log \widetilde{F})' = \widetilde{F}'/\widetilde{F}$ non-increasing,
which is precisely the required inequality. Therefore
$R_2'(x) \le 0$.
\end{proof}

\subsection{Proof of \Cref{thm:privitp_regret} (DP-ITP Regret)}
\label{sec:privitp_regret}

We analyze the regret of PrivITP against an arbitrary comparator policy $\pi^*$ using a unified smoothed-relu framework: both the Phase~1 threshold noise and the Phase~2 query noise induce policies whose weights differ from Huang et al.'s ideal $\chi^2$-policy by a bounded pointwise amount, and the regret gap is controlled uniformly by a single \emph{smoothed-weight lemma}.

\begin{proof}
Define $r(x,y):=\hat r_D(x,y)$ for brevity.
Let $\pi_\chi^\beta$ denote the $\chi^2$-regularized policy with clean rewards $r(x,y)$ and empirical threshold $\lambda^*$, i.e.,
\[
\pi_\chi^\beta(y|x) =\frac{1}{Z^*} \pi_0(y|x) \cdot w^*(x,y), \qquad w^*(y) := \mathrm{relu}\bigl(\beta^{-1}(r(x,y) - \lambda^*\bigr),
\]
where $\lambda^*$ is defined by $\frac{1}{n}\sum_{i=1}^n \mathrm{relu}\!\big(\beta^{-1}(\hat{r}_D(x, y_i) - \lambda)\big) = 1$ \citep[Lemma C.1]{huang2025bon}. Let $\pi_{\mathrm{PrivITP}}$ denote the actual policy induced by \Cref{alg:privitp}, marginalizing over all Phase~1 and Phase~2 noise and over the rejection sampling randomness.

\textbf{Regret Decomposition.} Introduce two intermediate policies:
\begin{align*}
\pi_b(y|x) &= \frac{1}{Z_b}  \pi_0(y|x)\cdot w_b(x,y), & w_b(x,y) &:= \mathrm{relu}\bigl(\beta^{-1}(r(x,y) - \tilde{\lambda})\bigr), \\
\pi_{\bar{w}}(y|x) &= \frac{1}{Z_{\bar w}}  \pi_0(y|x)\cdot \bar{w}(x,y), & \bar{w}(x,y) &:= \E_{g \sim \mathcal{N}(0,\sigma_Z^2)}\!\left[\mathrm{relu}\bigl(\beta^{-1}(r(x,y) + g - \tilde{\lambda})\bigr)\right],
\end{align*}
where $\tilde{\lambda} = \lambda^* + \zeta$ with $\zeta \sim \mathcal{N}(0,\sigma_X^2)$. The policy $\pi_b$ uses the noisy threshold $\tilde{\lambda}$ with clean rewards; $\pi_{\bar{w}}$ additionally marginalizes the Phase~2 Gaussian noise. 

Now, we decompose the (expected) regret of PrivITP as
\[
\mathrm{Reg}(x) = \underbrace{\bigl[J(\pi^*) - J(\pi_\chi^\beta)\bigr]}_{\text{(A): skyline}} + \underbrace{\bigl[J(\pi_\chi^\beta) - J(\pi_b)\bigr]}_{\text{(B1): Phase 1 noise}} + \underbrace{\bigl[J(\pi_b) - J(\pi_{\bar{w}})\bigr]}_{\text{(B2): Phase 2 noise}} + \underbrace{\bigl[J(\pi_{\bar{w}}) - J(\pi_{\mathrm{PrivITP}})\bigr]}_{\text{(B3): finite sample}}.
\]

\textbf{Term (A): Skyline.} Since $\pi_\chi^\beta$ is the $\chi^2$-regularized policy with clean $\hat{r}$ and $\lambda^*$, Theorem~4.1 of \citet{huang2025bon} applies and we get 
\begin{align}\label{eq:A}
 J(\pi^*) - J(\pi_\chi^\beta) \leq \sqrt{C^{\pi^*}(x)\varepsilon_{\mathrm{RM}}^2(x)} + \beta \cdot C^{\pi^*}(x) + \beta^{-1}\varepsilon_{\mathrm{RM}}^2(x)~   
\end{align}
for
$n = \Omega\!\left(\frac{R_{\max}}{\beta}\log\!\left(\frac{R_{\max}}{\beta \,\epsilon_{\mathrm{RM}}}\right)\right)$.

We now establish a lemma that will help us bound the remaining terms.

\begin{lemma}[Smoothed-weight policy comparison]
\label{lem:smooth_weight}
Let $w_1, w_2 : \mathcal{Y} \to \R_{\geq 0}$ be non-negative weight functions with normalizations $Z_i = \E_{\pi_\mathrm{ref}}[w_i] > 0$, and let $\pi_i(y) = \pi_{\mathrm{ref}}(y)w_i(y)/Z_i$. Suppose $\|w_1 - w_2\|_\infty \leq \eta$. Then:
\[
|J(\pi_1) - J(\pi_2)| \leq \frac{2R_{\max}\eta}{\max(Z_1, Z_2)}.
\]
\end{lemma}

\begin{proof}[Proof of Lemma \ref{lem:smooth_weight}]
Observe that
\begin{align*}
J(\pi_1) - J(\pi_2) &= \frac{\E_{\pi_{\mathrm{ref}}}[r^* w_1]}{Z_1} - \frac{\E_{\pi_{\mathrm{ref}}}[r^* w_2]}{Z_2} \\
&= \frac{Z_2\E[r^* w_1] - Z_1\E[r^* w_2]}{Z_1 Z_2} \\
&= \frac{Z_2\E[r^*(w_1 - w_2)] + (Z_2 - Z_1)\E[r^* w_2]}{Z_1 Z_2}.
\end{align*}
Using $r^* \leq R_{\max}$, $\|w_1 - w_2\|_\infty \leq \eta$, $|Z_1 - Z_2| \leq \eta$, and $\E[r^* w_2] \leq R_{\max}Z_2$:
\[
|J(\pi_1) - J(\pi_2)| \leq \frac{Z_2 R_{\max}\eta + \eta R_{\max}Z_2}{Z_1 Z_2} = \frac{2R_{\max}\eta}{Z_1}.
\]
By symmetry, the same bound holds with $Z_1$ replaced by $Z_2$, giving the stated result with $\max(Z_1, Z_2)$. \qedhere
\end{proof}

\textbf{Term (B1): Phase 1 noise effect.} Fix $\tau$. Since $\mathrm{relu}$ is 1-Lipschitz, we get
\[
\|w^* - w_b\|_\infty = \sup_y\left|\mathrm{relu}\!\left(\frac{r(x,y) - \lambda^*}{\beta}\right) - \mathrm{relu}\!\left(\frac{r(x,y) - \lambda^* - \zeta}{\beta}\right)\right| \leq \frac{|\zeta|}{\beta}.
\]
The normalization $Z^* = 1$ (by definition of $\lambda^*$). Thus $|Z_b - Z^*| \leq \|w_b - w^*\|_\infty \leq |\zeta|/\beta$. Hence, by Lemma \ref{lem:smooth_weight} with $\eta = |\zeta|/\beta$, and $\max\{Z^*,Z_b\} \ge Z^* =1$, we get
\[
|J(\pi_\chi^\beta) - J(\pi_b)| \cdot \mathbf{1}_{\mathcal{G}_1} \leq \frac{2R_{\max}|\zeta|/\beta}{1} = \frac{2R_{\max}|\zeta|}{\beta}.
\]
Taking expectation over $\zeta$, we get
\begin{equation}\label{eq:B1}
|J(\pi_\chi^\beta) - J(\pi_b)| \leq \frac{2R_{\max}\E|\zeta|}{\beta} \leq \frac{2R_{\max}\sigma_X\sqrt{2/\pi}}{\beta}=O(R_{\max}\sigma_X/\beta).
\end{equation}

\textbf{Term (B2): Phase 2 noise effect.} Define the Gaussian-smoothed relu $\varphi_\tau : \R \to \R_{\geq 0}$ by
\[
\varphi_\tau(z) := \E_{g \sim \mathcal{N}(0,\tau^2)}[\mathrm{relu}(z + g)] = z\Phi(z/\tau) + \tau\phi(z/\tau),
\]
where $\Phi, \phi$ are the standard Gaussian CDF and PDF. Setting $\tau = \sigma_Z/\beta$, we have $\bar{w}(x,y) = \varphi_\tau(\beta^{-1}(r(x,y) - \tilde{\lambda}))$.

Note that $\varphi_\tau(z) - \mathrm{relu}(z) \geq 0$ by Jensen inequality applied to the convex $\mathrm{relu}$ function. Moreover, observe that
\[
\sup_z\bigl(\varphi_\tau(z) - \mathrm{relu}(z)\bigr) = \varphi_\tau(0) = \tau\phi(0) = \frac{\tau}{\sqrt{2\pi}},
\]
as maximum holds at $z=0$ by standard calculus on $h(z) = \varphi_\tau(z) - \mathrm{relu}(z)$: $h'(z) = \Phi(z/\tau) > 0$ for $z < 0$ and $h'(z) = \Phi(z/\tau) - 1 < 0$ for $z > 0$. Therefore, we get
\[
\|w_b - \bar{w}\|_\infty \leq \frac{\tau}{\sqrt{2\pi}} = \frac{\sigma_Z}{\beta\sqrt{2\pi}}.
\]
This yields $|Z_b - Z_{\bar{w}}| \leq \|w_b - \bar{w}\|_\infty \leq \sigma_Z/(\beta\sqrt{2\pi})$. 

On the event $\mathcal{G}_1 := \{|\zeta| \leq \beta/2\}$, we have $|Z_b-1|=|Z_b-Z^*|\leq |\zeta|/\beta \leq 1/2$. Hence, on $\mathcal{G}_1$, we have $Z_b \geq 1/2$.
Hence by Lemma~\ref{lem:smooth_weight} with $\eta = \sigma_Z/(\beta\sqrt{2\pi})$ and $\max(Z_b, Z_{\bar{w}}) \ge Z_b \geq 1/2$, we obtain
\[
|J(\pi_b) - J(\pi_{\bar{w}})| \cdot \mathbf{1}_{\mathcal{G}_1} \leq \frac{2R_{\max} \cdot \sigma_Z/(\beta\sqrt{2\pi})}{1/2} = \frac{4R_{\max}\sigma_Z}{\beta\sqrt{2\pi}}.
\]
Now $\Pr[\mathcal{G}_1^c] = \Pr[|\zeta| > \beta/2] \leq 2\exp(-\beta^2/(8\sigma_X^2))$, and the value gap is bounded trivially by $R_{\max}$ on $\mathcal{G}_1^c$. Then we get
\[
|J(\pi_b) - J(\pi_{\bar{w}})| \leq \frac{4R_{\max}\sigma_Z}{\beta\sqrt{2\pi}} + R_{\max}\Pr[\mathcal{G}_1^c] \leq \frac{4R_{\max}\sigma_Z}{\beta\sqrt{2\pi}} + 2R_{\max}\exp\!\left(-\frac{\beta^2}{8\sigma_X^2}\right).
\]
This yields
\begin{equation}\label{eq:B2}
 |J(\pi_b) - J(\pi_{\bar{w}})| \leq O(R_{\max}\sigma_Z/\beta) + 2R_{\max}\exp\!\left(-\frac{\beta^2}{8\sigma_X^2}\right).   
\end{equation}

\textbf{Term (B3): Finite-sample and truncation correction.} We bound the gap between the Phase~2 target policy $\pi_{\bar{w}}$ and the actual algorithm output $\pi_{\mathrm{PrivITP}}$, which arises from the finite number of candidates and the handling of unbounded Gaussian query noise via truncation. The analysis proceeds by conditioning on the Phase~1 noise $\zeta$, applying a rejection sampling argument for each fixed $\zeta$, and then averaging.

\emph{Conditioning on $\zeta$.} 

Recall $\tilde{\lambda} = \lambda^* + \zeta$ with $\zeta \sim \mathcal{N}(0, \sigma_X^2)$. Both $\pi_{\bar{w}}$ and $\pi_{\mathrm{PrivITP}}$ depend on $\zeta$ through $\tilde{\lambda}$. For each fixed $\zeta$, define the conditional target
\[
\pi_{\bar{w}}^{(\zeta)}(y|x) = \frac{\pi_0(y)\, \bar{w}^{(\zeta)}(x,y)}{Z_{\bar{w}}^{(\zeta)}}, \quad \bar{w}^{(\zeta)}(x,y) = \E_{g'}\bigl[\mathrm{relu}\bigl(\beta^{-1}(r(x,y) + g' - \tilde{\lambda})\bigr)\bigr], \quad Z_{\bar{w}}^{(\zeta)} = \E_{\pi_0}[\bar{w}^{(\zeta)}],
\]
and let $\pi_{\mathrm{PrivITP}}^{(\zeta)}$ denote the algorithm's output distribution conditional on $\zeta$ (marginalizing over Phase~2 candidates, Phase~2 noise, and rejection sampling uniform samples). The full distributions satisfy $J(\pi_{\bar{w}}) = \E_\zeta[J(\pi_{\bar{w}}^{(\zeta)})]$ and $J(\pi_{\mathrm{PrivITP}}) = \E_\zeta[J(\pi_{\mathrm{PrivITP}}^{(\zeta)})]$. By the triangle inequality
\begin{equation}
\label{eq:B3_conditioning}
|J(\pi_{\bar{w}}) - J(\pi_{\mathrm{PrivITP}})| \leq \E_\zeta\bigl|J(\pi_{\bar{w}}^{(\zeta)}) - J(\pi_{\mathrm{PrivITP}}^{(\zeta)})\bigr|.
\end{equation}
It suffices to bound the conditional gap for each $\zeta$ in the ``good event'' $\mathcal{G}_1 := \{|\zeta| \leq \beta/2\}$. Under this event, $Z_b^{(\zeta)} \geq 1/2$. Since $\varphi_\tau(z) - \mathrm{relu}(z) \geq 0$ (by Jensen applied to the convex $\mathrm{relu}$), we get $\bar{w}^{(\zeta)} \geq w_b^{(\zeta)}$ and thus $Z_{\bar{w}}^{(\zeta)} \geq 1/2$. The gap under $\mathcal{G}_1^c$ can be bounded by $R_{\max}$.

\textit{Phase 2 event decomposition (fix $\zeta \in \mathcal{G}_1$).}

The Phase~2 algorithm draws $n$ i.i.d.\ triples $(y'_j, g'_j, U_j)_{j=1}^n$ where $y'_j \sim \pi_0(\cdot|x)$, $g'_j \sim \mathcal{N}(0, \sigma_Z^2)$, $U_j \sim \mathrm{Unif}[0, 1]$. Define the per-trial weight
\[
W_j := \mathrm{relu}\bigl(\beta^{-1}(r(x,y'_j) + g'_j - \tilde{\lambda})\bigr),
\]
and truncation level
\[
M := (R_{\max} + \sigma_Z L - \tilde{\lambda})/\beta, \qquad L := \sqrt{2\log(n/\delta)},
\]
for a parameter $\delta \in (0, 1)$ to be chosen. 

Note that $\lambda^* \in  [-\beta,R_{\max} - \beta]$ from \citet[Lemma C.1]{huang2025bon}. On $\mathcal{G}_1 = \{|\zeta| \leq \beta/2\}$ and hence $\tilde{\lambda} \leq R_{\max} - \beta/2$, hence $M \geq (\sigma_Z L + \beta/2)/\beta \geq 1/2 > 0$ on $\mathcal{G}_1$.

Now, the algorithm accepts at step $j$ if $U_j \leq \min(W_j/M, 1)$. The output $Y^*$ is the first accepted candidate (or a fallback $Y^* \sim \pi_0(\cdot|x)$ if no acceptance occurs). Define three events on the Phase~2 randomness (conditional on Phase~1 randomness $\zeta$):
\begin{align*}
\mathcal{T} &:= \{\exists j \in [n] : W_j > M\} \quad\text{(truncation)}, \\
\mathcal{F} &:= \{\forall j \in [n] : U_j > \min(W_j/M, 1)\} \quad\text{(fallback)}, \\
\mathcal{S} &:= \mathcal{T}^c \cap \mathcal{F}^c \quad\text{(success: no truncation and at least one acceptance)}.
\end{align*}

\textit{Bounding $\Pr[\mathcal{T} \mid \zeta]$.}

On $\mathcal{G}_1$, the event $W_j > M$ implies $r(x,y'_j) + g'_j >  R_{\max} + \sigma_Z L$. Since $r(x,y'_j) \leq R_{\max}$, this implies $g'_j > \sigma_Z L$. By the Gaussian tail bound, for any $L > 0$, we have
\[
\Pr[W_j > M] \le \Pr[g'_j > \sigma_Z L] \le e^{-L^2/2}.
\]
Substituting $ L = \sqrt{2\log(n/\delta)}$, we get $\Pr[g'_j > \sigma_Z L]  \leq \delta/n$. A union bound over the $n$ trials gives
\begin{equation}
\label{eq:B3_trunc_bound}
\Pr[\mathcal{T} \mid \zeta] \leq n \cdot \delta/n = \delta.
\end{equation}

\textit{Bounding $\Pr[\mathcal{F} \mid \zeta]$.}

The per-trial acceptance probability is
\[
p_{\mathrm{acc}} := \E_{y'_j, g'_j}[\Pr[U_j \le  \min(W_j/M, 1)]] =\E_{y',g'}[\min(W/M, 1)].
\]
Note that
\[
p_{\mathrm{acc}} \geq \E[(W/M)\mathbf{1}[W \leq M]] = \frac{1}{M}\bigl(Z_{\bar{w}}^{(\zeta)} - \E[W\mathbf{1}[W > M]]\bigr).
\]

We need an upper bound on the second term $\E[W\mathbf{1}[W > M]]$. Since $M > 0$ on the event $\mathcal{G}_1$, the relu argument in $W$ is positive on the event $\{ W > M\}$. Furthermore, since $\{W > M\} \subseteq \{g' > \sigma_Z L\}$, we have
\[
\E[W\mathbf{1}[W > M]] \leq \E[\beta^{-1}(r(x,y') + g' - \tilde{\lambda})\mathbf{1}[g' > \sigma_Z L]].
\]
Using $r(x,y') \leq R_{\max}$, we get
\begin{align*}
  \E[(R_{\max} + g' - \tilde{\lambda})\mathbf{1}[g' > \sigma_Z L]] = (R_{\max} - \tilde{\lambda})\Pr[g' > \sigma_Z L] + \sigma_Z\E\big[ g'/\sigma_Z\cdot\mathbf{1}[g'/\sigma_Z>L]\big]~.
\end{align*}
Using $R_{\max} - \tilde{\lambda} = \beta M - \sigma_Z L \leq \beta M$, $\Pr[g' > \sigma_Z L] \leq \delta/n$ and $\E\big[ g'/\sigma_Z\cdot\mathbf{1}[g'/\sigma_Z>L]\big]=\phi(L) = e^{-L^2/2}/\sqrt{2\pi} = \delta/(n\sqrt{2\pi})$, we obtain 
\begin{align}\label{eq:imp}
\E[W\mathbf{1}[W > M]] \leq \frac{(\beta M + \sigma_Z)\delta}{\beta n} = \frac{(M + \sigma_Z/\beta)\delta}{n}.
\end{align}
Since $\lambda^* \ge -\beta$ and on $\zeta \ge -\beta/2$ on $\mathcal{G}_1$, we have $\tilde \lambda \ge -3\beta/2$ and hence $M \le (R_{\max}+\sigma_Z L+3\beta/2)/\beta$ on $\mathcal{G}_1$. Hence $M+\sigma_z/B \le 2M$.

On $\mathcal{G}_1$, we also have $Z_{\bar{w}}^{(\zeta)} \geq 1/2$ and so $p_{\mathrm{acc}} \geq \frac{1}{M}\left(\frac{1}{2} -\frac{2M\delta}{n}\right)= \frac{1}{2M} -\frac{2\delta}{n}$.

Now, by independence of trials, we obtain
\begin{align}
\label{eq:B3_fallback_bound}
\Pr[\mathcal{F} \mid \zeta] = \prod_{j=1}^n (1 - p_{\mathrm{acc}}) &\leq \left(1 - \left(\frac{1}{2M} -\frac{2\delta}{n}\right)\right)^n \leq \exp\left(-\frac{n}{2M}+2\delta\right)\nonumber\\&\leq \exp\left(-\frac{\beta n}{2R_{\max}+2\sigma_Z L +3 \beta}+2\delta\right) .
\end{align}

\textit{Output distribution on $\mathcal{S}$.}

On $\mathcal{S}$, no trial has $W_j > M$, so the acceptance rule $U_j \leq W_j/M$ is the standard rejection sampling rule with a valid envelope $M$. Define the \emph{truncated} target policy
\[
\pi_{\bar{w}, M}^{(\zeta)}(y|x) := \frac{\pi_0(y|x) \bar{w}_M^{(\zeta)}(x,y)}{Z_{\bar{w}, M}^{(\zeta)}}, \qquad \bar{w}_M^{(\zeta)}(x,y) := \E_{g}[W(y, g)\mathbf{1}[W(y, g) \leq M]],
\]
with $W(y,g) := \mathrm{relu}\bigl(\beta^{-1}(r(x,y) + g_j - \tilde{\lambda})\bigr)$ and $Z_{\bar{w}, M}^{(\zeta)} = \E_{\pi_0}[\bar{w}_M^{(\zeta)}]$.

\begin{lemma}
   Conditional on $\mathcal{S}$ (and $\zeta$), the Phase 2 output $Y^* \sim \pi_{\bar{w}, M}^{(\zeta)}(\cdot|x)$. 
\end{lemma}

\begin{proof}[Proof of claim]
For a single trial, the joint density of $(y'_j, g'_j)$ on the event ``$W_j \leq M$ and accept at $j$'' is
\[
p(y, g) = \pi_0(y|x)\phi_{\sigma_Z}(g)\mathbf{1}[W(y, g) \leq M]\cdot (W(y, g)/M),
\]
with marginal (integrating over $g$):
\[
\int p_j(y, g)\, dg = \frac{\pi_0(y|x)}{M}\int \phi_{\sigma_Z}(g) W(y, g)\mathbf{1}[W(y, g) \leq M]\, dg = \frac{\pi_0(y|x)\bar{w}_M^{(\zeta)}(x,y)}{M}.
\]
Normalizing, we get $p(y \mid \text{accept at step } j, W_j \leq M) = \pi_0(y|x)\bar{w}_M^{(\zeta)}(x,y)/Z_{\bar{w}, M}^{(\zeta)} = \pi_{\bar{w}, M}^{(\zeta)}(y|x)$. Since each step conditional on acceptance and no truncation gives the same distribution, and the first-accepted-step stopping time does not bias this distribution, the output $Y^* \mid \mathcal{S} \sim \pi_{\bar{w}, M}^{(\zeta)}$ exactly.  
\end{proof}

\textit{Bounding the $\bar{w}$ vs $\bar{w}_M$ gap.}

We bound the pointwise difference
\begin{align*}
   |\bar{w}^{(\zeta)}(x,y) - \bar{w}_M^{(\zeta)}(x,y)| &= \E_{g}[W(y, g)\mathbf{1}[W(y, g) > M]] \\&\leq \frac{(M + \sigma_Z/\beta)\delta}{n} \le \frac{(R_{\max} + \sigma_Z(L+1)+3\beta/2)\delta}{\beta n}  
\end{align*}
by the same computation as in \eqref{eq:imp} and using $M \le (R_{\max}+\sigma_Z L+3\beta/2)/\beta$ on $\mathcal{G}_1$. The same bound applies to $|Z_{\bar{w}}^{(\zeta)} - Z_{\bar{w}, M}^{(\zeta)}|$. Applying Lemma~\ref{lem:smooth_weight} with $\eta = \frac{(R_{\max} + \sigma_Z(L+1)+3\beta/2)\delta}{\beta n} $ and since $\max(Z_{\bar{w}}^{(\zeta)}, Z_{\bar{w}, M}^{(\zeta)}) \geq Z_{\bar{w}}^{(\zeta)} \geq 1/2$ on $\mathcal{G}_1$, we get
\begin{equation}
\label{eq:B3_bar_to_barM}
|J(\pi_{\bar{w}}^{(\zeta)}) - J(\pi_{\bar{w}, M}^{(\zeta)})| \leq \frac{2R_{\max}(R_{\max}+\sigma_Z (L+1)+3/2)\delta}{(\beta n) \cdot (1/2)} = O(R_{\max}(R_{\max}+\sigma_Z L)\delta/\beta n).
\end{equation}

\textit{Assembling the conditional gap.}

Note that
\[
J(\pi_{\mathrm{PrivITP}}^{(\zeta)}) = \E[r^*(x,Y^*)] = \E[r^*(x,Y^*)\mathbf{1}_\mathcal{S}] + \E[r^*(x,Y^*)\mathbf{1}_{\mathcal{S}^c}].
\]
We know that $Y^* \mid \mathcal{S} \sim \pi_{\bar{w}, M}^{(\zeta)}(\cdot|x)$, so $\E[r^*(Y^*)\mathbf{1}_\mathcal{S}] = J(\pi_{\bar{w}, M}^{(\zeta)})\Pr[\mathcal{S}]$. Therefore:
\begin{align*}
J(\pi_{\bar{w}, M}^{(\zeta)}) - J(\pi_{\mathrm{PrivITP}}^{(\zeta)})
&= J(\pi_{\bar{w}, M}^{(\zeta)}) - J(\pi_{\bar{w}, M}^{(\zeta)})\Pr[\mathcal{S}] - \E[r^*(x,Y^*)\mathbf{1}_{\mathcal{S}^c}] \\
&= J(\pi_{\bar{w}, M}^{(\zeta)})\Pr[\mathcal{S}^c] - \E[r^*(Y^*)\mathbf{1}_{\mathcal{S}^c}].
\end{align*}
Taking absolute values and using $r^* \in [0, R_{\max}]$, we get $J(\pi_{\bar{w}, M}^{(\zeta)}) \in [0, R_{\max}]$ and $\E[r^*(Y^*)\mathbf{1}_{\mathcal{S}^c}] \in [0, R_{\max}\Pr[\mathcal{S}^c]]$ and thus
\[
|J(\pi_{\bar{w}, M}^{(\zeta)}) - J(\pi_{\mathrm{PrivITP}}^{(\zeta)})| \leq 2R_{\max}\Pr[\mathcal{S}^c \mid \zeta].
\]
Combining with \eqref{eq:B3_bar_to_barM} via triangle inequality
\begin{equation*}
|J(\pi_{\bar{w}}^{(\zeta)}) - J(\pi_{\mathrm{PrivITP}}^{(\zeta)})| \leq 2R_{\max}\Pr[\mathcal{S}^c \mid \zeta] +  O(R_{\max}(R_{\max}+\sigma_Z L)\delta/\beta n).
\end{equation*}
From \eqref{eq:B3_trunc_bound} and \eqref{eq:B3_fallback_bound}, we have $\Pr[\mathcal{S}^c \mid \zeta] \leq \delta + \exp\left(-\frac{\beta n}{2R_{\max}+2\sigma_Z L +3 \beta}+2\delta\right)$. Setting $\delta \le 1/\sqrt{n}$ and $n$ such that $n/\log n \ge (2R_{\max}+2\sigma_Z \sqrt{3\log n} +3 \beta)/\beta$, we get $\Pr[\mathcal{S}^c \mid \zeta] \le O(1/\sqrt{n})$.

Therefore, on $\mathcal{G}_1$, we have
\begin{equation}\label{eq:B3_conditional_final}
|J(\pi_{\bar{w}}^{(\zeta)}) - J(\pi_{\mathrm{PrivITP}}^{(\zeta)})| \leq  O(R_{\max}(R_{\max}+\sigma_Z L)/\beta n^{3/2})+ O(R_{\max}/\sqrt{n})= O(R_{\max}/\sqrt{n}).
\end{equation}

\textit{Averaging over $\zeta$.}

Observe that $\Pr[\mathcal{G}_1^c] = \Pr[|\zeta| > \beta/2] \leq 2\exp(-\beta^2/(8\sigma_X^2))$ and  $|J(\pi_{\bar{w}}^{(\zeta)}) - J(\pi_{\mathrm{PrivITP}}^{(\zeta)})| \leq R_{\max}$ for all $\zeta$. Combining this with \eqref{eq:B3_conditioning} and \eqref{eq:B3_conditional_final}, we get
\begin{align}\label{eq:B3}
    |J(\pi_{\bar{w}}) - J(\pi_{\mathrm{PrivITP}})| &\leq \E_\zeta\bigl[|J(\pi_{\bar{w}}^{(\zeta)}) - J(\pi_{\mathrm{PrivITP}}^{(\zeta)})| \cdot \mathbf{1}_{\mathcal{G}_1}\bigr] + R_{\max}\Pr[\mathcal{G}_1^c]\nonumber\\
    &\leq O(R_{\max}/\sqrt{n}) + 2R_{\max} \exp(-\beta^2/(8\sigma_X^2)).
\end{align}
Combining \eqref{eq:A}, \eqref{eq:B1}, \eqref{eq:B2} and \eqref{eq:B3}, we obtain the regret
\[
\mathrm{Reg}(x) \leq \sqrt{C^{\pi^*}\varepsilon_{\mathrm{RM}}^2} + \beta C^{\pi^*} + \frac{\varepsilon_{\mathrm{RM}}^2}{\beta} + O\left(R_{\max}\left(\frac{\sigma_X+\sigma_Z}{\beta} +\frac{1}{\sqrt{n}}+\exp\left(-\frac{\beta^2}{8\sigma_X^2}\right)\right)\right). 
\]
For $\beta \geq  2\sigma_X$, we can subsume the negative exponential term in $\sigma_X/\beta$ and obtain
\begin{align*}
  \mathrm{Reg}(x) \leq \sqrt{C^{\pi^*}\varepsilon_{\mathrm{RM}}^2} + \beta C^{\pi^*} + \frac{\varepsilon_{\mathrm{RM}}^2}{\beta} + \frac{R_{\max}(\sigma_X+\sigma_Z)}{\beta} + \frac{R_{\max}}{\sqrt{n}}~.  
\end{align*}
Combining the requirement on $n$ for Term (A) an (B3), the bounds holds for $n = \max \left\lbrace\Omega\!\left(\frac{R_{\max}}{\beta}\log\!\left(\frac{R_{\max}}{\beta \,\epsilon_{\mathrm{RM}}}\right)\right), \widetilde \Omega\left(\frac{R_{\max}+\sigma_Z}{\beta}\right)\right\rbrace$.
\end{proof}

\section{Multi-Query Deployment: Proof of proposition~\ref{thm:privitp_fsrc_comp}}
\label{app:composition}

\begin{algorithm}[t]
\caption{PrivITP with Filtered Self-Reporting Composition}
\label{alg:privitp_fsrc}
\begin{algorithmic}[1]
\REQUIRE Total privacy budget $(\varepsilon_{\mathrm{total}}, \delta)$; stream of prompts $x_1, x_2, \ldots$; PrivITP hyperparameters $(\beta, \sigma_X, \sigma_Z, L)$
\STATE Set cumulative privacy spend $\varepsilon_{\mathrm{spent}} \leftarrow 0$
\STATE Compute per-query worst-case bound $\varepsilon_{\max}$ from \Cref{thm:privitp_priv} at $t = n$
\FOR{$\tau = 1, 2, \ldots$}
    \IF{$\varepsilon_{\mathrm{spent}} + \varepsilon_{\max} \geq \varepsilon_{\mathrm{total}}$}
        \STATE \textbf{HALT} 
    \ENDIF
    \STATE Run PrivITP on $x_\tau$ (\Cref{alg:privitp}), obtaining response $y^{*(\tau)}$ and halting time $t_\tau$
    \STATE Compute ex-post privacy cost $\varepsilon_\tau \leftarrow \varepsilon_1 + \varepsilon_2^{\mathrm{post}}(t_\tau)$ from \Cref{thm:privitp_priv}
    \STATE Update $\varepsilon_{\mathrm{spent}} \leftarrow \varepsilon_{\mathrm{spent}} + \varepsilon_\tau$
    \STATE Release $y^{*(\tau)}$
\ENDFOR
\end{algorithmic}
\end{algorithm}

\begin{proof}
Privacy follows directly from~\citet[Theorem~11]{lebensold2024privacy}: \Cref{alg:privitp_fsrc} is an instance of their Filtered Composition framework with $\varepsilon_{\max,t} = \varepsilon_{\max}$ (worst-case PrivITP bound) and $\varepsilon_{\mathrm{post},t}(o_t) = \varepsilon_1 + \varepsilon_2^{\mathrm{post}}(t_\tau)$ (ex-post PrivITP bound from \Cref{thm:privitp_priv}). Since PrivITP is both $(\varepsilon_{\max}, \delta)$-pDP (via its worst-case bound) and $\varepsilon_{\mathrm{post},\tau}$-ex-post-DP (via \Cref{thm:privitp_priv}), the filter condition ensures the cumulative ex-post cost never exceeds $\varepsilon_{\mathrm{total}}$, and the pDP bound handles the tail risk in the final query.

The bound is a direct consequence of the stopping rule: the algorithm halts when $\varepsilon_{\mathrm{spent}} + \varepsilon_{\max} \geq \varepsilon_{\mathrm{total}}$, so by Wald's identity applied to the cumulative cost process, $\E[T^*] \cdot \E[\varepsilon_\tau] \geq \varepsilon_{\mathrm{total}} - \varepsilon_{\max}$, giving the stated lower bound (ignoring the boundary term).
\end{proof}

\begin{remark}[Comparison with privacy filters]
\label{rem:fsrc_vs_filters}
An alternative to FSRC is the privacy filter framework of~\citet{rogers2023adaptive,whitehouse2023fully}, which provides a stopping rule based on the cumulative privacy cost computed via advanced composition. The key difference is that privacy filters use the worst-case per-query cost, while FSRC uses the actual ex-post cost. For PrivITP, where the ex-post cost is typically much smaller than the worst-case cost, FSRC yields strictly tighter deployment budgets. \citet{lebensold2024privacy} demonstrate this advantage empirically, showing improvements in query throughput at the same total privacy budget.
\end{remark}

\section{PrivBoN with Gaussian Noise}\label{sec:app_gauss}

Gumbel noise in PrivBoN yields pure $\epsilon$-DP and converges to a closed-form tilted
policy. A natural question is whether Gaussian
noise---standard in most DP analyses---offers comparable guarantees. We show here that
Gaussian noise has qualitatively different properties that make it a less clean match
for inference-time alignment, which motivates our choice of Gumbel noise.




\begin{lemma}[Privacy]
\label{thm:gaussian_priv}
Let $\hat{r}_D(x, \cdot) \in [0, R_{\max}]$ with sensitivity $\Delta_r$. Then, \Cref{alg:mbon} with $g_i \stackrel{\text{iid}}{\sim} \mathcal{N}(0, \sigma^2)$ satisfies \textbf{pure} $\varepsilon$-DP with
\begin{equation}
\label{eq:gaussian_mbon_eps}
\varepsilon = \log\frac{\displaystyle \E_{z \sim \mathcal{N}(0,1)}\!\left[\Phi\!\left(z - \frac{ R_{\max}- 2\Delta_r}{\sigma}\right)^{n-1}\right]}{\displaystyle \E_{z \sim \mathcal{N}(0,1)}\!\left[\Phi\!\left(z - \frac{R_{\max}}{\sigma}\right)^{n-1}\right]},
\end{equation}
where $\Phi$ denotes the standard Gaussian CDF. 
\end{lemma}
\begin{proof}
PrivBon with Gaussian noise is an instance of Gaussian Report Noisy Max applied to the $n$ reward scores $\hat{r}_D(x, y_1), \ldots, \hat{r}_D(x, y_n)$, each bounded in $[0, R_{\max}]$ with sensitivity $\Delta_r$. The result follows directly from  \citet[Theorem 8]{lebensold2024privacy}, which establishes pure $\varepsilon$-DP for Gaussian Report Noisy Max over bounded queries with the closed-form expression in~\eqref{eq:gaussian_mbon_eps}. The proof identifies the worst-case reward configurations on adjacent datasets $D \sim D'$ via a monotonicity argument, then expresses the privacy loss as a ratio of Gaussian expectations over products of CDFs.
\end{proof}
The bound is computable via numerical integration (see \citealp{lebensold2024privacy}, Section 3.4) and depends on the reward range $R_{\max}$, the sensitivity $\Delta_r$, the noise scale $\sigma$, and the number of candidates $n$ -- but crucially does not require the approximate-DP slack $\delta$.  Unlike Gumbel noise, the privacy budget $\epsilon$ depends
on $n$ through a numerical integral, though the growth is very slow in practice.


\begin{proposition}[Policy]
\label{prop:gaussian_policy}
PrivBoN with Gaussian noise induces the selection probability $w_i = \E_Z[\prod_{j\neq i}\Phi(Z+(\hat{r}_D(x,y_i)-\hat{r}_D(x,y_j))/\sigma)]$, $i \in [n]$, where $Z \sim \mathcal{N}(0,1)$. 
\end{proposition}

\begin{proof}
Define $r_i := \hat r_D(x,y_i)$ for brevity. The derivation follows the same symmetry argument as \Cref{prop:gumbel_policy}. The only difference is the selection probability. Candidate $i$ wins if $r_i + g_i \geq r_j + g_j$ for all $j \neq i$. Conditioning on $g_i = \sigma \cdot z$, where $z \sim \mathcal{N}(0,1)$, we get
\[
\Pr(i \text{ wins} \mid g_i = \sigma \cdot z) = \prod_{j \neq i} \Pr(g_j \leq \sigma \cdot z + r_i - r_j) = \prod_{j \neq i} \Phi\!\left(z + \frac{r_i - r_j}{\sigma}\right).
\]
Integrating over $z$, we have
\[
w_i:=\Pr(i \text{ wins}) = \int_{-\infty}^{\infty} \varphi(z)\prod_{j \neq i}\Phi\!\left(z + \frac{r_i-r_j}{\sigma}\right)dz = \E_{Z\sim\mathcal{N}(0,1)}\!\left[\prod_{j\neq i}\Phi\!\left(Z + \frac{r_i-r_j}{\sigma}\right)\right].
\]
The rest (symmetry reduction, tower property, marginal policy) is identical to the Gumbel case with $w_i$ replacing the softmax weight.
\end{proof}
This probit integral has no closed-form simplification for general $n$. Unlike the Gumbel case, PrivBON with Gaussian noise does not converge to a tilted policy as $n \to \infty$: Gaussian extreme value theory dominates, and the selection concentrates on the candidate with the highest $\hat{r}_D$, recovering hard BoN in the limit rather than a smoothed distribution.

\section{Details on Experiments}
\label{app:experiments}

\subsection{Experimental Setup}

\textbf{Compute resources.} All experiments were run using 4 parallel 48GB GPU cards and one 140 GB GPU card. The time to generate responses and score them depends on the size of the base and reward models -- the larger the models, the longer the inference time.

\textbf{Datasets.}
We evaluate on three benchmarks with verifiable ground-truth correctness, assigning
$r^\star(x,y) = 1$ for correct answers and $0$ otherwise. GSM8K~\citep{cobbe2021gsm8k}
is the standard $1$K-prompt grade-school math test split. MMLU~\citep{hendrycks2020mmlu}
covers $\sim\!100$ college-level multiple-choice questions each in math and chemistry.
MATH-500 \citep{lightman2023lets}, a widely adopted, representative subset of 500 test problems from the original MATH dataset \citep{hendrycks2021math}. We use zero-shot Chain-of-Thought prompting \citep{wei2022cot} throughout.

\textbf{Reward models.}
We use four reward models as the proxy $\hat{r}_\mathcal{D}$:
Oasst-RM (Pythia-1.4b), Gemma-RM (Gemma-2-2b),
Llama-RM (Llama-3-3b) and Armo-RM (Llama-3-8b).
The reference policy $\pi_{\mathrm{ref}}$ (denoted interchangeably $\pi_0$)
is Gemma-2-2B-Instruct and Phi-3-Mini-Instruct.

\textbf{Sampling and evaluation.}
For each prompt, we sample $10$K responses at temperature~$1$ from $\pi_{\mathrm{ref}}$,
then draw $M = 50$ bootstrap replicates of $n$ responses and run each algorithm on
every replicate. Per-prompt accuracy is the fraction of replicates yielding a correct
answer; reported accuracy averages over prompts, with standard errors estimated across prompt distribution.

\subsection{Ablation: Noise Scale for PrivBoN}

The noise scale $\sigma$ in PrivBoN controls both the privacy budget
($\varepsilon = 2\Delta_r/\sigma$) and the strength of KL-regularization, with
\Cref{thm:gumbel_regret} predicting an optimal $\sigma^* = \varepsilon_{\mathrm{RM}}/\sqrt{C^{\pi^*}}$
balancing under-regularization (reward hacking) and over-regularization (suppressing
the reward signal). Since $C^{\pi^*}$ is not computable in practice, to select $\sigma$ for the experiments above, we sweep
$\sigma$ and
report accuracy lift, and expected proxy reward in \Cref{fig:sigma_ablation} on a held out set. The resulting profile is
single-peaked: small $\sigma$ allows reward hacking (accuracy drops as PrivBoN
approaches BoN's behavior), while a large $\sigma$ over-regularizes (accuracy drops
toward $\pi_{\mathrm{ref}}$). We use the empirically best
$\sigma$ for the main results in \Cref{tab:rm_dataset_sweep}. The same $\sigma$ is also taken as a total noise in PrivITP for a fair comparison.

\subsection{Ablation: Regularization for ITP}

The regularization parameter $\beta$ in ITP and PrivITP plays a role analogous to $\sigma$
in PrivBoN, with \Cref{thm:privitp_regret} predicting an optimal
$\beta^*$ that balances
$\chi^2$-bias against overoptimization and privacy noise. Sweeping
$\beta \in \{0.0005,0.001,0.005,0.01,0.05,0.1,0.5\}$, 
and
reporting accuracy lift, and expected proxy reward in \Cref{fig:beta_ablation}, we find the empirically best $\beta$ on a held out set.
Crucially, we observe that
PrivITP's optimal $\beta$ is essentially independent of the privacy budget
$\varepsilon$---validating Theorem~\ref{thm:privitp_regret}'s decoupling claim that regularization and privacy can be tuned separately. We use the empirically best
$\beta$ for the main results in \Cref{tab:rm_dataset_sweep}.

\subsection{Algorithm comparison across reward models and datasets}

We compare four algorithms (BoN, PrivBoN, ITP, and PrivITP) across four
reward models (Oasst, Gemma, Llama, Armo) and three datasets (GSM8K, MMLU, MATH),
fixing the base policy to Gemma-2-2b-Instruct and Phi-3-Mini-Instruct. For each (algorithm, RM, dataset) configuration, we sweep the number of candidates $N \in \{2, 4, 8, 16, 32, 64, 128, 256, 512, 1024, 2048, 4096\}$ (interchangeably denoted by $n$). Private algorithms are calibrated to the empirical best $\sigma$ obtained for PrivBoN from the ablation study; non-private ITP serves as the regret-optimal skyline. We report
two quantities per configuration: (a)~accuracy lift over $\pi_{\mathrm{ref}}$,
measuring \emph{true-reward} performance, and (b)~expected proxy reward
$\mathbb{E}[\hat{r}_\mathcal{D}]$ under each algorithm's output distribution,
measuring how aggressively each algorithm exploits the reward model.
\Cref{fig:gsm8k_comparison,fig:mmlu_comparison,fig:math_comparison} present
the results for the base policy Gemma-2-2b-Instruct, one figure per dataset, each with a $4 \times 2$ grid: rows index
reward models, the left column shows accuracy lift vs.~$N$, the right column
shows expected proxy reward vs.~$N$. In a similar way, \Cref{fig:gsm8k_comparison_phi3,fig:math_comparison_phi3,fig:mmlu_comparison_phi3}  present the results for the base policy Phi-3-Mini-Instruct. \Cref{fig:privacy-vs-utility,fig:privacy-vs-utility_MMLU,fig:privacy-vs-utility_MATH} compare \% accuracy lift of ITP vs. PrivITP for different values of $\sigma$ on GSM8K, MMLU and MATH datasets respectively under Gemma2-2B-Instruct base policy, showing privacy-utility trade-off. \Cref{tab:rm_dataset_sweep} summarizes the results at $N = 2^{12}$ across all datasets and RMs for Gemma2-2B-Instruct base policy whereas \Cref{tab:rm_dataset_sweep_llama} summarizes the results at $N = 2^{12}$ across all datasets and RMs for Llama-3.2-3B-Instruct base policy

Three patterns hold consistently (except, in some cases, for the harder MATH dataset).
\emph{First}, BoN's accuracy lift exhibits the reward-hacking signature: it rises
with $N$ initially, then \emph{decreases} starts decreasing, while its expected
proxy reward continues to climb monotonically. The two curves diverge---a
demonstration of Goodhart's law---confirming that BoN exploits proxy errors at
the cost of true reward. \emph{Second}, PrivBoN, ITP, and PrivITP are all
scaling-monotone: accuracy lift improves (or plateaus) with $N$ and proxy reward
saturates rather than diverging. \emph{Third}, PrivITP dominates PrivBoN on every (RM, dataset) configuration where both
provide positive lift, and on GSM8K it recovers between 60\% and 95\% of
the non-private ITP skyline; on MMLU and MATH, where the skyline itself
is small or noisy, PrivITP closely tracks ITP within standard error. The proxy-reward panels make the mechanism visible:
PrivITP's expected proxy reward saturates at a level above PrivBoN's, indicating
that PrivITP's $\chi^2$-pessimism allows more aggressive exploitation of
\emph{informative} reward signals while still preventing overfitting to errors.
Similar to \citet{huang2025bon}, we also observe that ITP variants tend to have higher average performance than BoN variants, although in many instances this difference is not statistically significant.

\subsection{FSRC composition}
We supplement the main-text composition experiment (\Cref{fig:main_all}(bottom right))
with sweeps across reward models using Gemma2-2B-Instruct base policy
on GSM8K dataset with $\varepsilon_{\mathrm{total}} = 50$. See~\Cref{fig:fsrc_comp}.
Three patterns hold. (i)~Absolute query counts increase with $\sigma$ for
both FSRC and standard composition, since larger noise lowers the per-query
privacy cost faster than it raises the expected halting time $\mathbb{E}[t]$; both compositors benefit, but FSRC retains its advantage. (ii)~The FSRC
answers $\sim 3\times$ more queries, reflecting the gap between worst-case $\varepsilon_{\max}$
and realized $\mathbb{E}[\varepsilon_\tau]$. (iii)~Empirical halting time admits $t_\tau \ll n$, confirming faster rejection ($\mathbb{E}[t_\tau] = O(1)$). These results confirm that the FSRC gain
is structural -- driven by the ex-post bound's halting-time dependence.

\begin{table}[t]
\centering
\caption{Percentage lift in accuracy over base policy $\text{Gemma-2-2B-Instruct}$ at $n = 2^{12}$
candidates. Values are mean $\pm$ standard
error over all prompts across the GSM8K, MMLU, and MATH test splits. Same (RM, dataset) configuration uses the same noise $\sigma$ for private variants and the same regularization $\beta$ for ITP variants. 
}
\label{tab:rm_dataset_sweep}
\begin{tabular}{@{}llcccc@{}}
\toprule
\textbf{Dataset} & \textbf{Algorithm} & \textbf{Oasst-RM} & \textbf{Gemma-RM} & \textbf{Llama-RM} & \textbf{Armo-RM} \\
\midrule
\multirow{4}{*}{GSM8K}
 & BoN     & $-6.46 \pm 0.92$ & $\phantom{-}0.75 \pm 0.97$ & $16.70 \pm 1.12$ & $3.43 \pm 1.07$ \\
 & PrivBoN & $\phantom{-}5.94 \pm 0.88$ & $10.95 \pm 0.90$ & $21.46 \pm 0.92$ & $7.65 \pm 0.87 $\\
 & PrivITP & $6.77 \pm 0.90$ & $11.81 \pm 0.91$ & $25.67 \pm 0.93$ & $14.16 \pm 0.90$ \\
 \cmidrule(l){2-6}
 & ITP  & \textit{$7.36 \pm 0.90$} & \textit{$12.50 \pm 0.92$} & \textit{$29.69 \pm 0.93$} & $23.16 \pm 0.93$ \\
\midrule
\multirow{4}{*}{MMLU}
 & BoN     & $10.55 \pm 2.21$ & $11.09 \pm 2.24$ & $8.59 \pm 2.18$ & $23.18 \pm 2.89$\\
 & PrivBoN & $10.14 \pm 1.93$ & $9.04 \pm 1.71$ & $4.01 \pm 1.63$ & $1.12 \pm 1.48$ \\
 & PrivITP & $8.26 \pm 1.68$ & $14.07 \pm 1.93$ & $8.89 \pm 1.91$ & $2.49 \pm 1.57$ \\
 \cmidrule(l){2-6}
 & ITP  & $8.45 \pm 1.69$ & $15.44 \pm 2.01$ & $9.03 \pm 1.99$ & $15.69 \pm 2.31$ \\
\midrule
\multirow{4}{*}{MATH-500}
 & BoN     & $\phantom{-}1.32 \pm 1.27$ & $7.38 \pm 1.73$ & $20.48 \pm 1.86$ & $17.77 \pm 1.81$ \\
 & PrivBoN & $-0.48 \pm 1.02$ & $\phantom{-}3.11 \pm 1.33$ & $\phantom{0}7.53 \pm 1.33$ & $9.34 \pm 1.41$ \\
 & PrivITP & $-0.42 \pm 1.02$ & $5.44 \pm 1.43$ & $14.03 \pm 1.53$ & $15.01 \pm 1.57$ \\
 \cmidrule(l){2-6}
 & ITP  & \textit{$\phantom{-}0.01 \pm 1.05$} & \textit{$\phantom{-}6.61 \pm 1.51$} & \textit{$16.56 \pm 1.64$} & $16.98 \pm 1.68$ \\
\bottomrule
\end{tabular}
\end{table}

\begin{table}[t]
\centering
\caption{Percentage lift in accuracy over base policy $\text{Llama-3.2-3B-Instruct}$ at $n = 2^{12}$
candidates. Values are mean $\pm$ standard
error over all prompts across the GSM8K, MMLU, and MATH test splits. Same (RM, dataset) configuration uses the same noise $\sigma$ for private variants and the same regularization $\beta$ for ITP variants. 
}
\label{tab:rm_dataset_sweep_llama}
\begin{tabular}{@{}llcccc@{}}
\toprule
\textbf{Dataset} & \textbf{Algorithm} & \textbf{Oasst-RM} & \textbf{Gemma-RM} & \textbf{Llama-RM} & \textbf{Armo-RM} \\
\midrule
\multirow{4}{*}{GSM8K}
 & BoN     & $ -0.41 \pm 1.06 $ & $ 4.68 \pm 1.00$ & $ 19.96 \pm 0.75 $ & $ 16.94 \pm 0.14 $ \\
 & PrivBoN & $ 3.02 \pm 0.84 $ & $ 5.04 \pm 0.82$ & $ 17.71 \pm 0.69 $ & $ 11.96 \pm 0.11 $ \\
 & PrivITP & $ 4.07 \pm 0.85$ & $ 6.58 \pm 0.84$ & $ 19.63 \pm 0.69 $ & $ 14.26 \pm 0.11 $ \\
 \cmidrule(l){2-6}
 & ITP     & $ 3.98 \pm 0.87 $ & $ 6.17 \pm 0.87 $& $ 19.75 \pm 0.71 $ & $ 16.31 \pm 0.11 $ \\
\midrule
\multirow{4}{*}{MMLU}
 & BoN     & $7.66 \pm 2.68$ & $7.83 \pm 2.88$ & $-2.08 \pm 2.71$ & $13.80 \pm 3.08$\\
 & PrivBoN & $3.26 \pm 1.94$ & $6.70 \pm 2.09$ & $2.55 \pm 1.97$ & $13.78 \pm 3.08$ \\
 & PrivITP & $5.76 \pm 2.07$ & $8.78 \pm 2.27$ & $4.64 \pm 2.10$ & $13.87 \pm 2.68$ \\
 \cmidrule(l){2-6}
 & ITP  & $7.90 \pm 2.25$ & $9.14 \pm 2.40$ & $9.27 \pm 2.45$ & $8.18 \pm 2.28$ \\
\midrule
\multirow{4}{*}{MATH-500}
 & BoN     & $ -4.48 \pm 1.72$ & $ 3.77 \pm 1.87 $ & $ -8.16 \pm 1.82$ & $ -3.95 \pm 1.86$ \\
 & PrivBoN & $ 0.34 \pm 1.58 $ & $ 2.70 \pm 1.62$ & $ 2.16 \pm 1.63 $ & $ 1.35 \pm 1.59 $ \\
 & PrivITP & $ 0.46 \pm 1.58$ & $ 4.14 \pm 1.66$ & $ 3.83 \pm 1.67$ & $ 2.21 \pm 1.61 $ \\
 \cmidrule(l){2-6}
 & ITP     & $ 0.73 \pm 1.60 $ & $5.79 \pm 1.71 $ & $ 6.52 \pm 1.80 $ & $ 6.24 \pm 1.73$ \\
\bottomrule
\end{tabular}
\end{table}

\begin{figure}[htbp]
    \centering
    
    \begin{subfigure}[b]{\textwidth}
        \centering
        \includegraphics[width=\textwidth, height=5cm]{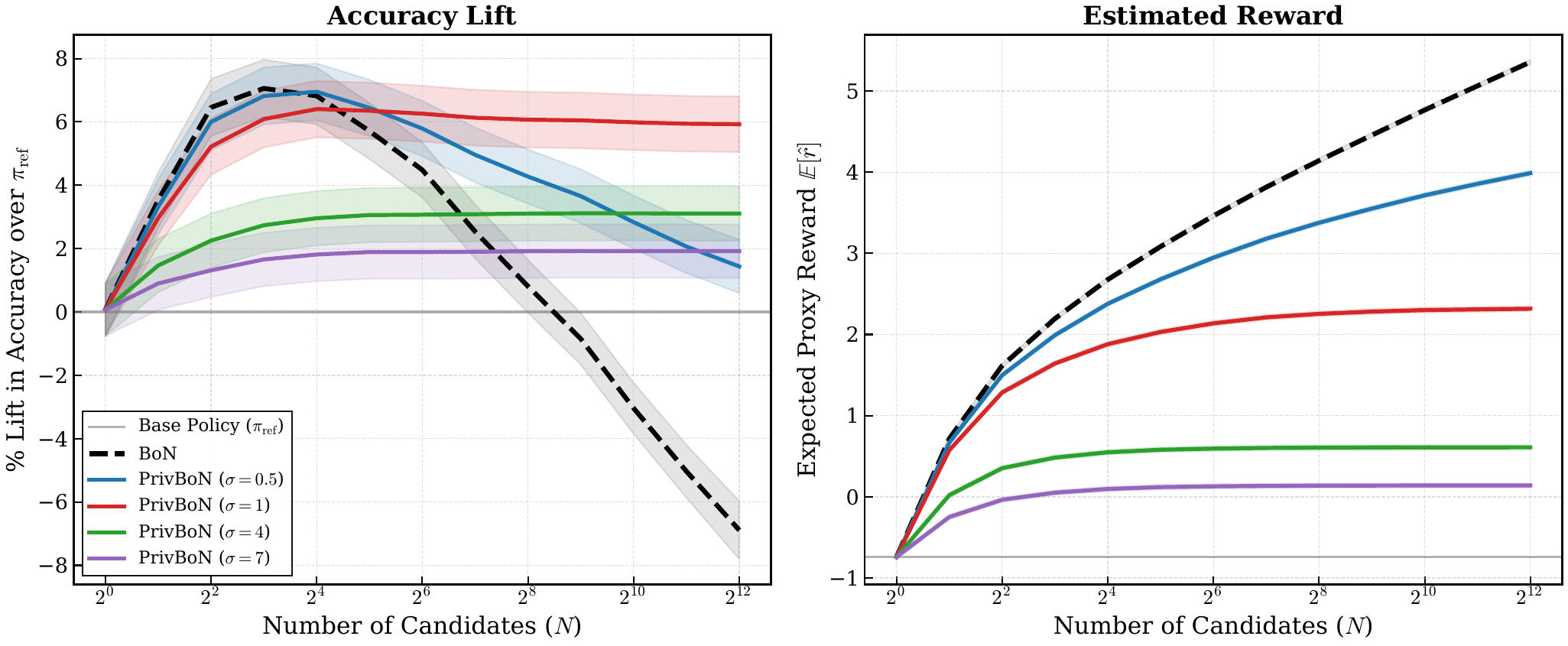}
        \caption{Oasst-RM}
        \label{fig:plot1}
    \end{subfigure}
    
    \begin{subfigure}[b]{\textwidth}
        \centering
\includegraphics[width=\textwidth,height=5cm]{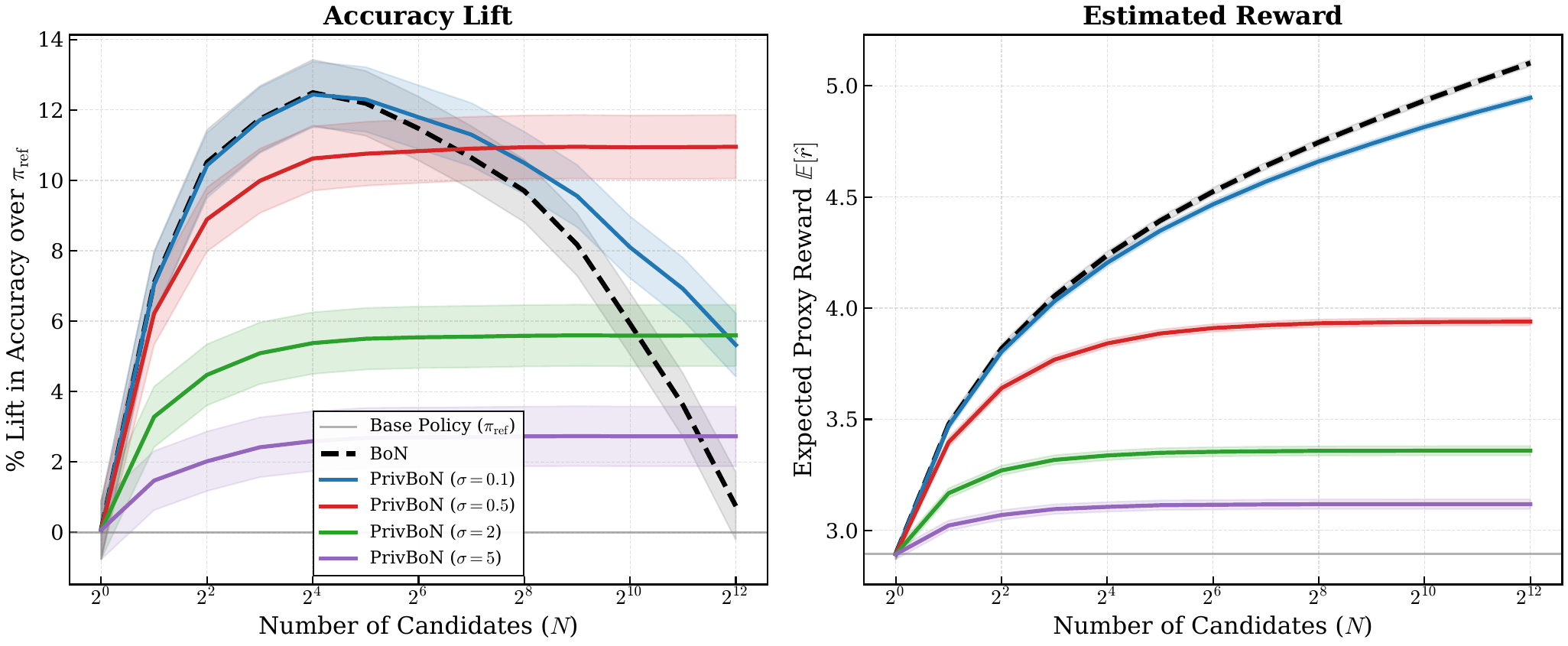}
        \caption{Gemma-RM}
        \label{fig:plot2}
    \end{subfigure}
    
    \begin{subfigure}[b]{\textwidth}
        \centering
        \includegraphics[width=\textwidth, height=5cm]{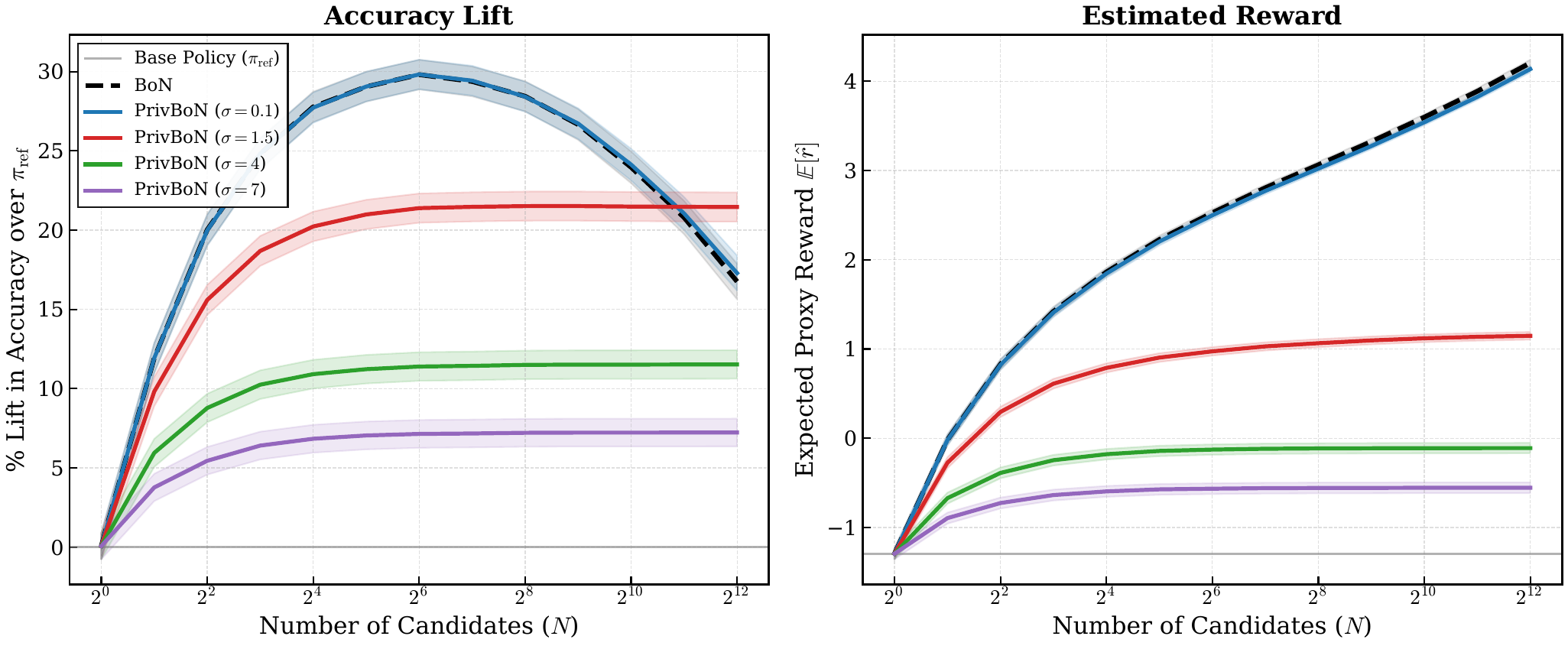}
        \caption{Llama-RM}
        \label{fig:plot3}
    \end{subfigure}
    
    \begin{subfigure}[b]{\textwidth}
        \centering
        \includegraphics[width=\textwidth, height=5cm]{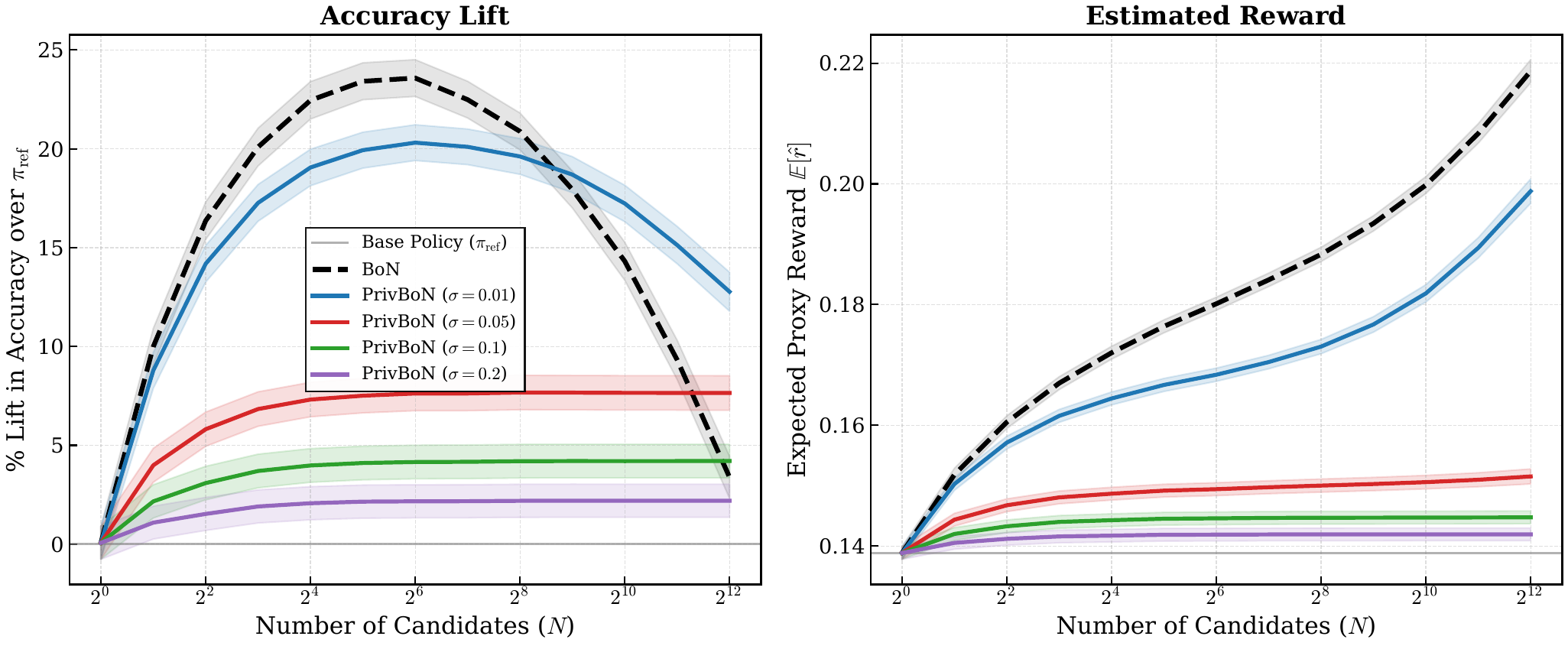}
        \caption{Armo-RM}
        \label{fig:plot4}
    \end{subfigure}
    \caption{Ablation on noise level $\sigma$ for PrivBoN on GSM8K dataset with Gemma-2-2b-Instruct as base model}
    \label{fig:sigma_ablation}
\end{figure}

\begin{figure}[htbp]
    \centering
    
    \begin{subfigure}[b]{\textwidth}
        \centering
        \includegraphics[width=\textwidth, height=5cm]{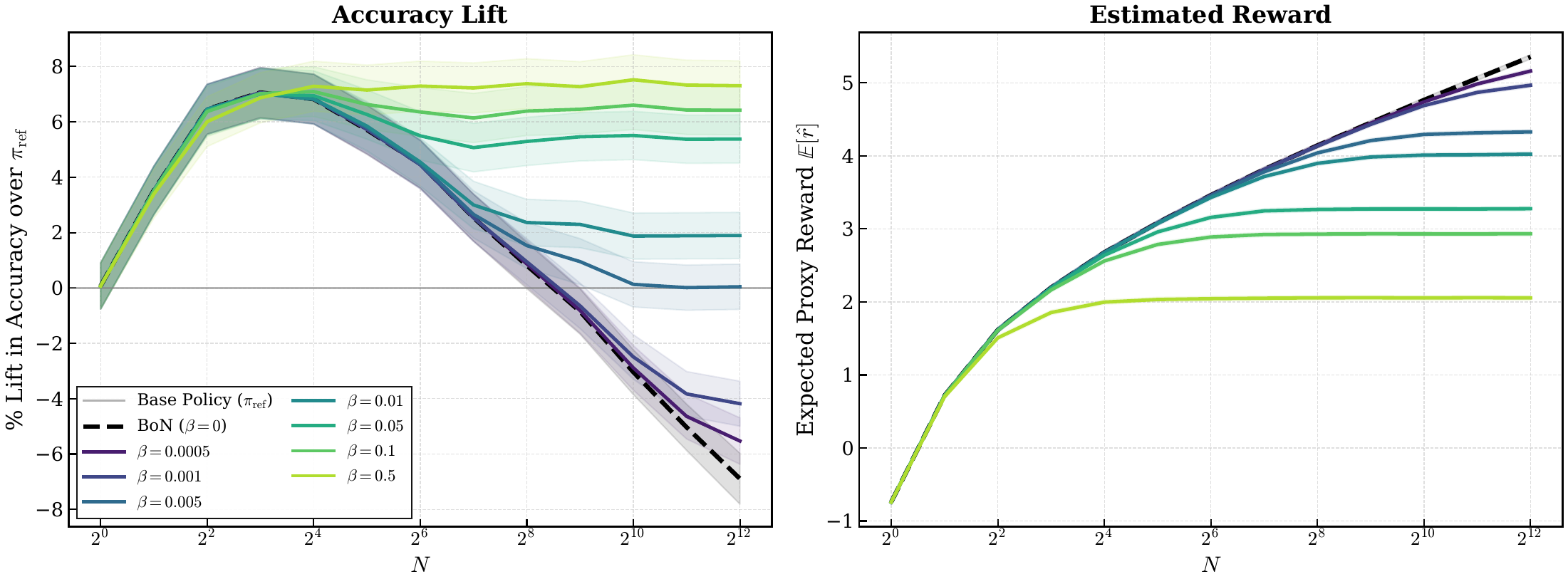}
        \caption{Oasst-RM}
        \label{fig:plot1}
    \end{subfigure}
    
    \begin{subfigure}[b]{\textwidth}
        \centering
\includegraphics[width=\textwidth,height=5cm]{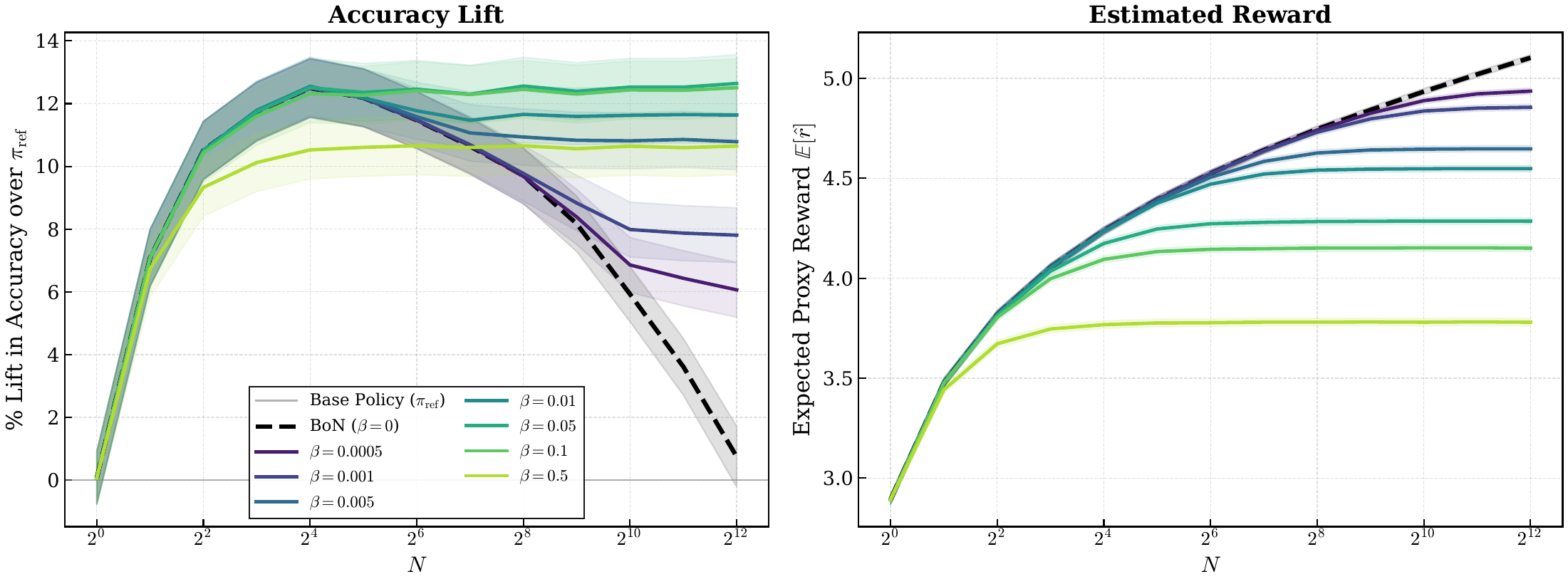}
        \caption{Gemma-RM}
        \label{fig:plot2}
    \end{subfigure}
    
    \begin{subfigure}[b]{\textwidth}
        \centering
        \includegraphics[width=\textwidth, height=5cm]{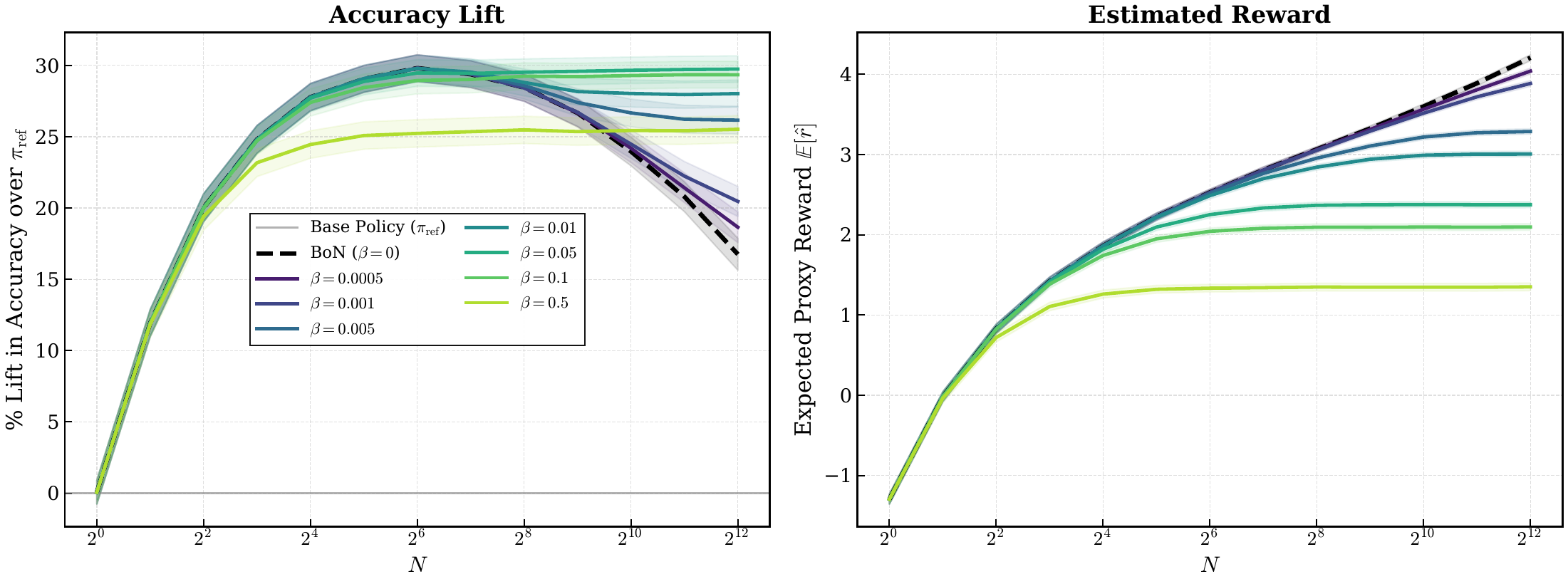}
        \caption{Llama-RM}
        \label{fig:plot3}
    \end{subfigure}
    
    \begin{subfigure}[b]{\textwidth}
        \centering
        \includegraphics[width=\textwidth, height=5cm]{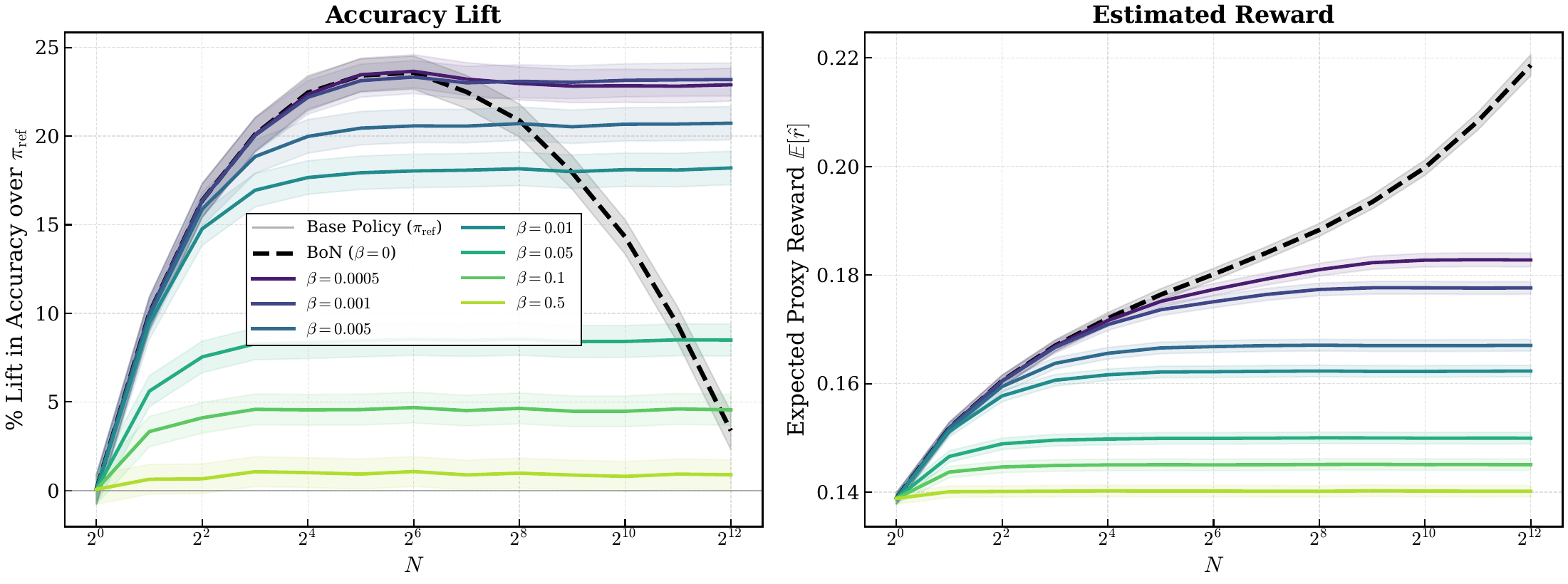}
        \caption{Armo-RM}
        \label{fig:plot4}
    \end{subfigure}
    \caption{Ablation on regularization $\beta$ for ITP on GSM8K dataset with Gemma-2-2b-Instruct as base model}
    \label{fig:beta_ablation}
\end{figure}

\begin{figure}[htbp]
    \centering
    
    \begin{subfigure}[b]{\textwidth}
        \centering
        \includegraphics[width=\textwidth, height=5cm]{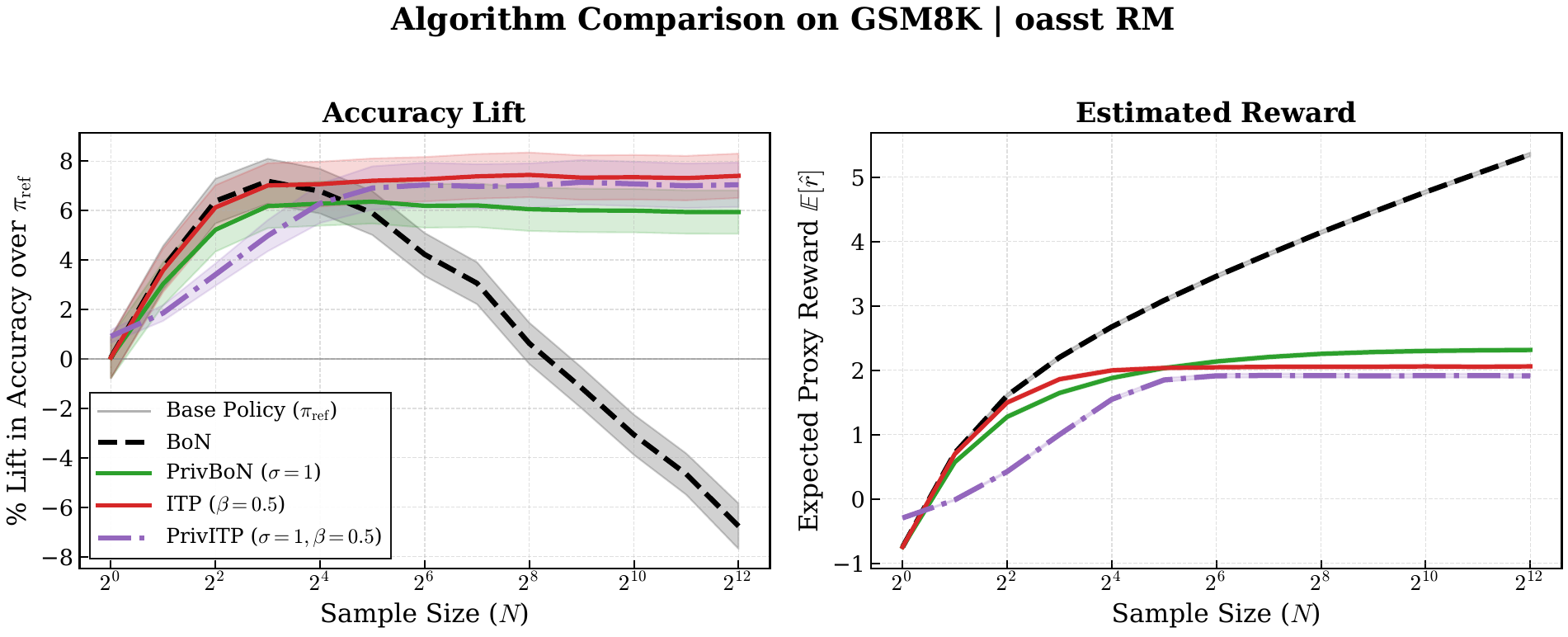}
        \label{fig:plot1}
    \end{subfigure}
    
    \begin{subfigure}[b]{\textwidth}
        \centering    
        \includegraphics[width=\textwidth,height=5cm]{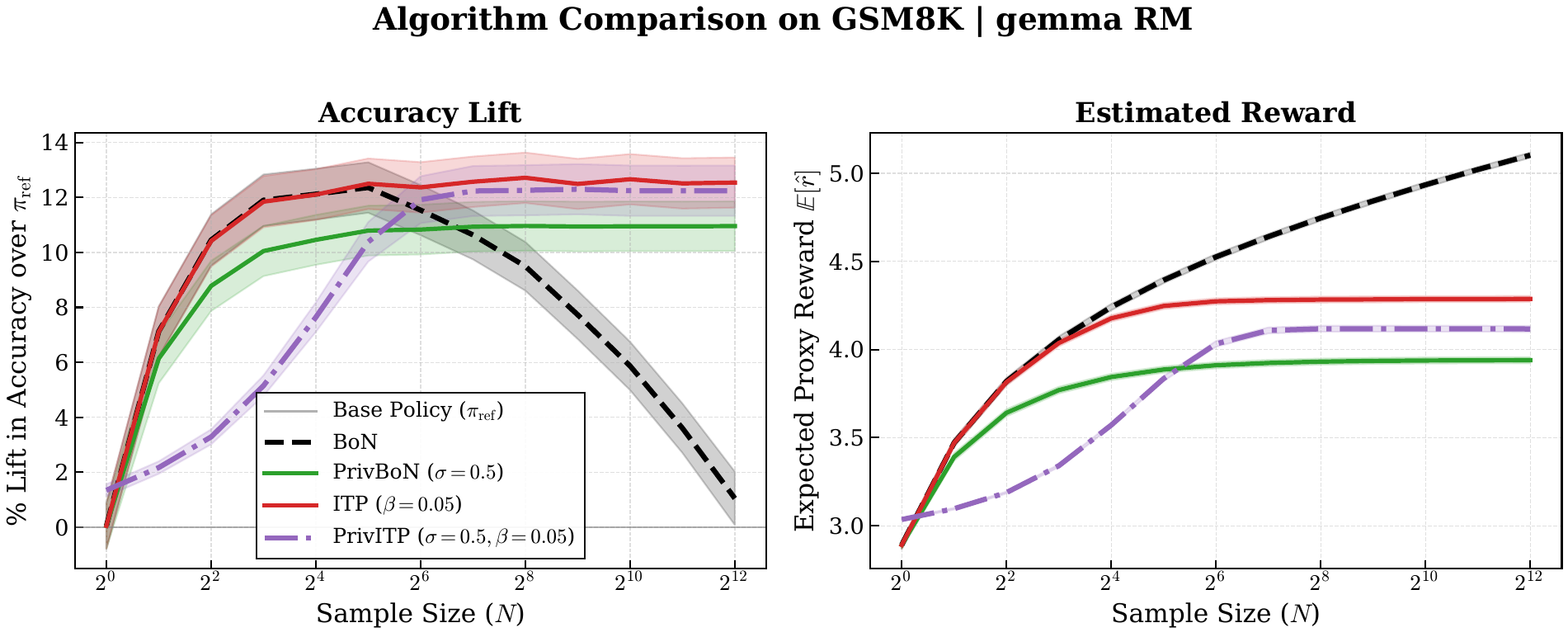}
       
        \label{fig:plot2}
    \end{subfigure}
    
    \begin{subfigure}[b]{\textwidth}
        \centering
        \includegraphics[width=\textwidth, height=5cm]{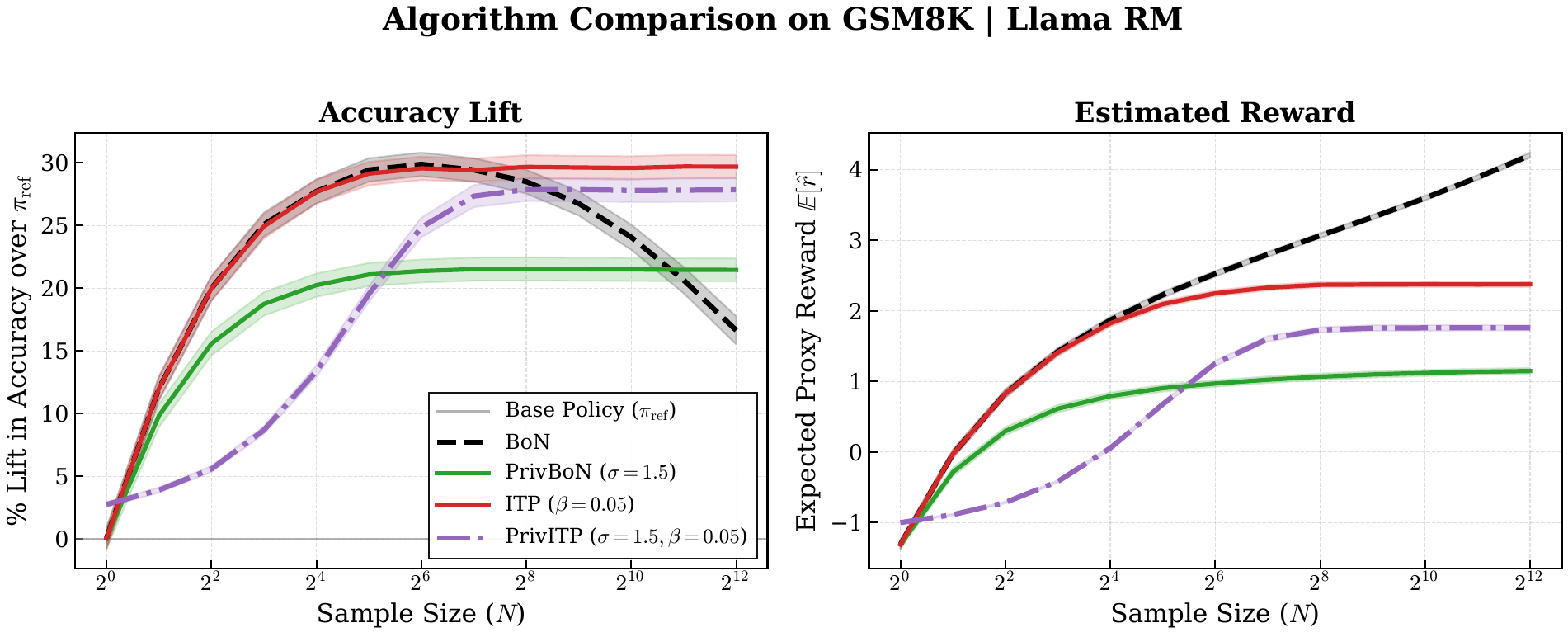}
        
        \label{fig:plot3}
    \end{subfigure}

    \begin{subfigure}[b]{\textwidth}
        \centering
        \includegraphics[width=\textwidth, height=5cm]{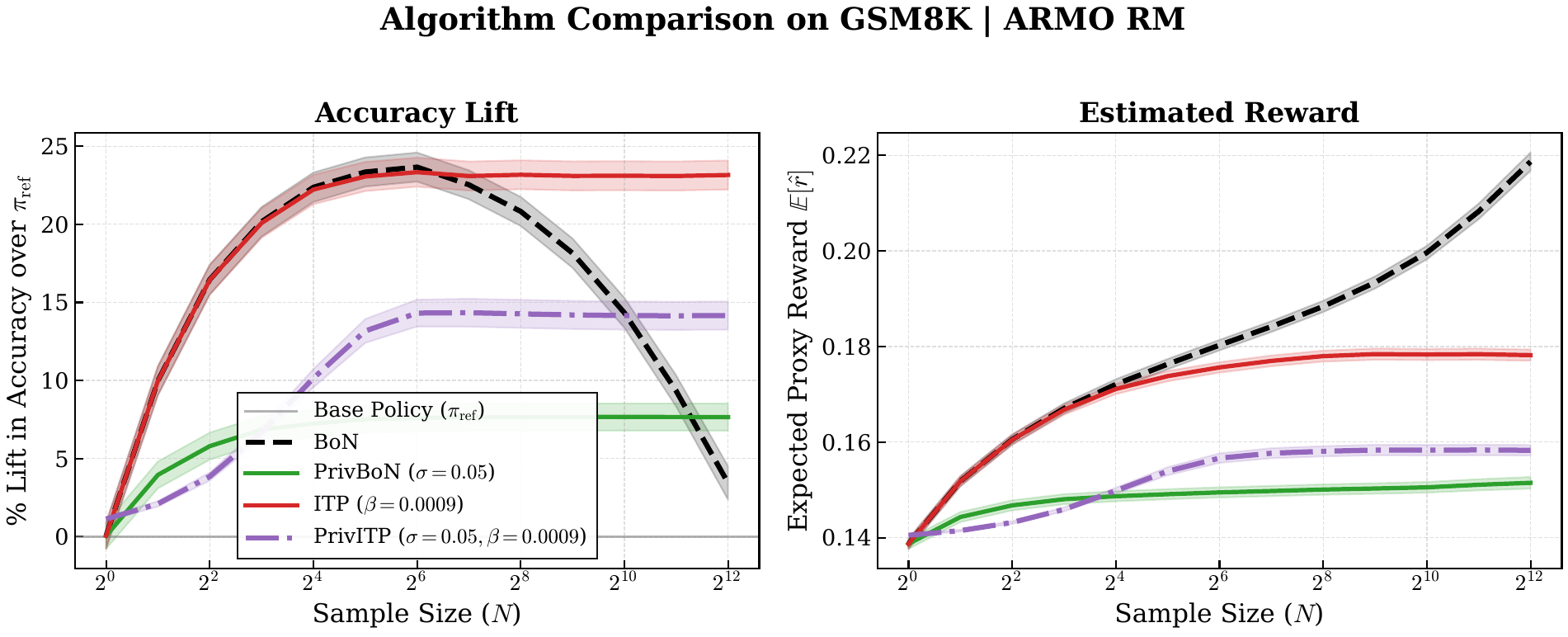}
        \label{fig:plot4}
    \end{subfigure}
    \caption{Comparison of BoN, PrivBoN, ITP, and PrivITP in accuracy
and estimated reward $\hat{r}$ for GSM8K for four reward models and Gemma-2-2B-Instruct $\pi_\mathrm{ref}$}
    \label{fig:gsm8k_comparison}
\end{figure}

\begin{figure}[htbp]
    \centering
    
    \begin{subfigure}[b]{\textwidth}
        \centering
        \includegraphics[width=\textwidth, height=5cm]{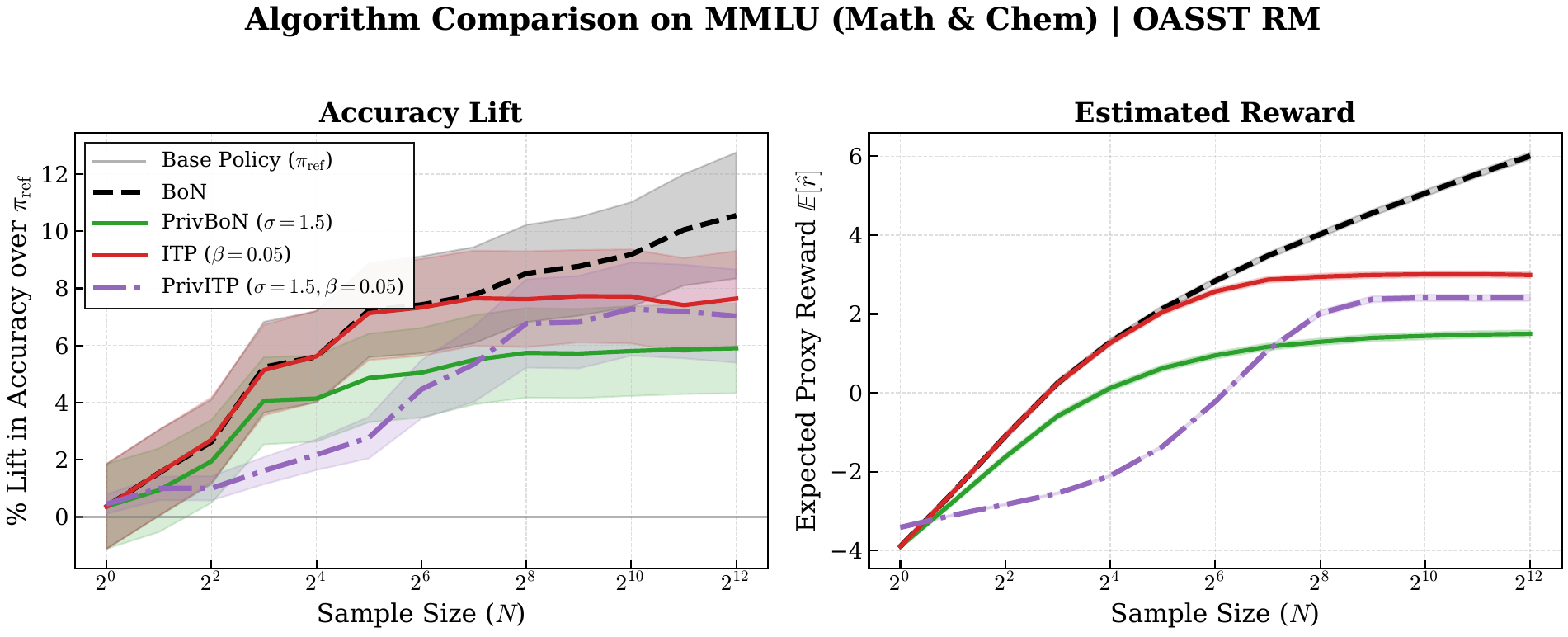}
        \label{fig:plot1}
    \end{subfigure}
    
    \begin{subfigure}[b]{\textwidth}
        \centering    
        \includegraphics[width=\textwidth,height=5cm]{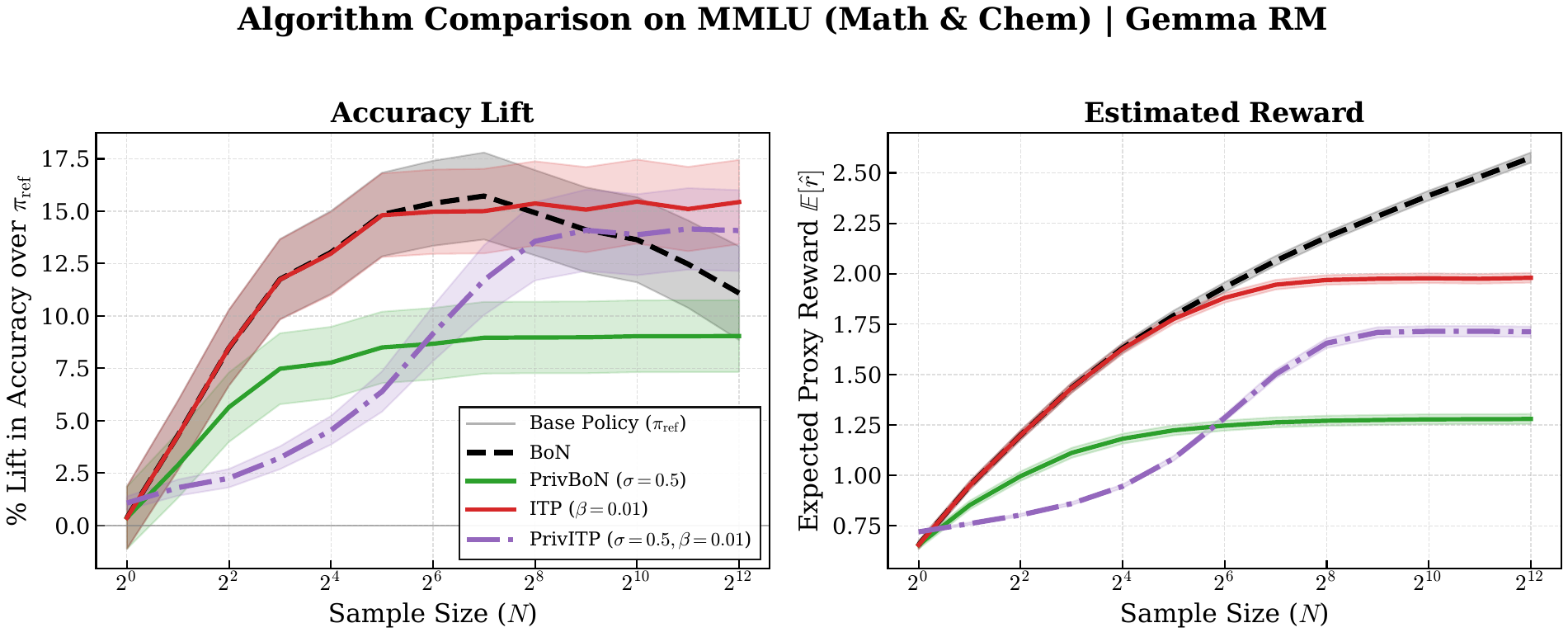}
       
        \label{fig:plot2}
    \end{subfigure}
    
    \begin{subfigure}[b]{\textwidth}
        \centering
        \includegraphics[width=\textwidth, height=5cm]{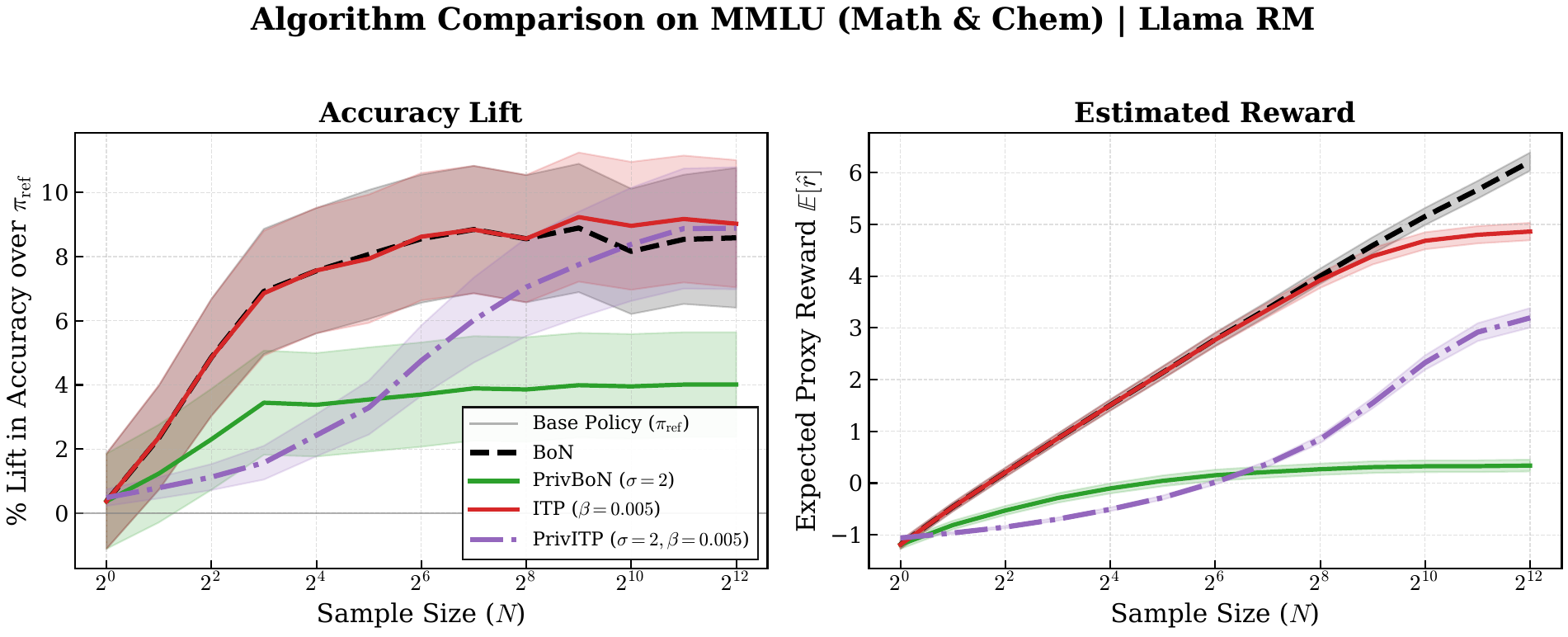}
        
        \label{fig:plot3}
    \end{subfigure}

    \begin{subfigure}[b]{\textwidth}
        \centering
        \includegraphics[width=\textwidth, height=5cm]{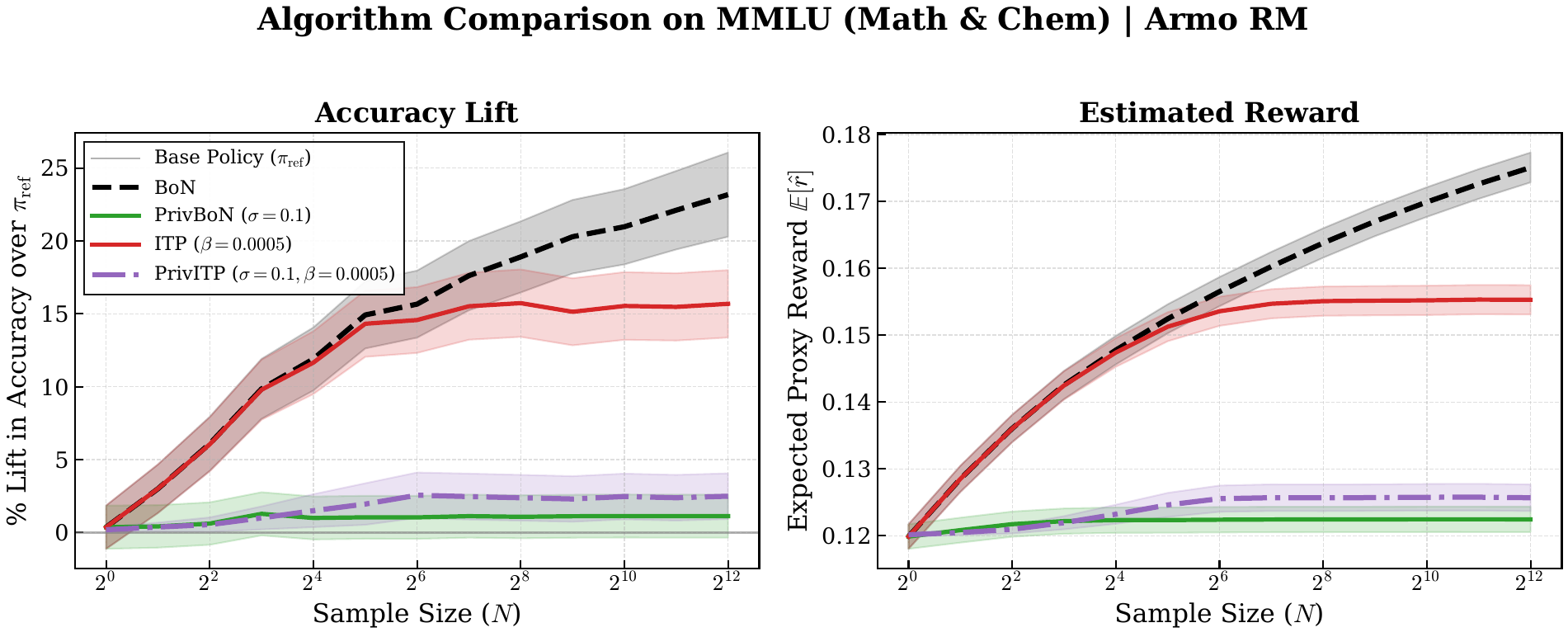}
        \label{fig:plot4}
    \end{subfigure}
    \caption{Comparison of BoN, PrivBoN, ITP, and PrivITP in accuracy
and estimated reward $\hat{r}$ for MMLU for four reward models and Gemma-2-2B-Instruct $\pi_\mathrm{ref}$}
    \label{fig:mmlu_comparison}
\end{figure}

\begin{figure}[htbp]
    \centering
    
    \begin{subfigure}[b]{\textwidth}
        \centering
        \includegraphics[width=\textwidth, height=5cm]{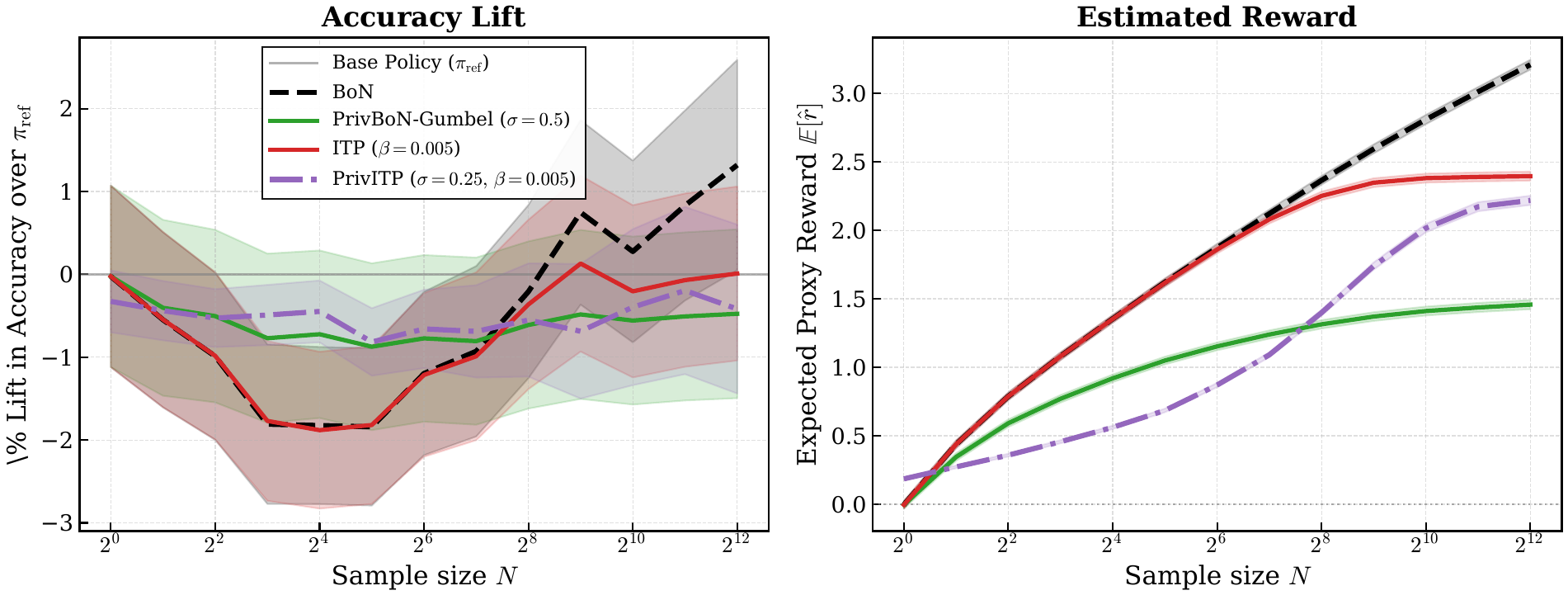}
        \caption{Oasst-RM}
        \label{fig:plot1}
    \end{subfigure}
    
    \begin{subfigure}[b]{\textwidth}
        \centering
        \includegraphics[width=\textwidth, height=5cm]{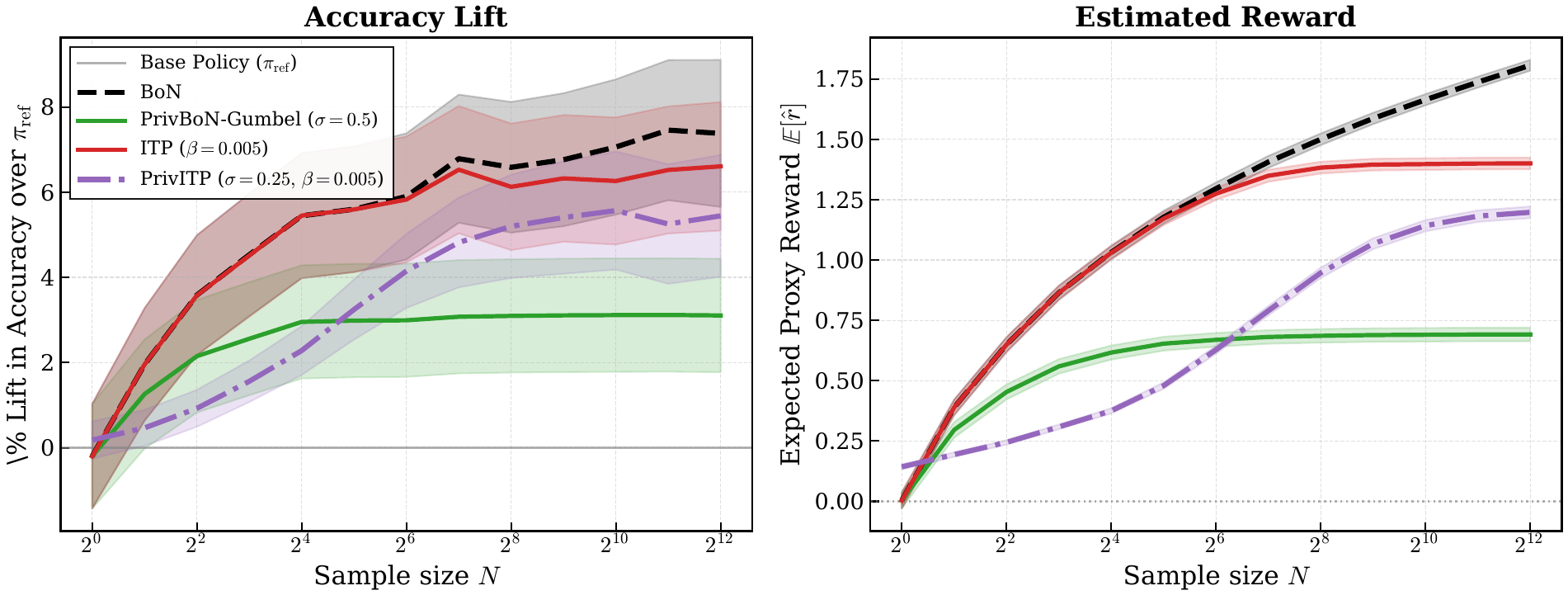}
        \caption{Gemma-RM}
        \label{fig:plot2}
    \end{subfigure}
    
    \begin{subfigure}[b]{\textwidth}
        \centering
\includegraphics[width=\textwidth,height=5cm]{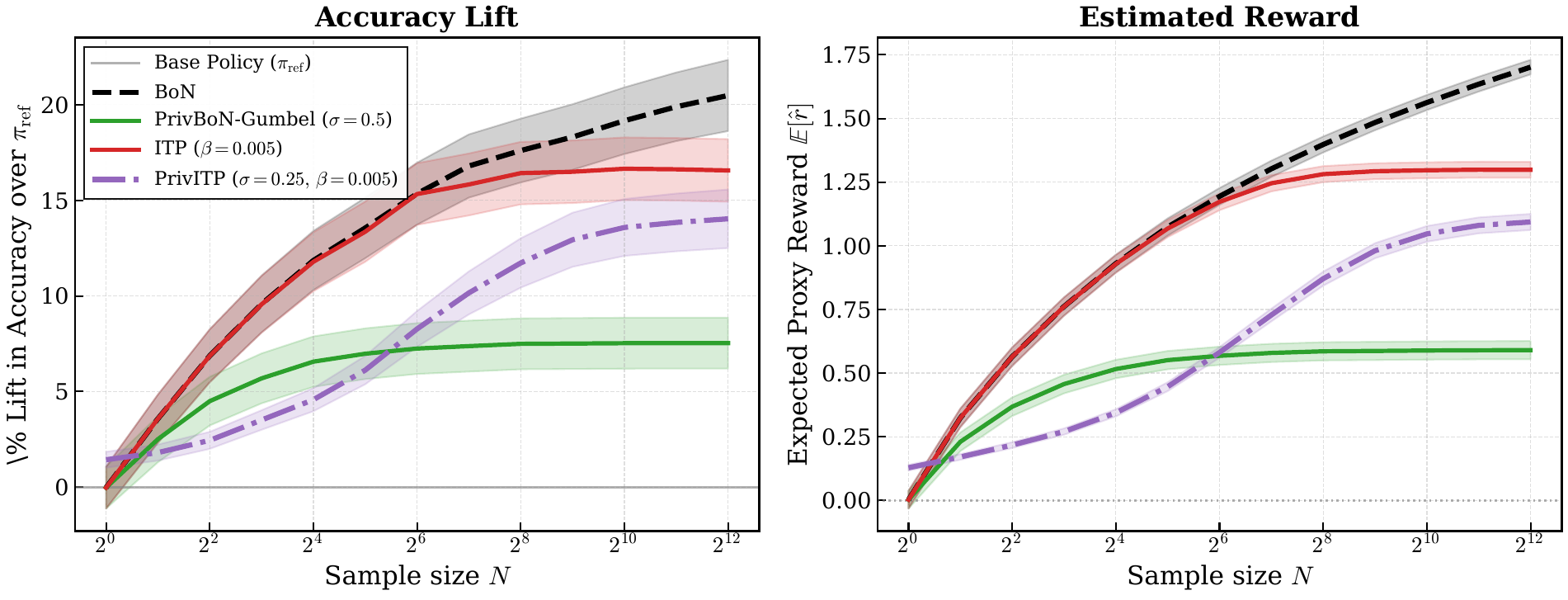}
        \caption{Llama-RM}
        \label{fig:plot3}
    \end{subfigure}
    \hfill
    \begin{subfigure}[b]{\textwidth}
        \centering
\includegraphics[width=\textwidth,height=5cm]{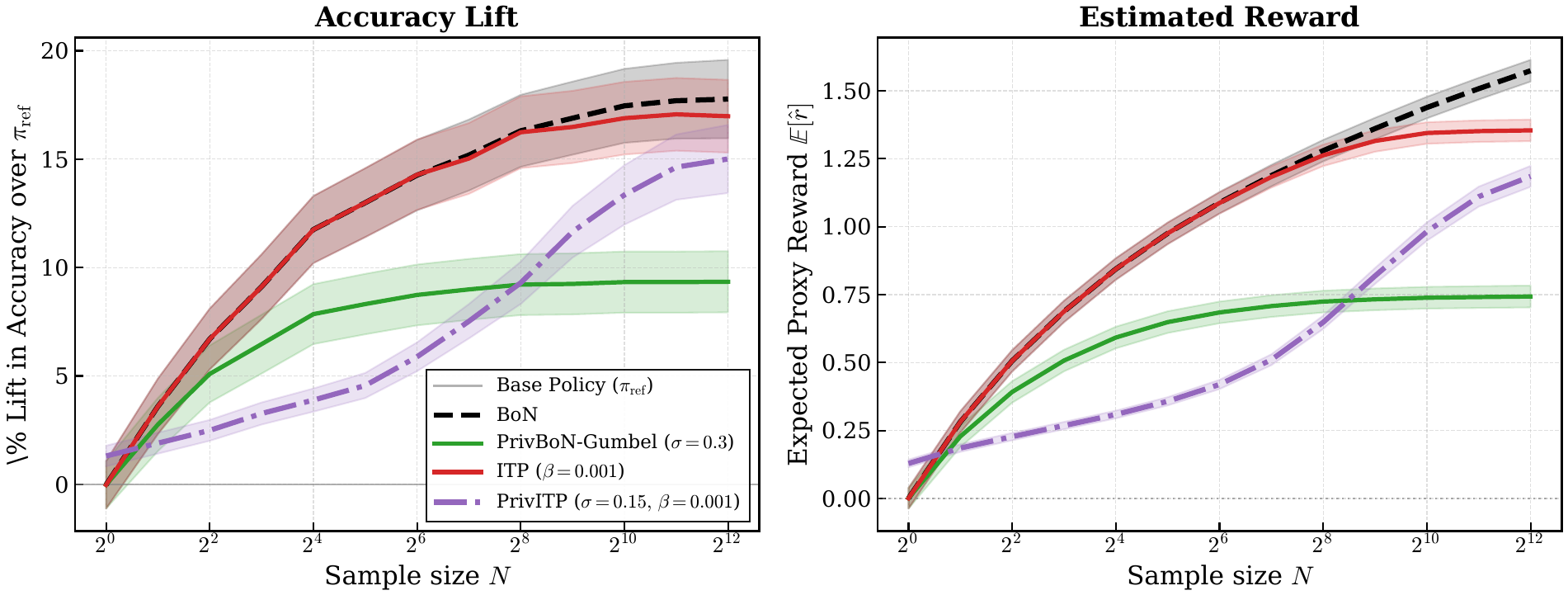}
        \caption{Armo-RM}
        \label{fig:plot4}
    \end{subfigure}
    \caption{Comparison of BoN, PrivBoN, ITP, and PrivITP in accuracy
and estimated reward $\hat{r}$ for MATH for four reward models and Gemma-2-2B-Instruct $\pi_\mathrm{ref}$}
    \label{fig:math_comparison}
\end{figure}

\begin{figure}[htbp]
    \centering
    
    \begin{subfigure}[b]{\textwidth}
        \centering
        \includegraphics[width=\textwidth, height=5cm]{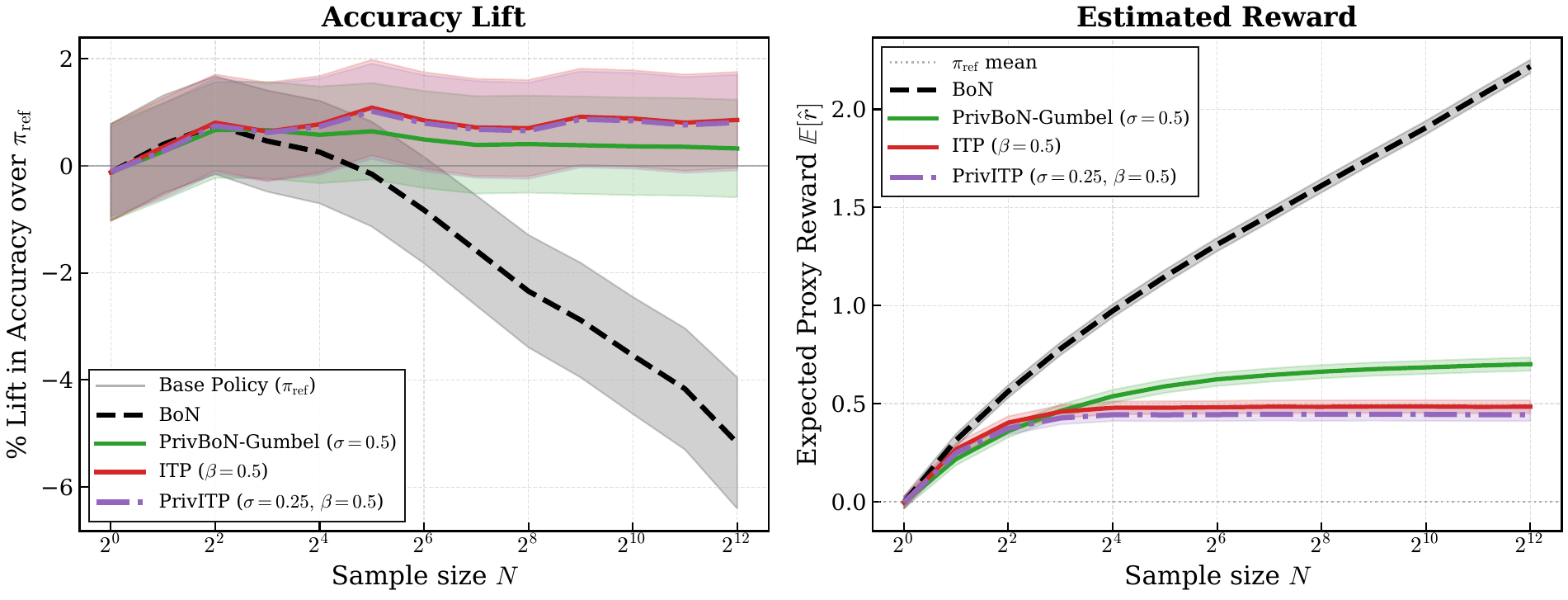}
        \caption{Oasst-RM}
        \label{fig:plot1}
    \end{subfigure}
    
    \begin{subfigure}[b]{\textwidth}
        \centering
        \includegraphics[width=\textwidth, height=5cm]{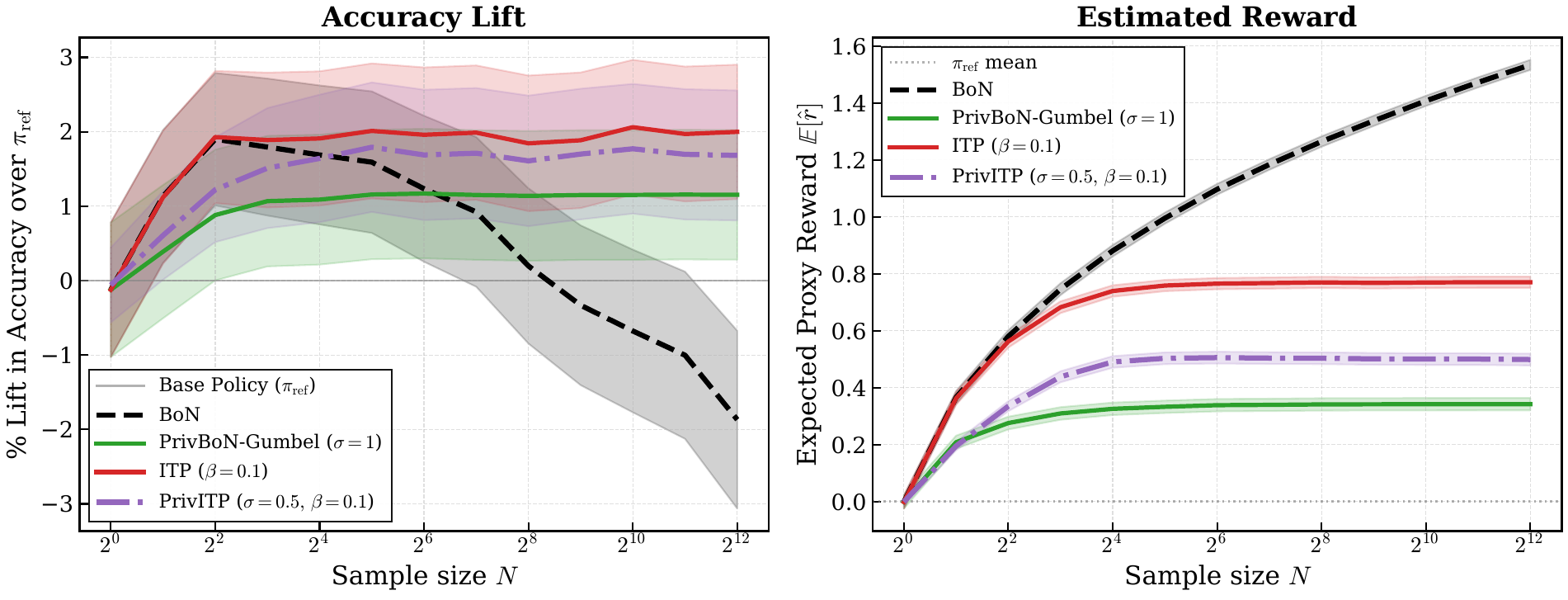}
        \caption{Gemma-RM}
        \label{fig:plot2}
    \end{subfigure}
    
    \begin{subfigure}[b]{\textwidth}
        \centering
\includegraphics[width=\textwidth,height=5cm]{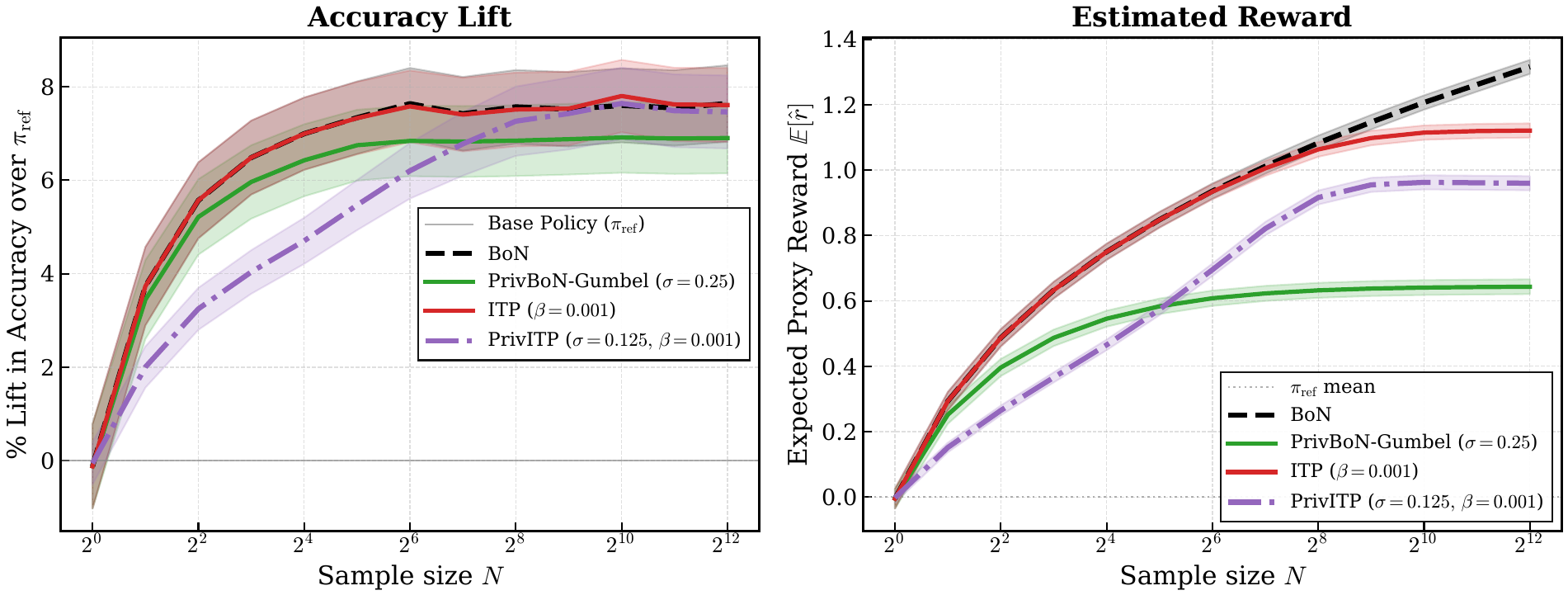}
        \caption{Llama-RM}
        \label{fig:plot3}
    \end{subfigure}
    \hfill
    \begin{subfigure}[b]{\textwidth}
        \centering
\includegraphics[width=\textwidth,height=5cm]{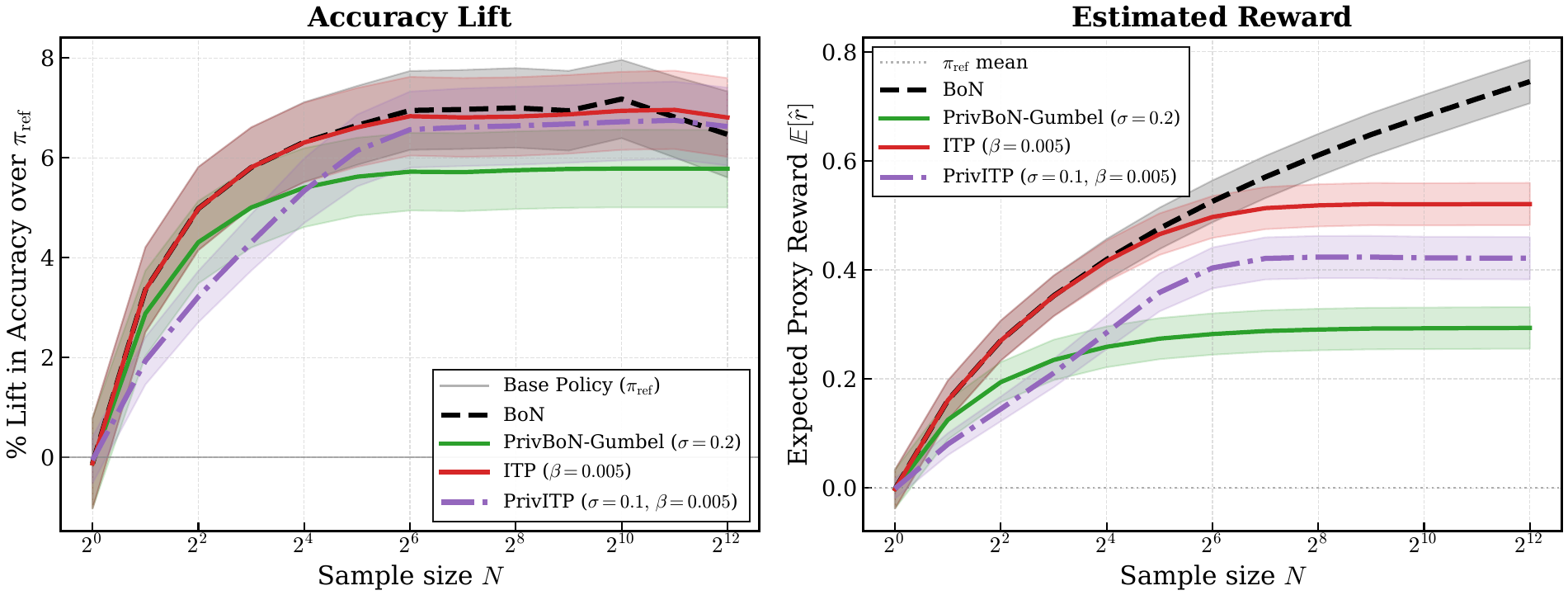}
        \caption{Armo-RM}
        \label{fig:plot4}
    \end{subfigure}
    \caption{Comparison of BoN, PrivBoN, ITP, and PrivITP in accuracy
and estimated reward $\hat{r}$ for GSM8K for four reward models and Phi-3-Mini-Instruct $\pi_\mathrm{ref}$}
    \label{fig:gsm8k_comparison_phi3}
\end{figure}

\begin{figure}[htbp]
    \centering
    
    \begin{subfigure}[b]{\textwidth}
        \centering
        \includegraphics[width=\textwidth, height=5cm]{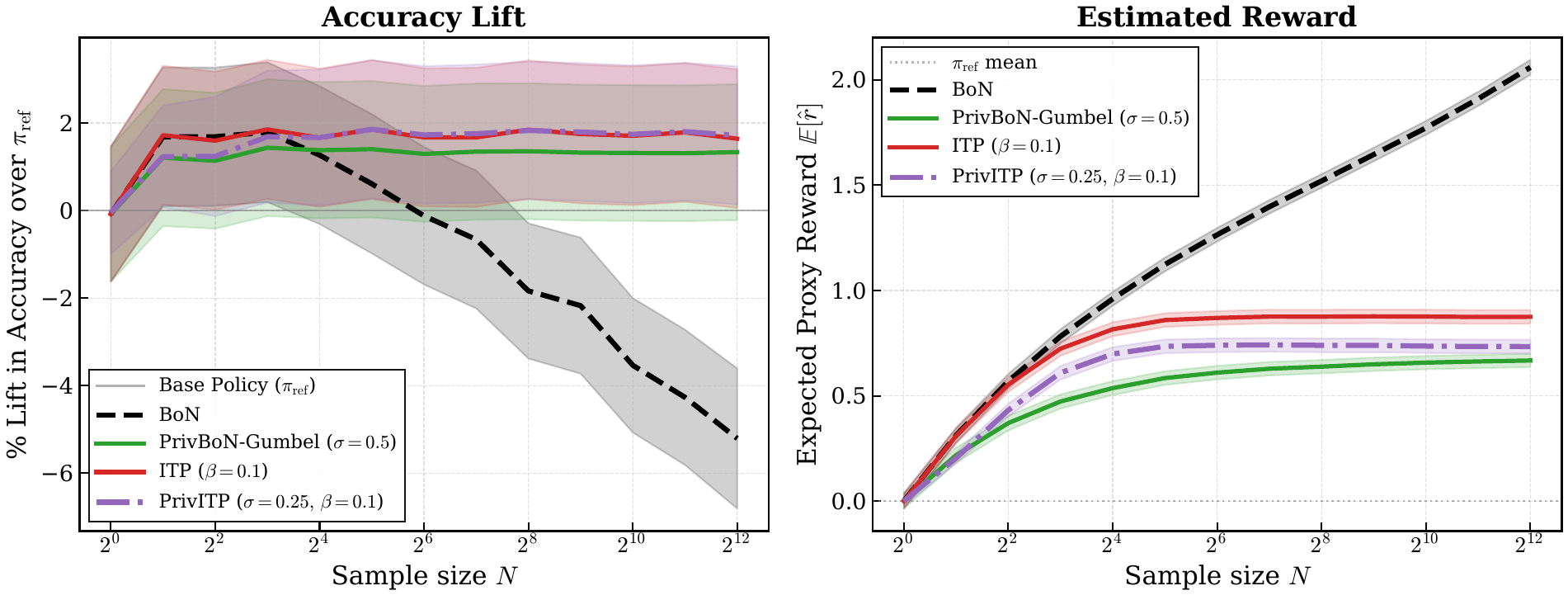}
        \caption{Oasst-RM}
        \label{fig:plot1}
    \end{subfigure}
    
    \begin{subfigure}[b]{\textwidth}
        \centering
        \includegraphics[width=\textwidth, height=5cm]{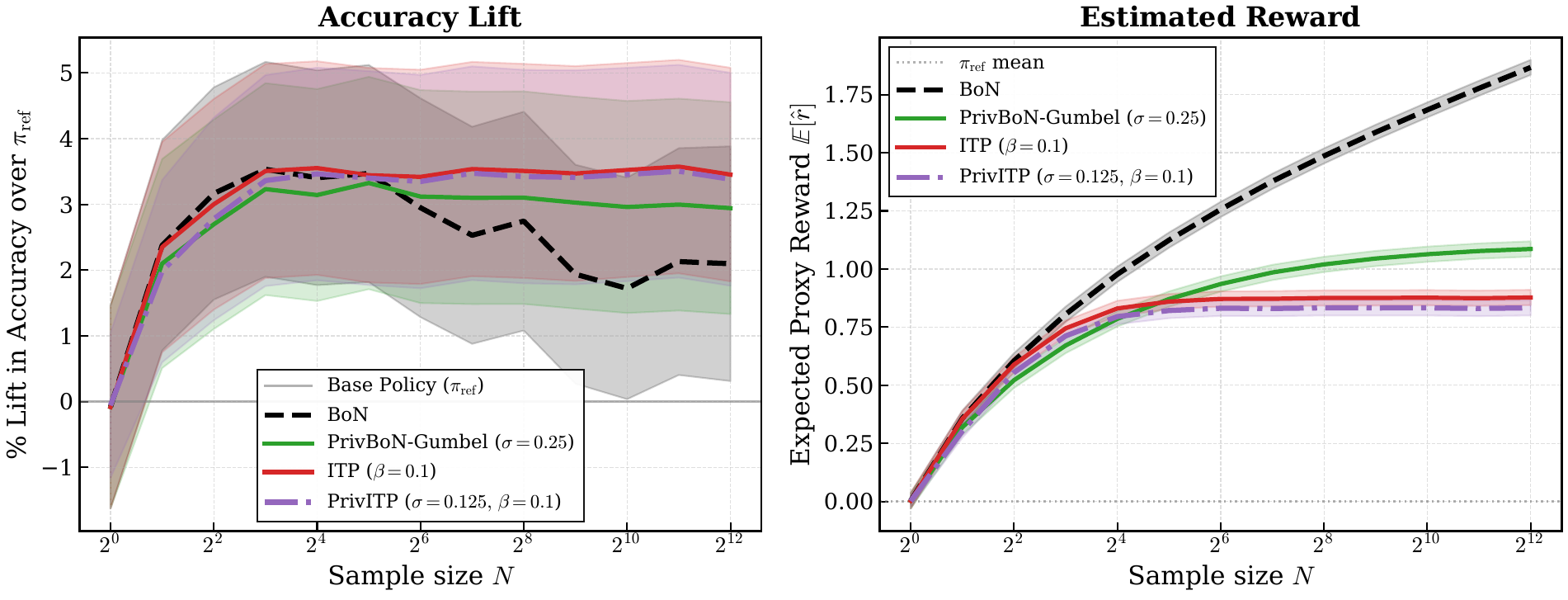}
        \caption{Gemma-RM}
        \label{fig:plot2}
    \end{subfigure}
    
    \begin{subfigure}[b]{\textwidth}
        \centering
\includegraphics[width=\textwidth,height=5cm]{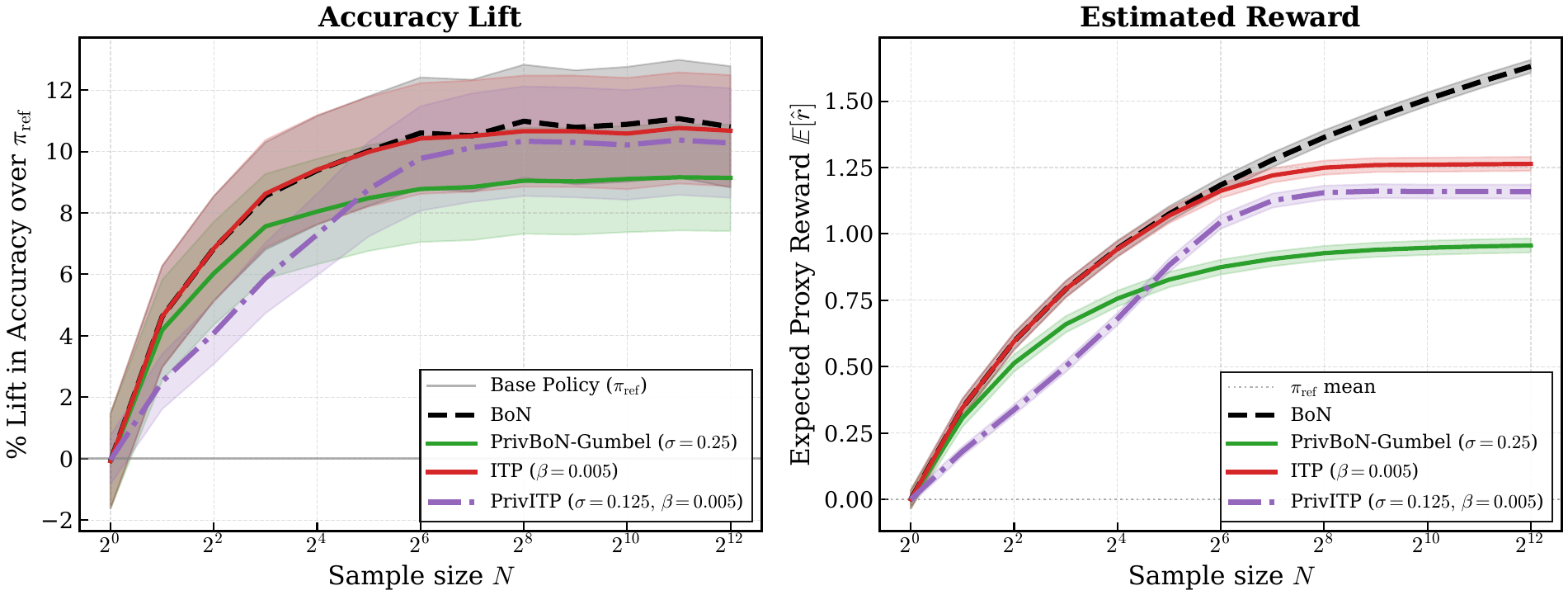}
        \caption{Llama-RM}
        \label{fig:plot3}
    \end{subfigure}
    \hfill
    \begin{subfigure}[b]{\textwidth}
        \centering
\includegraphics[width=\textwidth,height=5cm]{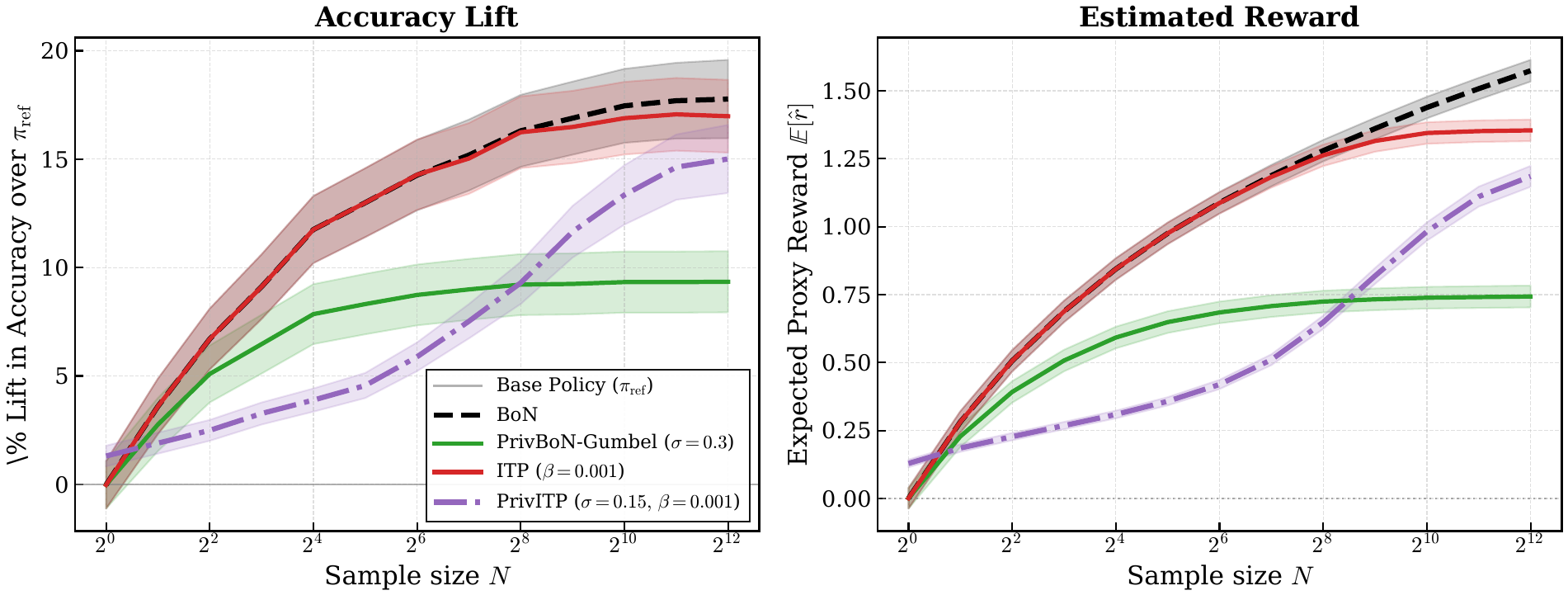}
        \caption{Armo-RM}
        \label{fig:plot4}
    \end{subfigure}
    \caption{Comparison of BoN, PrivBoN, ITP, and PrivITP in accuracy
and estimated reward $\hat{r}$ for MATH for four reward models and Phi-3-Mini-Instruct $\pi_\mathrm{ref}$}
    \label{fig:math_comparison_phi3}
\end{figure}

\begin{figure}[htbp]
    \centering
    
    \begin{subfigure}[b]{\textwidth}
        \centering
        \includegraphics[width=\textwidth, height=5cm]{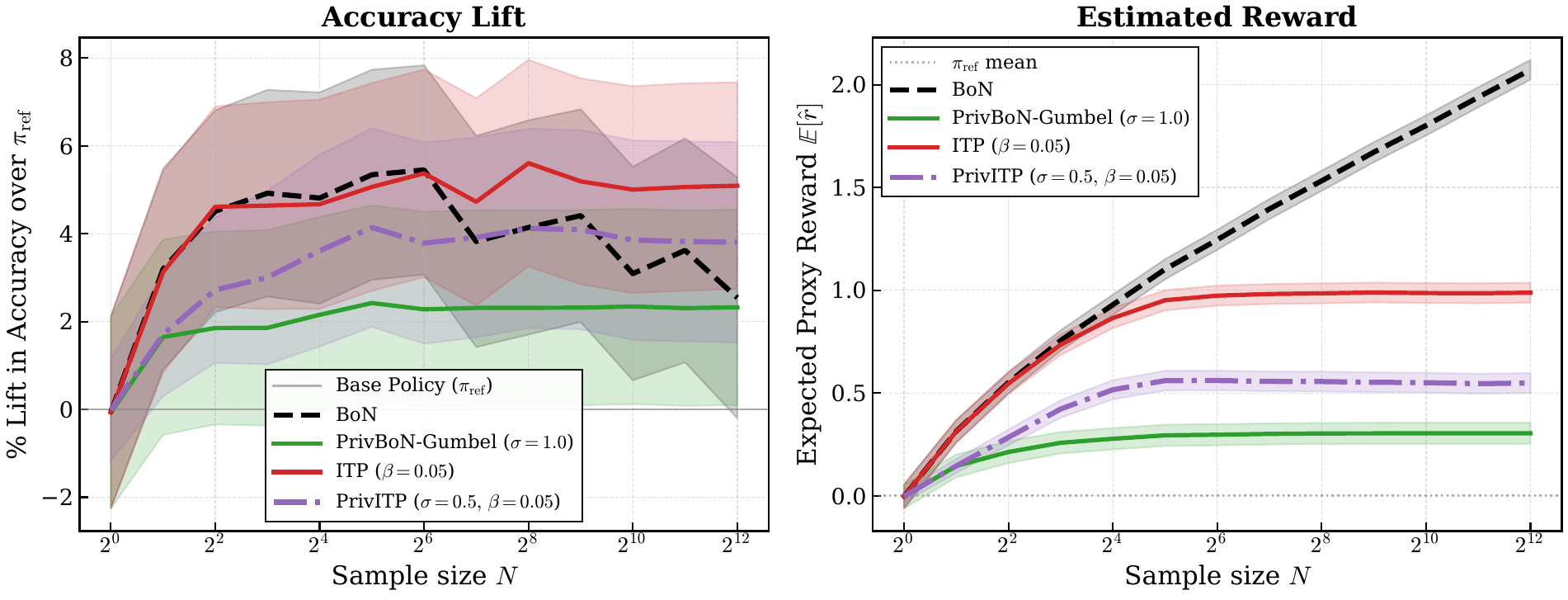}
        \caption{Oasst-RM}
        \label{fig:plot1}
    \end{subfigure}
    
    \begin{subfigure}[b]{\textwidth}
        \centering
        \includegraphics[width=\textwidth, height=5cm]{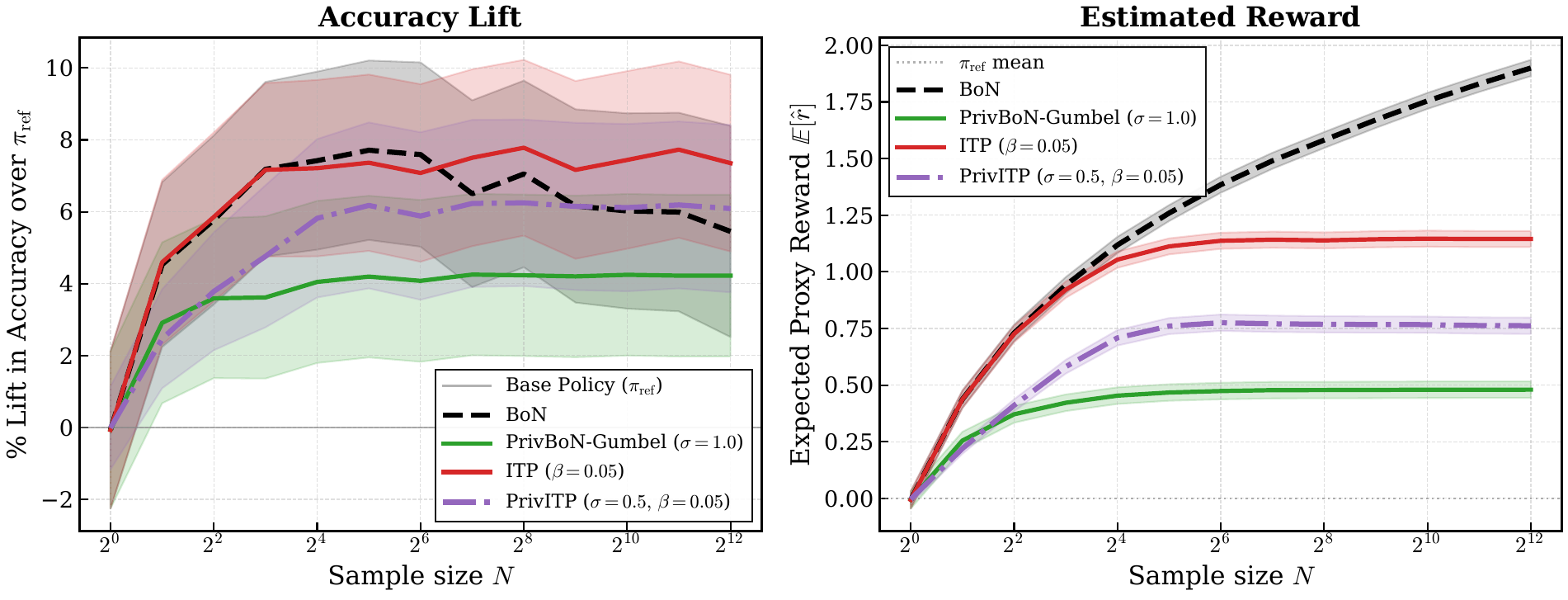}
        \caption{Gemma-RM}
        \label{fig:plot2}
    \end{subfigure}
    
    \begin{subfigure}[b]{\textwidth}
        \centering
        \includegraphics[width=\textwidth,height=5cm]{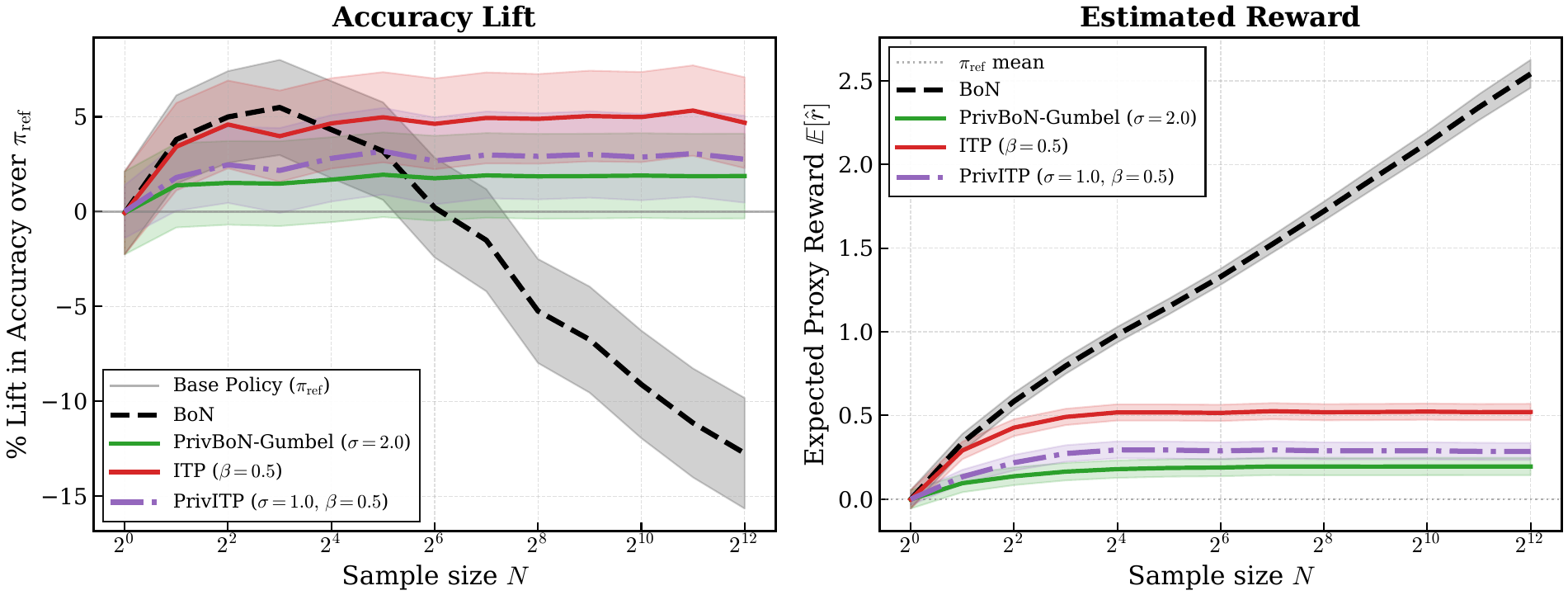}
        \caption{Llama-RM}
        \label{fig:plot3}
    \end{subfigure}
    
    \begin{subfigure}[b]{\textwidth}
        \centering
        \includegraphics[width=\textwidth,height=5cm]{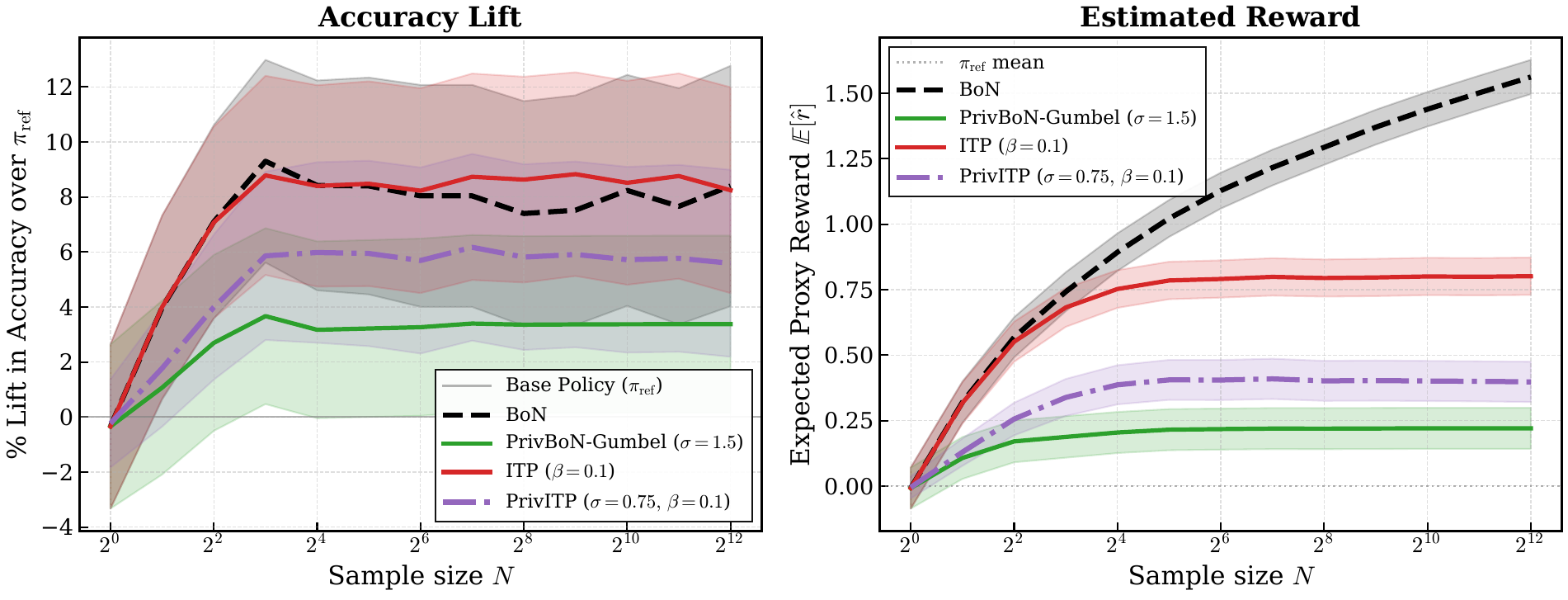}
        \caption{Armo-RM}
        \label{fig:plot4}
    \end{subfigure}
    
    \caption{Comparison of BoN, PrivBoN, ITP, and PrivITP in accuracy
and estimated reward $\hat{r}$ for MMLU for four reward models and Phi-3-Mini-Instruct $\pi_\mathrm{ref}$}
    \label{fig:mmlu_comparison_phi3}
\end{figure}

\begin{figure}
    \centering
    \begin{subfigure}[t]{0.49\textwidth}
        \centering
        \includegraphics[width=\textwidth]{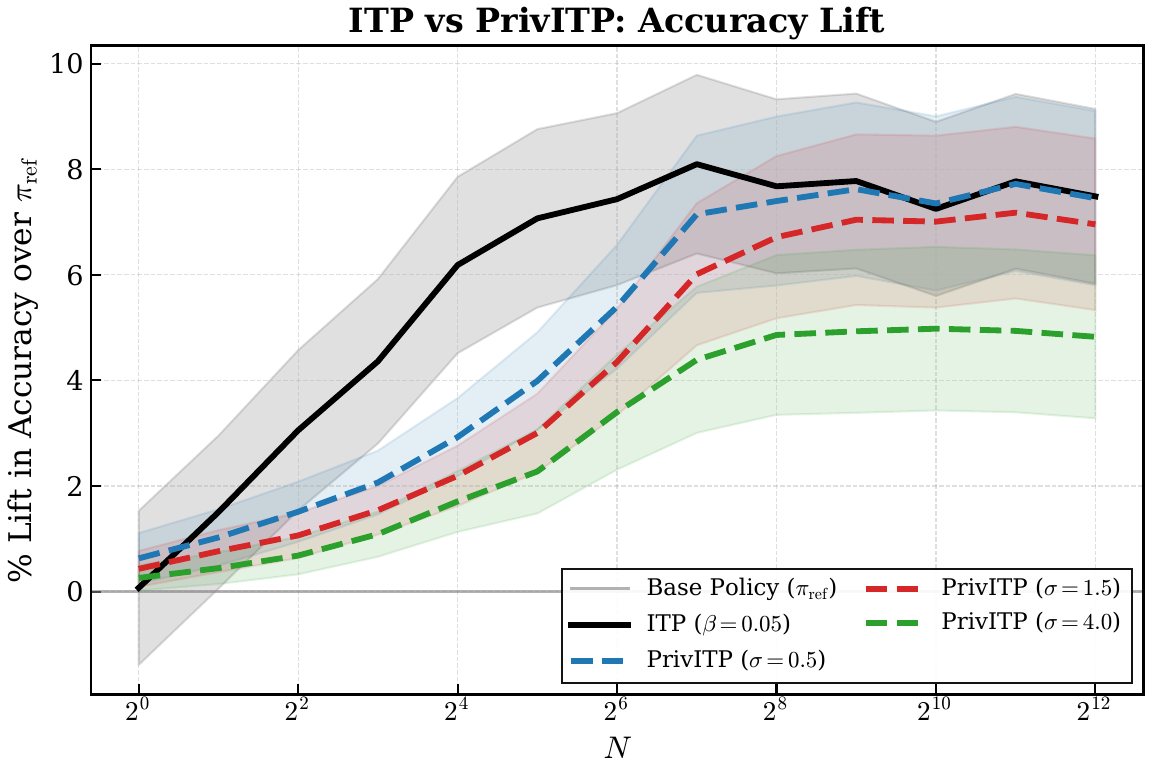}
        \caption{Oasst-RM}
        \label{fig:sub1}
    \end{subfigure}
    \hfill
    \begin{subfigure}[t]{0.49\textwidth}
        \centering
        \includegraphics[width=\textwidth]{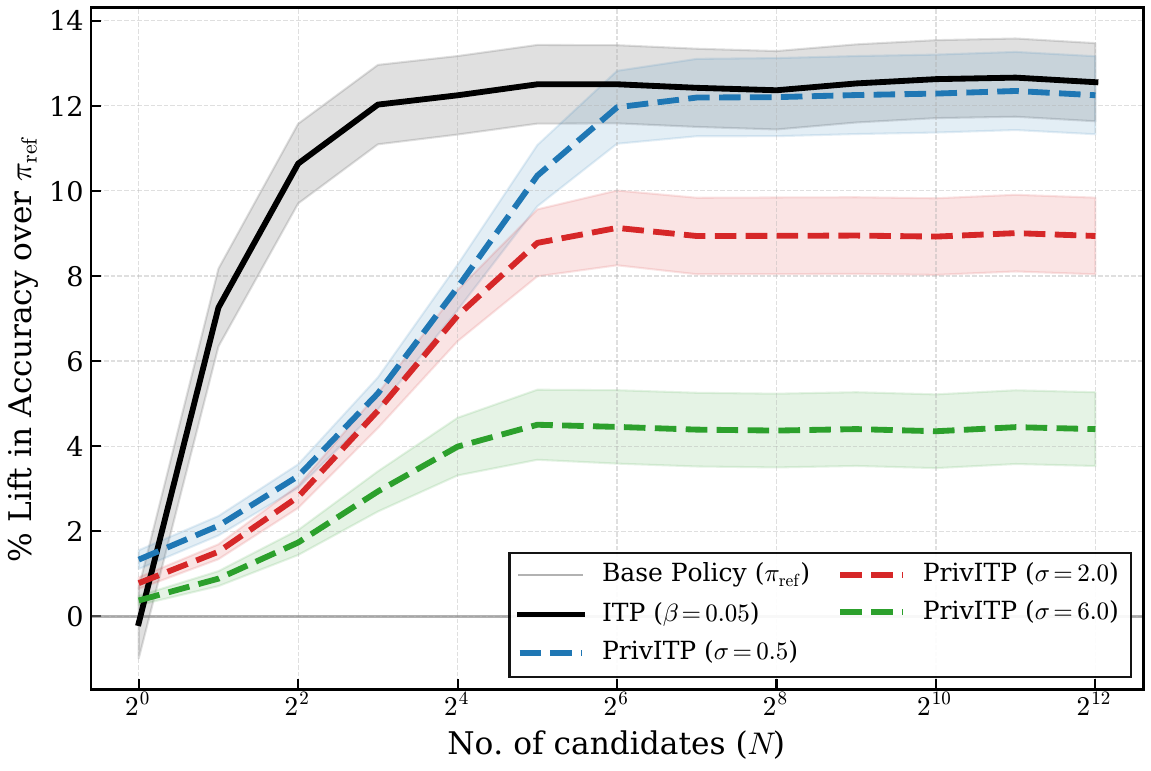}
        \caption{Gemma-RM}
        \label{fig:sub2}
    \end{subfigure}
    
    \begin{subfigure}[t]{0.49\textwidth}
        \centering
        \includegraphics[width=\textwidth]{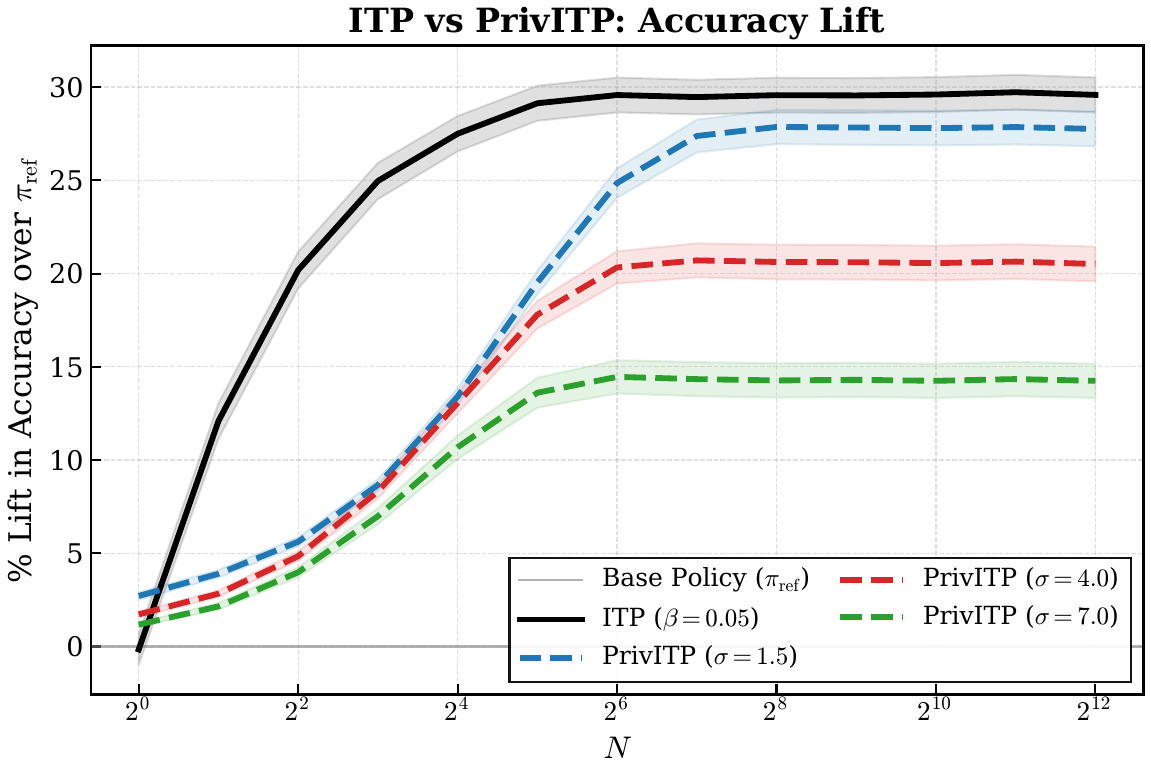}
        \caption{Llama-RM}
        \label{fig:sub3}
    \end{subfigure}
    \hfill 
    \begin{subfigure}[t]{0.49\textwidth}
        \centering
        \includegraphics[width=\textwidth]{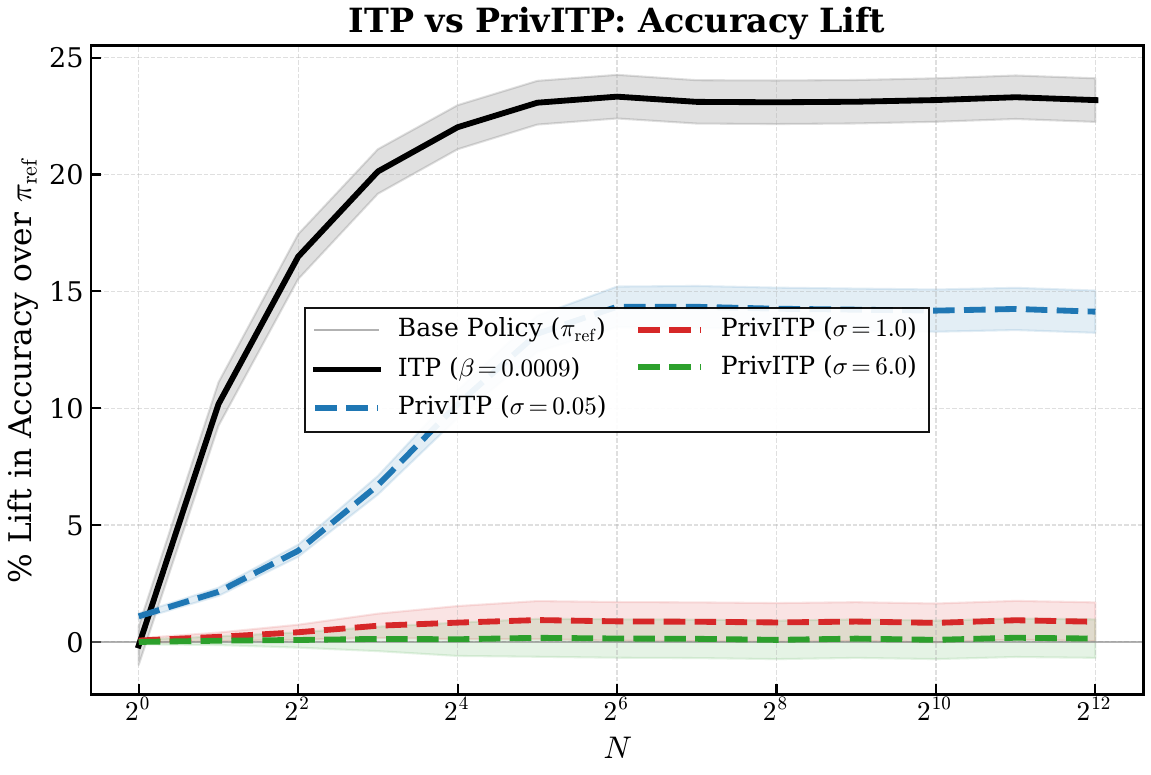}
        \caption{Armo-RM}
        \label{fig:sub4}
    \end{subfigure}
    \caption{ITP vs PrivITP for different $\sigma$ on GSM8K dataset with Gemma-2-2b-Instruct base policy}
    \label{fig:privacy-vs-utility}
\end{figure}

\begin{figure}
    \centering
    \begin{subfigure}[t]{0.49\textwidth}
        \centering
        \includegraphics[width=\textwidth]{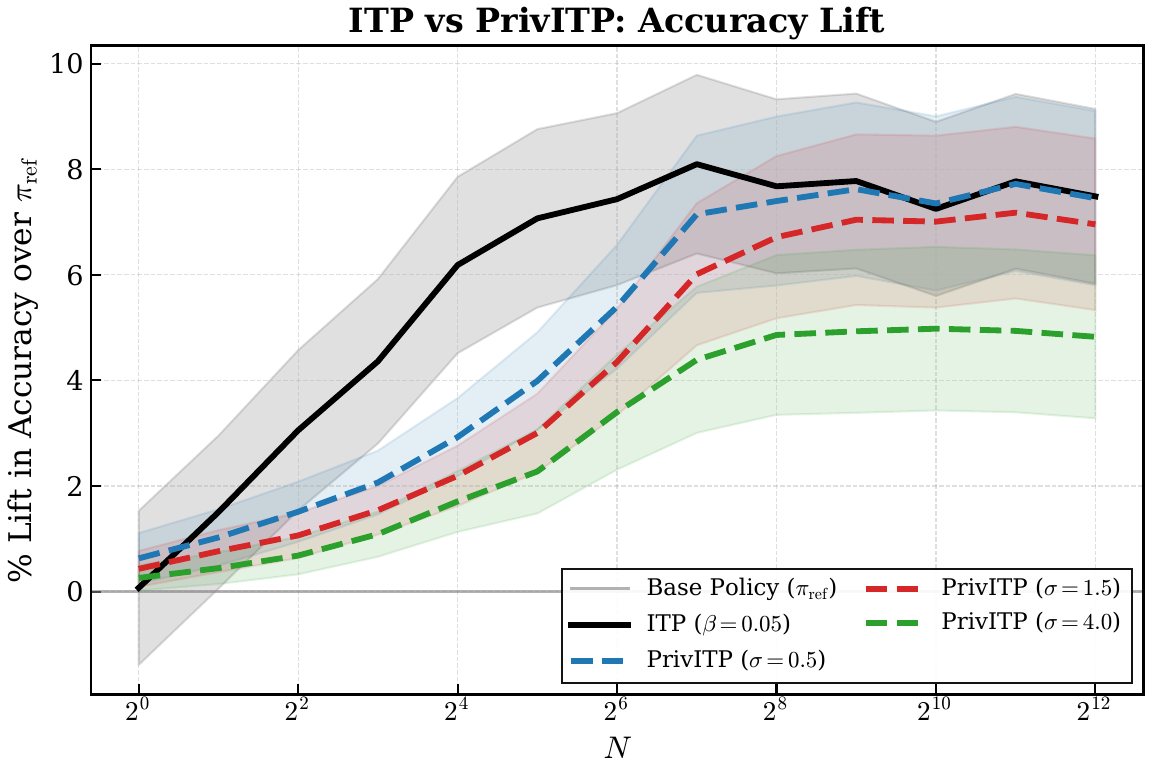}
        \caption{Oasst-RM}
        \label{fig:sub1}
    \end{subfigure}
    \hfill
    \begin{subfigure}[t]{0.49\textwidth}
        \centering
        \includegraphics[width=\textwidth]{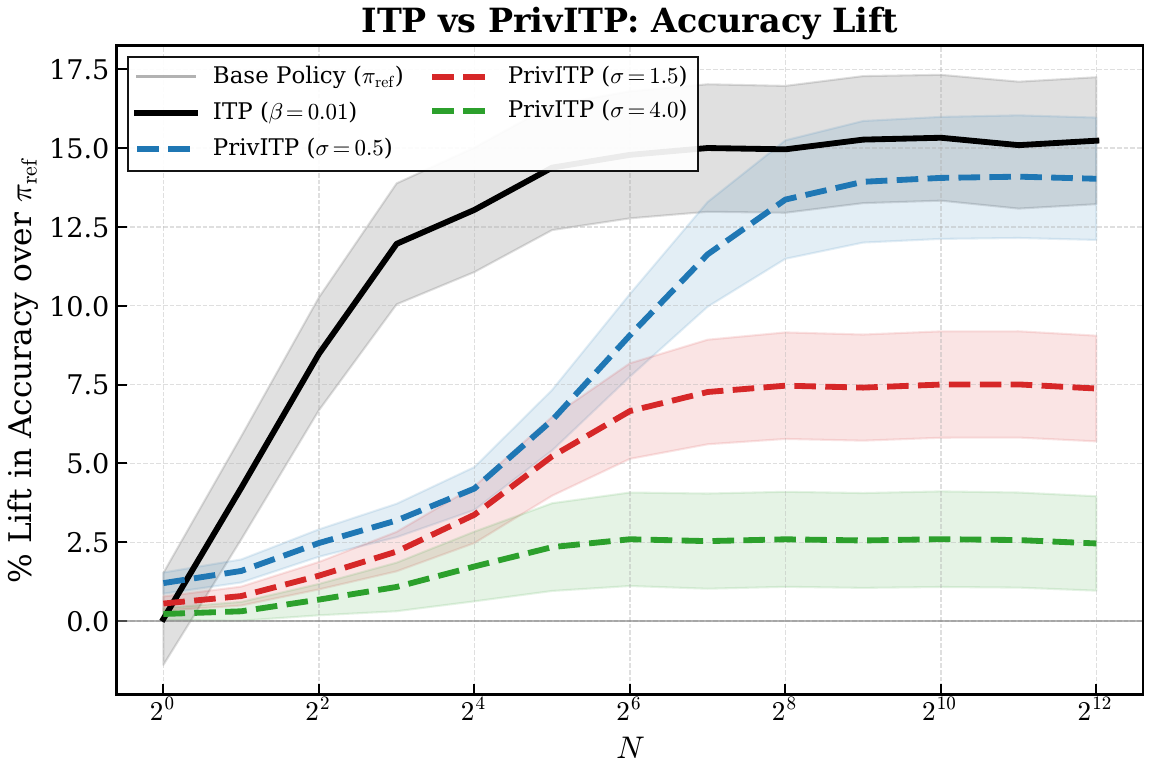}
        \caption{Gemma-RM}
        \label{fig:sub2}
    \end{subfigure}
    
    \begin{subfigure}[t]{0.49\textwidth}
        \centering
        \includegraphics[width=\textwidth]{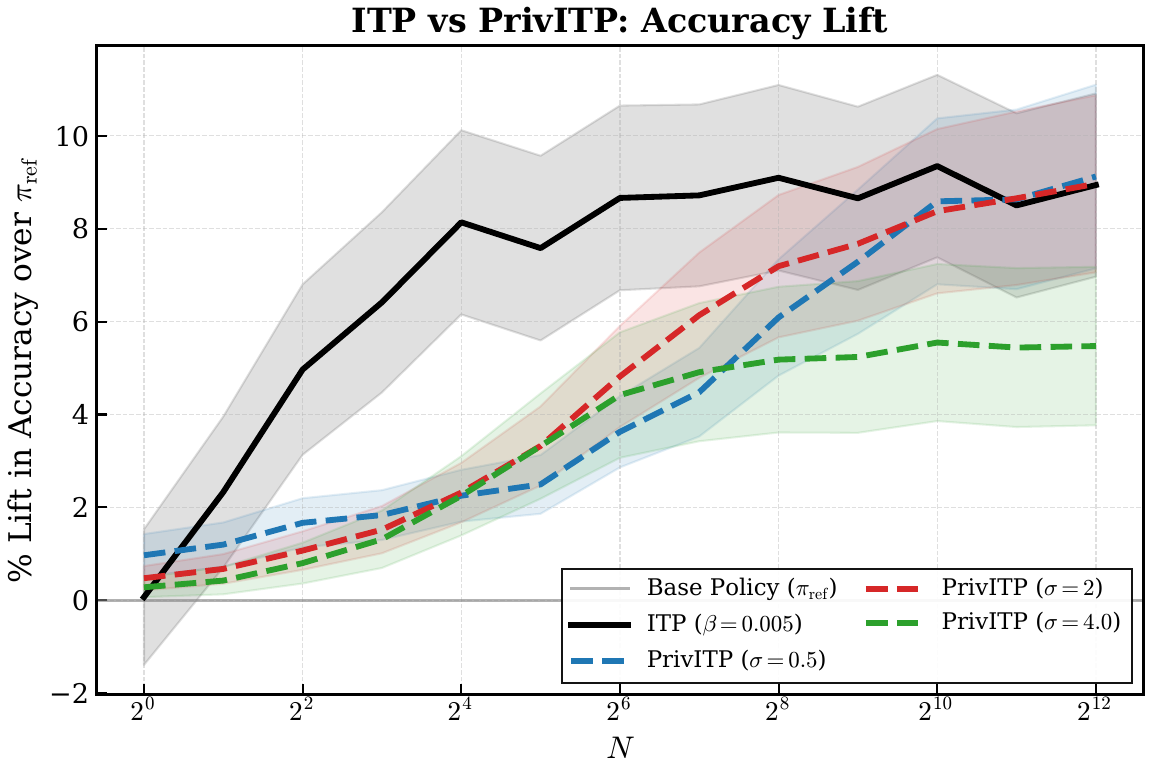}
        \caption{Llama-RM}
        \label{fig:sub3}
    \end{subfigure}
    \hfill 
    \begin{subfigure}[t]{0.49\textwidth}
        \centering
        \includegraphics[width=\textwidth]{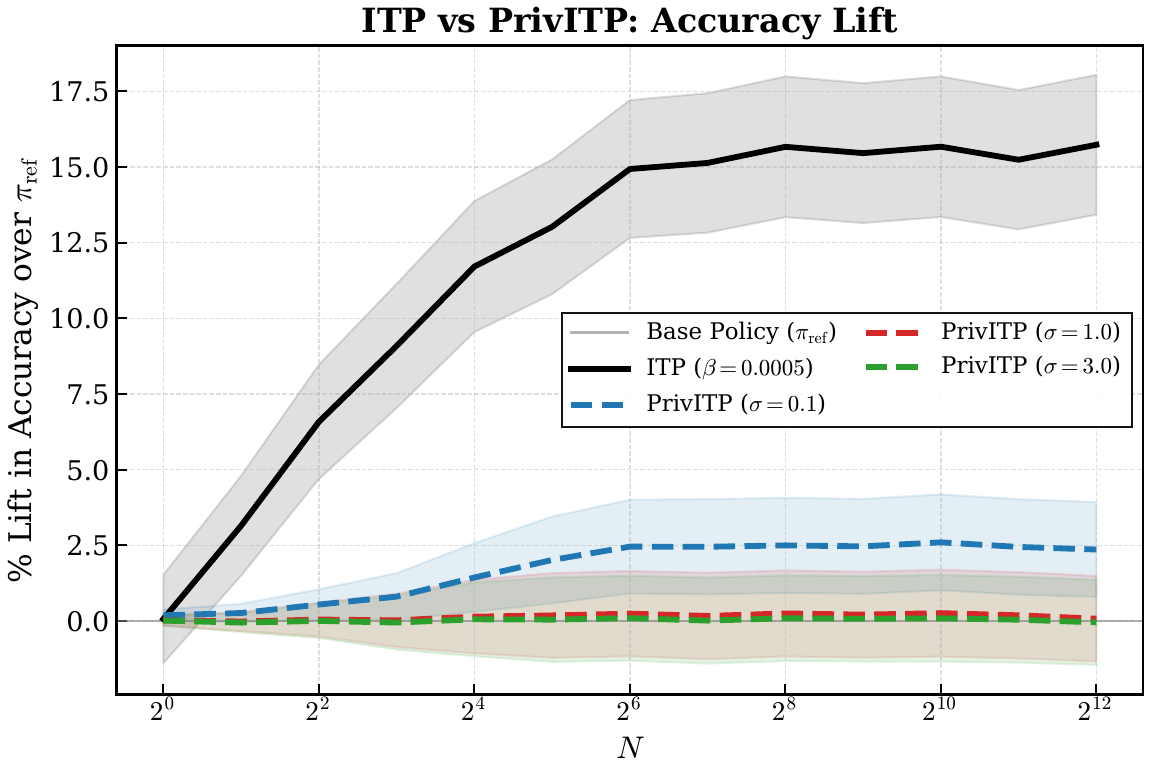}
        \caption{Armo-RM}
        \label{fig:sub4}
    \end{subfigure}
    \caption{ITP vs PrivITP for different $\sigma$ on MMLU dataset with Gemma-2-2b-Instruct base policy}
    \label{fig:privacy-vs-utility_MMLU}
\end{figure}

\begin{figure}
    \centering
    \begin{subfigure}[t]{0.49\textwidth}
        \centering
        \includegraphics[width=\textwidth]{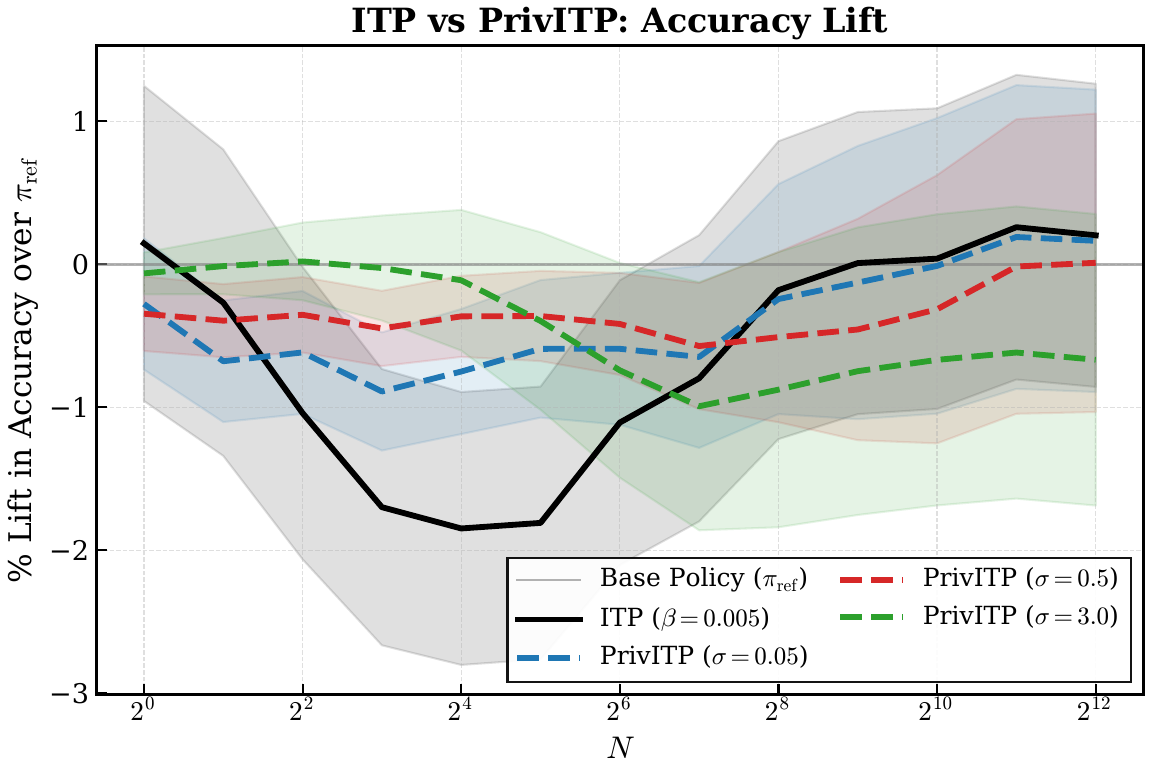}
        \caption{Oasst-RM}
        \label{fig:sub1}
    \end{subfigure}
    \hfill
    \begin{subfigure}[t]{0.49\textwidth}
        \centering
        \includegraphics[width=\textwidth]{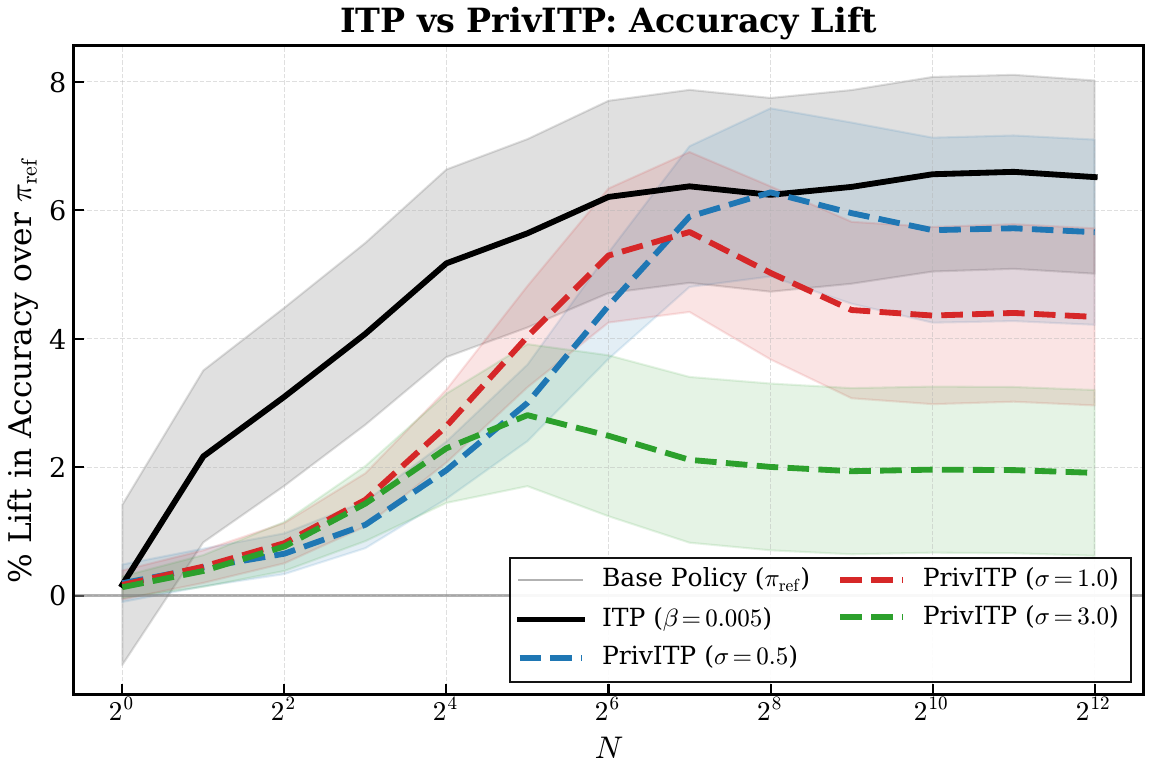}
        \caption{Gemma-RM}
        \label{fig:sub2}
    \end{subfigure}
    
    \begin{subfigure}[t]{0.49\textwidth}
        \centering
        \includegraphics[width=\textwidth]{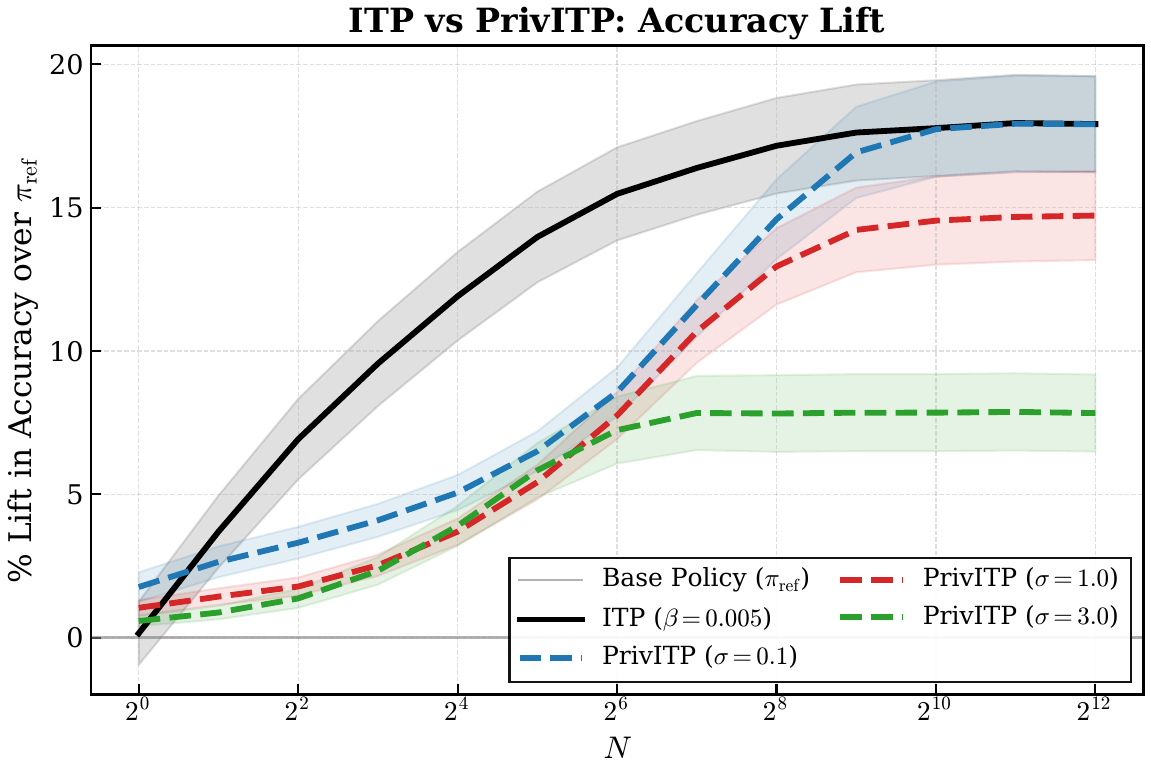}
        \caption{Llama-RM}
        \label{fig:sub3}
    \end{subfigure}
    \hfill 
    \begin{subfigure}[t]{0.49\textwidth}
        \centering
        \includegraphics[width=\textwidth]{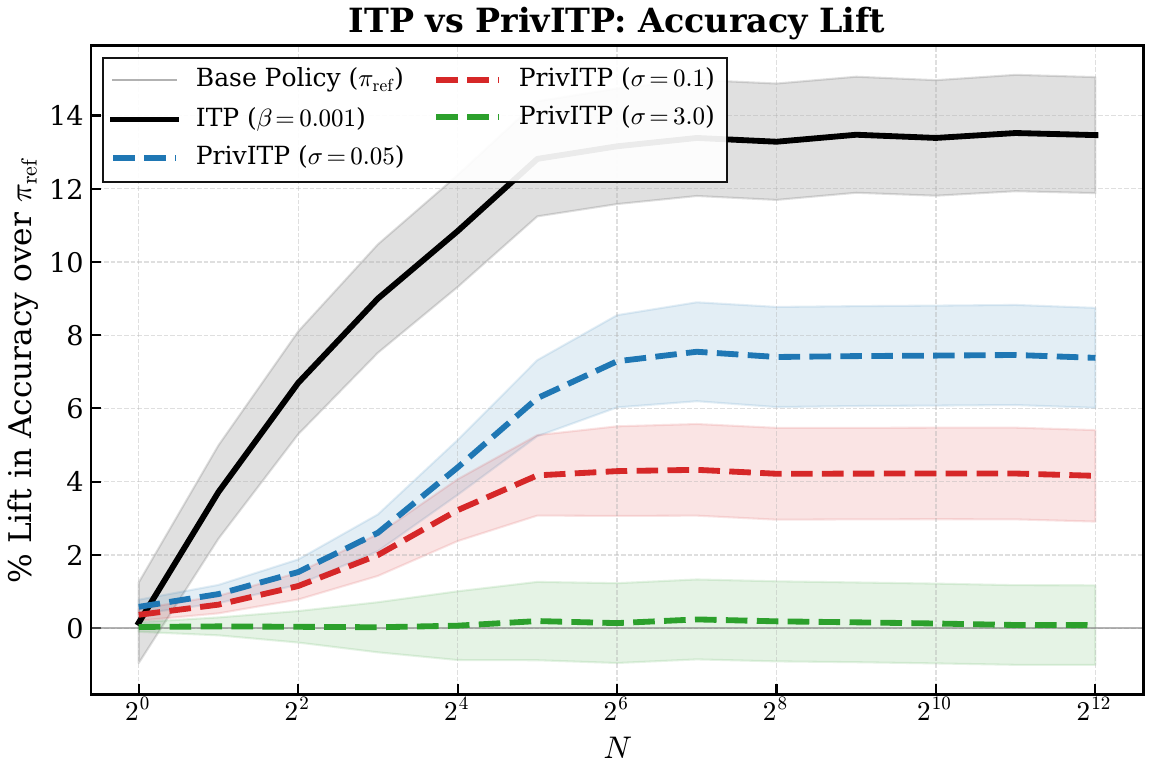}
        \caption{Armo-RM}
        \label{fig:sub4}
    \end{subfigure}
    \caption{ITP vs PrivITP for different $\sigma$ on MATH dataset with Gemma-2-2b-Instruct base policy}
    \label{fig:privacy-vs-utility_MATH}
\end{figure}

\begin{figure}
    \centering
    \begin{subfigure}[t]{0.49\textwidth}
        \centering
        \includegraphics[width=\textwidth]{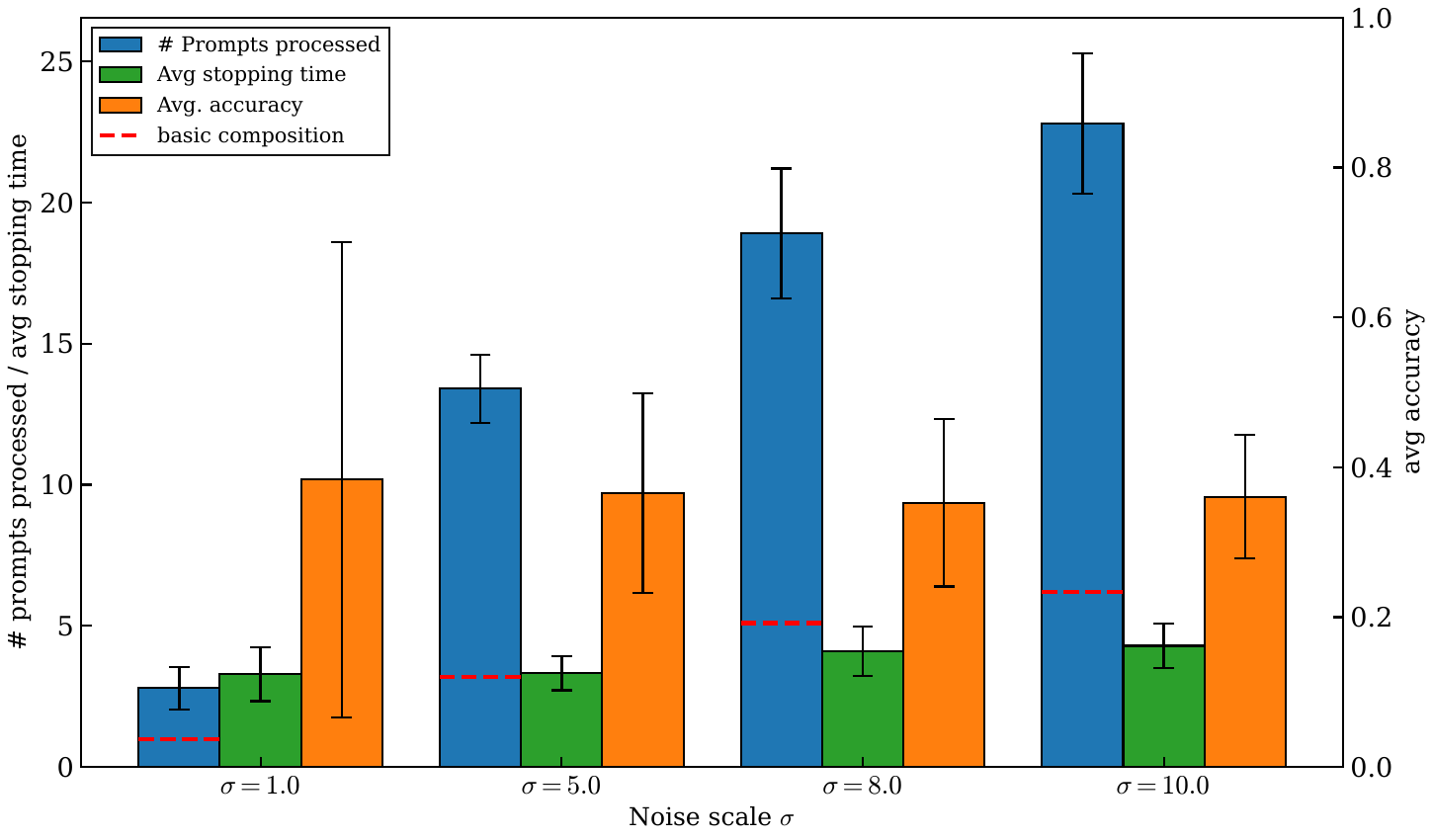}
        \caption{Oasst-RM}
        \label{fig:sub1}
    \end{subfigure}
    \hfill
    \begin{subfigure}[t]{0.49\textwidth}
        \centering
        \includegraphics[width=\textwidth]{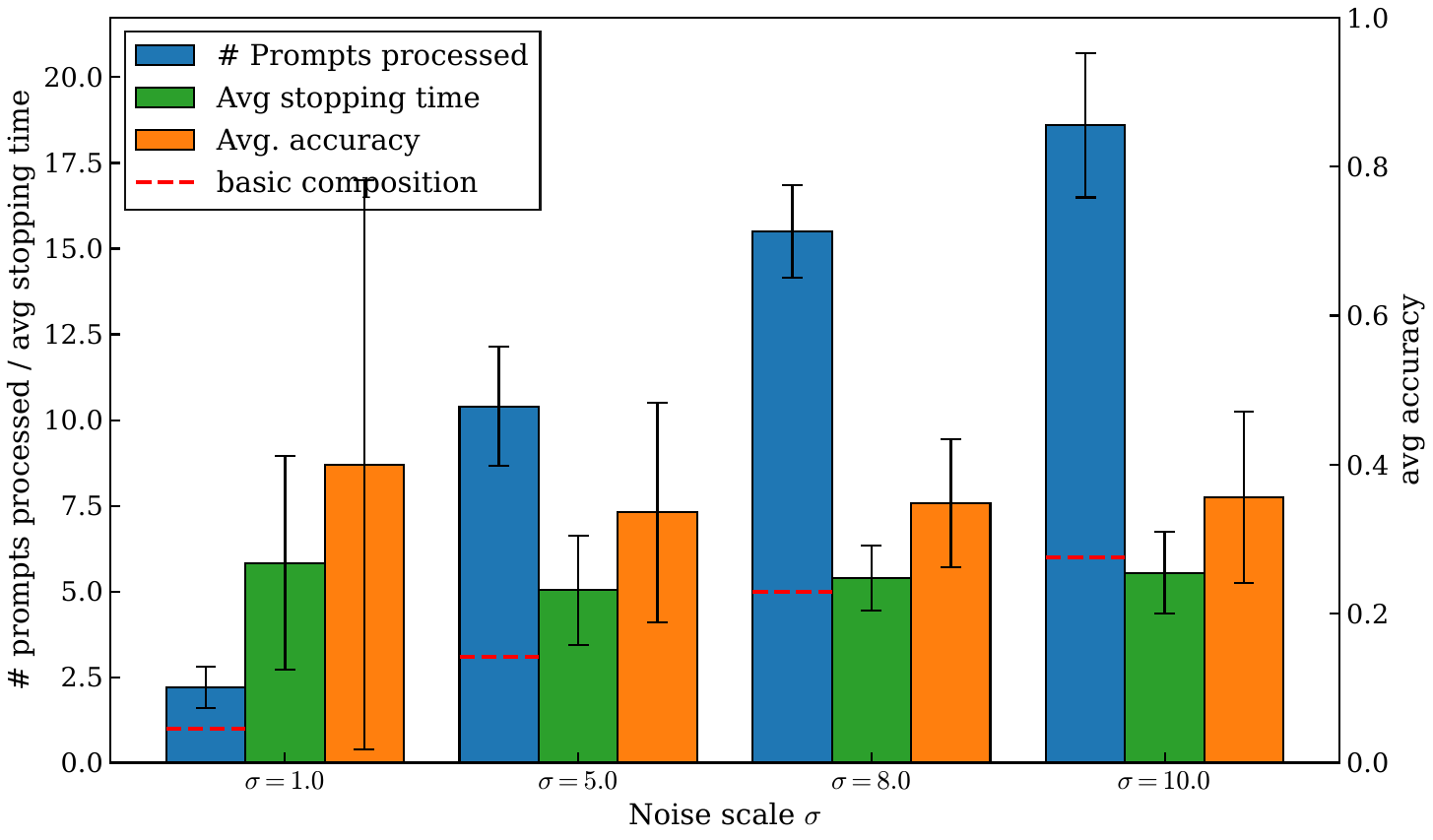}
        \caption{Gemma-RM}
        \label{fig:sub2}
    \end{subfigure}
    
    \begin{subfigure}[t]{0.49\textwidth}
        \centering
        \includegraphics[width=\textwidth]{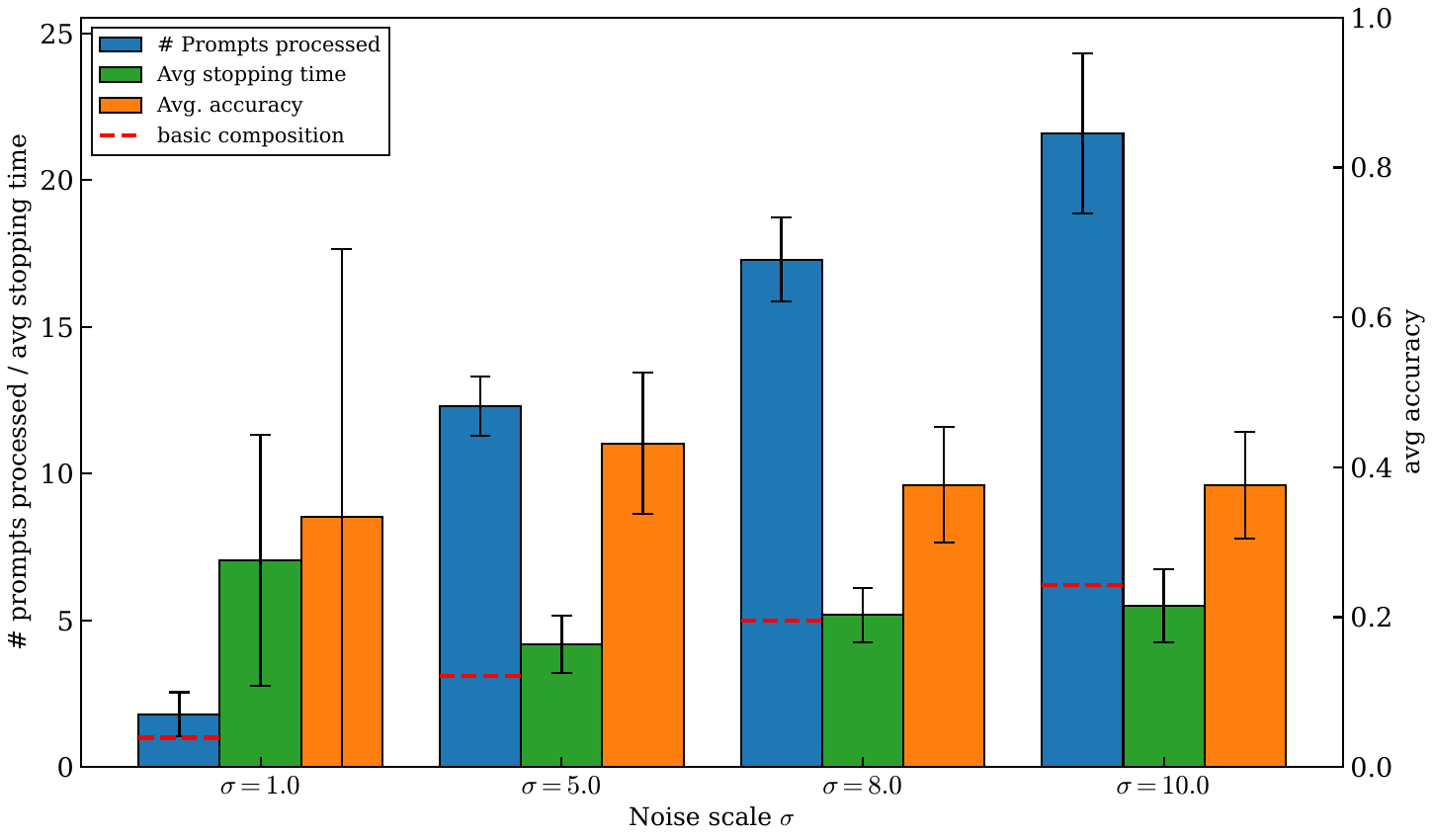}
        \caption{Llama-RM}
        \label{fig:sub3}
    \end{subfigure}
    \hfill 
    \begin{subfigure}[t]{0.49\textwidth}
        \centering
        \includegraphics[width=\textwidth]{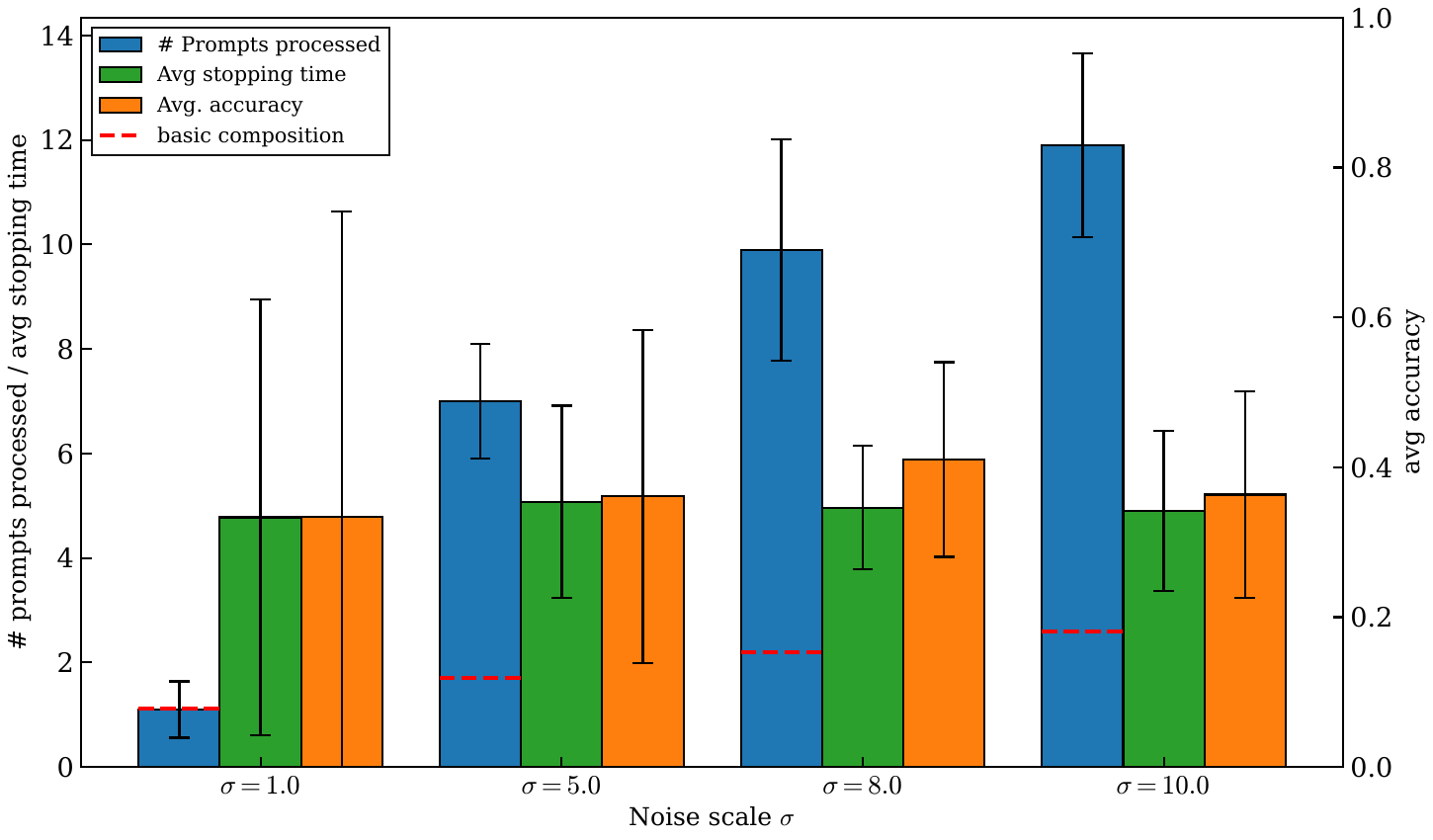}
        \caption{Armo-RM}
        \label{fig:sub4}
    \end{subfigure}
    \caption{Advantage of FSRC composition in PrivITP's ex-post privacy bound compared to the standard composition in ex-ante bound for different reward models and Gemma-2-2b-Instruct as base model}
    \label{fig:fsrc_comp}
\end{figure}

\section{Additional Experiments and Analysis}

\subsection{Open-Ended Evaluation (AlpacaEval-2.0)}

We evaluate our method on the AlpacaEval-2.0 task~\cite{li2023alpacaeval} under the standard proxy/gold protocol. In this protocol, a weaker reward model (RM) is used as the proxy $\hat r$ for selection, and a strong held-out RM is used as the gold $r^\star$ for scoring. Table \ref{tab:alpaca_eval} reports the gold win-rate (\%) versus the base policy, showing the mean and standard error for Gemma-2-2B-Instruct across three proxy RMs (Gemma, Llama, and Armo). We utilize Prometheus-7B-v2.0~\cite{kim2024prometheus2} as the gold $r^\star$ and test across $n \in \{64, 512, 2048, 4096\}$.

\begin{table}[h]
\centering
\caption{Gold win-rate (\%) vs.\ the base policy (Gemma-2-2B-Instruct) on AlpacaEval-2.0. We use Prometheus-7B-v2.0 as the gold judge $r^\star$ and evaluate three different proxy reward models.}
\label{tab:alpaca_eval}
\begin{tabular}{llcccc}
\toprule
\textbf{RM} & \textbf{n} & \textbf{BoN} & \textbf{PrivBoN} & \textbf{ITP} & \textbf{PrivITP} \\
\midrule
\textbf{Gemma} & 64 & 90.82 $\pm$ 1.91 & 88.35 $\pm$ 2.08 & 90.56 $\pm$ 1.99 & 87.89 $\pm$ 2.23 \\
 & 512 & 92.19 $\pm$ 1.72 & 89.65 $\pm$ 1.97 & 90.30 $\pm$ 1.97 & 89.78 $\pm$ 1.95 \\
 & 2048 & 92.06 $\pm$ 1.81 & 89.39 $\pm$ 1.98 & 90.82 $\pm$ 1.89 & 89.97 $\pm$ 2.07 \\
 & 4096 & 92.38 $\pm$ 1.84 & 90.17 $\pm$ 2.02 & 90.36 $\pm$ 1.95 & 91.15 $\pm$ 1.84 \\
\midrule
\textbf{Llama} & 64 & 92.32 $\pm$ 1.87 & 89.39 $\pm$ 1.90 & 92.32 $\pm$ 1.78 & 89.19 $\pm$ 2.03 \\
 & 512 & 91.73 $\pm$ 1.82 & 89.32 $\pm$ 2.03 & 91.93 $\pm$ 1.66 & 90.62 $\pm$ 1.92 \\
 & 2048 & 93.29 $\pm$ 1.67 & 88.87 $\pm$ 2.17 & 92.32 $\pm$ 1.73 & 91.34 $\pm$ 1.85 \\
 & 4096 & 93.16 $\pm$ 1.76 & 89.39 $\pm$ 2.01 & 92.19 $\pm$ 1.73 & 90.89 $\pm$ 2.03 \\
\midrule
\textbf{Armo} & 64 & 91.47 $\pm$ 1.77 & 88.93 $\pm$ 2.10 & 92.45 $\pm$ 1.67 & 87.63 $\pm$ 2.18 \\
 & 512 & 92.19 $\pm$ 1.64 & 89.45 $\pm$ 2.15 & 90.49 $\pm$ 1.99 & 90.04 $\pm$ 1.89 \\
 & 2048 & 91.21 $\pm$ 1.94 & 89.91 $\pm$ 1.99 & 90.17 $\pm$ 2.02 & 89.45 $\pm$ 2.08 \\
 & 4096 & 92.64 $\pm$ 1.75 & 89.52 $\pm$ 2.11 & 91.02 $\pm$ 1.87 & 89.65 $\pm$ 2.03 \\
\bottomrule
\end{tabular}
\end{table}

\textbf{Observations:} First, BoN does not exhibit reward hacking within the tested range, as the gold win-rate remains flat-to-rising as $n$ increases (e.g., Llama increases from 92.32 to 93.16) rather than peaking and declining. This behavior is expected because AlpacaEval win-rates saturate near 90\%+, and the proxy RMs track the gold judge closely enough that $n \le 4096$ does not produce the proxy–gold divergence that drives hacking. Consequently, the private variants match BoN's utility to within standard error at every $(n, \text{RM})$ configuration while adding the privacy guarantee that BoN lacks. The private–non-private gaps are approximately 2–3 points and overlap within one standard error, demonstrating that open-ended generation privacy is obtained at no meaningful utility cost.

\newpage
\subsection{Wall-Clock Overhead and Compute Cost}

The actual computational cost for PrivITP is well below a naive $2\times$ overhead because Phase 2 generation is lazy. Rejection sampling returns at the first accepted candidate at index $t$, meaning candidates only need to be generated on demand. Therefore, Phase 2 costs $\mathbb{E}[t]$ generations rather than $n$, and experiments show that $\mathbb{E}[t] = O(1)$ for informative RMs. The net generation cost is approximately $n + \mathbb{E}[t]$. Table \ref{tab:wall_clock} reports the wall-clock time (seconds per query) of PrivITP compared to PrivBoN at $n = 256$ on the GSM8K dataset using the Phi3-mini-instruct base model on a single H100 GPU.

\begin{table}[h]
\centering
\caption{Wall-clock generation time (seconds per query) for PrivBoN and PrivITP at $n=256$. Evaluated on the GSM8K dataset using the Phi3-mini-instruct base model.}
\label{tab:wall_clock}
\begin{tabular}{lcccc}
\toprule
\textbf{RM} & \textbf{Oasst} & \textbf{Gemma} & \textbf{Llama} & \textbf{Armo} \\
\midrule
\textbf{PrivBoN} & 0.4145 & 0.9003 & 1.1804 & 0.4530 \\
\textbf{PrivITP} & 0.4704 & 0.9827 & 1.2838 & 0.4812 \\
\textbf{Overhead} & +13.5\% & +9.2\% & +8.8\% & +6.2\% \\
\bottomrule
\end{tabular}
\end{table}

PrivITP's overhead ranges from 6–14\% across RMs, which is far below a $2\times$ multiplier. The mechanism overhead itself is negligible in both approaches—PrivBoN draws $n$ Gumbel scalars, while PrivITP draws one scalar Gaussian plus per-round noise, which cost microseconds compared to the $n$ forward passes. Therefore, the gap is dominated primarily by Phase 1's scalar-estimation pass.

\subsection{Privacy Budgets and Sensitivity Analysis}

In our practical implementation, we bound the sensitivity $\Delta_r = \sup|\hat r_D - \hat r_{D'}|$ using the maximum reward range $R_{\max}$, since a single example cannot move a bounded-output head beyond its fixed range. This provides an assumption-free bound when reward model training is outside of the deployer's control.  Table \ref{tab:budgets} reports the privacy budgets across different noise levels for GSM8K with Armo-RM, evaluated at $\delta = 10^{-2}$ and $n=16$. We compute $\varepsilon = 2\Delta_r/\sigma$ for PrivBoN. For PrivITP, the total ex-post cost is $\varepsilon_1 + \varepsilon_2^{\mathrm{post}}$, where Phase 1 costs $\varepsilon_1 = \Delta_r\sqrt{2\ln(1.25/\delta)}/\sigma_X$ and Phase 2 costs $\varepsilon_2^{\mathrm{post}}(\mathbb{E}[t])$ calculated via one-dimensional numerical integration at the mean halting time.

\begin{table}[h]
\centering
\caption{Privacy budgets ($\varepsilon$) across varying noise levels ($\sigma$) for Armo-RM evaluated on the GSM8K dataset. Budgets are computed at $\delta = 10^{-2}$ and a fallback cap of $n=16$. PrivITP's ex-post cost ($\varepsilon_1 + \varepsilon_2^{\mathrm{post}}$) is evaluated at the mean halting time $\mathbb{E}[t]$.}
\label{tab:budgets}
\begin{tabular}{lcccc}
\toprule
$\sigma$ & $\varepsilon$ \textbf{(PrivBoN)} & $\varepsilon_1$ & $\varepsilon_2^{\mathrm{post}}$ & $\varepsilon_1+\varepsilon_2^{\mathrm{post}}$ \textbf{(PrivITP)} \\
\midrule
0.10 & 7.80 & 24.25 & 23.25 & 47.50 \\
0.25 & 3.12 & 9.70 & 8.49 & 18.19 \\
0.50 & 1.56 & 4.85 & 3.48 & 8.33 \\
0.75 & 1.04 & 3.23 & 2.22 & 5.45 \\
1.00 & 0.78 & 2.42 & 1.43 & 3.85 \\
\bottomrule
\end{tabular}
\end{table}

While PrivBoN spans $\varepsilon \in [0.78, 7.80]$, PrivITP's worst-case ex-post total appears larger because it is evaluated at $\mathbb{E}[t]$ under conservative worst-case query values. However, its realized per-query cost is lower, and FSRC composition exploits this to answer approximately $3\times$ more queries at a fixed total budget. Furthermore, tightening to $\delta = 10^{-5}$ increases $\varepsilon_1$ by a factor of only $\approx 1.56$, and $\varepsilon_2^{\mathrm{post}}$ grows at most logarithmically in $n$ because the halting time $t$ is governed by the per-round acceptance probability where $\mathbb{E}[t] = O(1)$.

\FloatBarrier  

\end{document}